%% file: tuning.tex
\PassOptionsToPackage{font=small}{caption}
\PassOptionsToPackage{hyperfootnotes=false}{hyperref}
\documentclass{antgroup}
\definecolor{title_blue}{HTML}{1677FF}
\definecolor{cite_blue}{HTML}{044dc1}
\usepackage{antgroup}

\usepackage[hyperfootnotes=false]{hyperref}
\usepackage{url}
\usepackage{xcolor}
\usepackage{booktabs}
\definecolor{color1}{HTML}{406058}
\definecolor{color2}{HTML}{78A0C8}
\definecolor{color3}{HTML}{3850A0}
\definecolor{color4}{HTML}{6860A8}
\definecolor{color5}{HTML}{A03850}
\definecolor{color6}{HTML}{AE7694}
\definecolor{color7}{HTML}{8488B4}
\hypersetup{
    pdftitle={Intrinsic Distribution Matching},
    pdfauthor={Deyuan Liu, Yansen Han, Kaiwen Zheng, Jun Xie, Jiacheng Liu, Zhenglin Cheng, Peng Sun, Jintao Zhang, Jianguo Li, Tao Lin},
    colorlinks=true,
    linkcolor=title_blue,
    citecolor=cite_blue,
    urlcolor=title_blue,
    pdfborder={0 0 0}
}
\usepackage[T1]{fontenc}

\input{resources/preamble.tex}

\title{Intrinsic Distribution Matching}

\author{Deyuan Liu$^{1,2,\dagger}$ \,
  Yansen Han$^{2,4,\dagger}$ \,
  Kaiwen Zheng$^{3,\dagger}$ \,
  Jun Xie$^{1,2,4,5}$ \,
  Jiacheng Liu$^{1,6}$ \\
  Zhenglin Cheng$^{1,2,4,5}$ \,
  Peng Sun$^{1,2,4}$ \,
  Jintao Zhang$^{3,*}$ \,
  Jianguo Li$^{1,*}$ \,
  Tao Lin$^{1,2,*}$
}

\affiliation{
  $^{1}$Inclusion AI \quad
  $^{2}$Westlake University \quad
  $^{3}$Tsinghua University \quad
  $^{4}$Zhejiang University \\
  $^{5}$Shanghai Innovation Institute \quad
  $^{6}$MBZUAI
}

\fancypagestyle{firststyle}{
  \fancyhead{}
  \fancyhead[L]{\raisebox{0pt}[0pt][0pt]{\includegraphics[height=20pt]{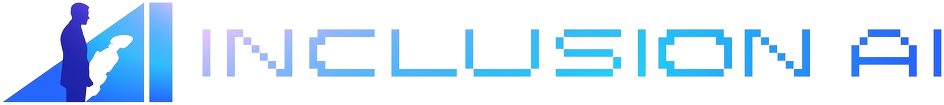}}}
  \renewcommand{\headrulewidth}{0pt}
}

\begin{document}

\maketitle

\begingroup
\renewcommand{\thefootnote}{\fnsymbol{footnote}}
\setcounter{footnote}{2}\footnotetext{Equal contribution.}
\setcounter{footnote}{1}\footnotetext{Corresponding author.}
\endgroup
\setcounter{footnote}{0}

\begin{abstract}
One-step and few-step text-to-image generation must accurately render the details specified in the prompt despite limited opportunities for refinement.
Distribution matching distillation (DMD) uses score differences to bring the generated distribution closer to a reference distribution.
Following prior work, we estimate this difference with internal velocity branches, avoiding a separate fake-score network.
We ask whether an explicit moment constraint improves this learned score supervision.
A decomposition of the reverse-KL field into a Gaussian moment component and a non-Gaussian residual identifies statistics that can be constrained explicitly.
We introduce \textbf{Intrinsic Distribution Matching (\method)}, which retains the full DMD correction and matches feature means and per-coordinate standard deviations using the generator's exponential-moving-average (EMA) copy, requiring no external feature encoder.
Experiments use full-parameter and low-rank adaptation (LoRA) tuning on Z-Image-Turbo 6B and Qwen-Image 20B.
At two network evaluations, \method reaches 94.25\% and 90.20\% on Chinese and English LongText-Bench, the highest scores among the compared methods at this budget.
Controlled ablations show gains from the added moment loss, while generator EMA features perform competitively with an external DINOv2 encoder across the five reported metrics.
Code is available \href{https://github.com/inclusionAI/IDM}{here}.
\end{abstract}

\input{resources/main.tex}

\clearpage

\bibliography{resources/reference}
\bibliographystyle{antgroup}

\clearpage
\appendix

\input{resources/appendix.tex}

\end{document}

%% file: resources/preamble.tex
\usepackage{amsmath,amssymb,amsthm}

\providecommand{\R}{\mathbb{R}}

\providecommand{\EEb}[2]{{\mathbb E}_{#1}\left[#2\right] }

\providecommand{\dm}{\mathrm{d}}

\DeclareMathOperator*{\argmin}{arg\,min\,}

\DeclareMathOperator*{\diag}{diag}

\providecommand{\bb}{\mathbf{b}}
\providecommand{\cc}{\mathbf{c}}
\providecommand{\dd}{\mathbf{d}}

\providecommand{\gv}{\mathbf{g}}
\providecommand{\hh}{\mathbf{h}}

\providecommand{\mm}{\mathbf{m}}

\renewcommand{\ss}{\mathbf{s}}

\providecommand{\uu}{\mathbf{u}}
\providecommand{\vv}{\mathbf{v}}
\providecommand{\ww}{\mathbf{w}}
\providecommand{\xx}{\mathbf{x}}
\providecommand{\yy}{\mathbf{y}}
\providecommand{\zz}{\mathbf{z}}

\providecommand{\mA}{\mathbf{A}}

\providecommand{\mI}{\mathbf{I}}

\providecommand{\mX}{\mathbf{X}}
\providecommand{\mY}{\mathbf{Y}}

\providecommand{\mmF}{\boldsymbol{F}}

\providecommand{\mepsilon}{\boldsymbol{\epsilon}}

\providecommand{\cC}{\mathcal{C}}
\providecommand{\cD}{\mathcal{D}}

\providecommand{\cL}{\mathcal{L}}

\providecommand{\cN}{\mathcal{N}}

\providecommand{\cQ}{\mathcal{Q}}

\newenvironment{talign*}
{\csname align*\endcsname}
{\endalign}

\usepackage{algorithm}
\usepackage{algorithmic}
\usepackage{graphicx}
\usepackage{subcaption}
\usepackage[flushleft]{threeparttable}
\usepackage{float}
\usepackage{placeins}
\usepackage{multirow}
\usepackage{xspace}
\usepackage{natbib}
\usepackage{enumitem}
\usepackage[font=small]{caption}
\usepackage{autobreak}
\usepackage{sidecap}
\usepackage{wrapfig}
\usepackage{bbding}
\usepackage[toc, page, header]{appendix}
\usepackage{tikz}
\usetikzlibrary{arrows.meta,positioning,calc}
\usepackage{pifont}
\usepackage{mdframed}
\usepackage{tcolorbox}
\tcbuselibrary{skins,breakable}
\usepackage{colortbl}
\usepackage{mathrsfs}
\usepackage{footmisc}

\newcommand{\firsttablegroup}[2]{\multicolumn{#1}{@{}l}{\textcolor{black!55}{\itshape\footnotesize #2}}\\[-0.15em]}
\newcommand{\tablegroup}[2]{\midrule\multicolumn{#1}{@{}l}{\textcolor{black!55}{\itshape\footnotesize #2}}\\[-0.15em]}
\newcommand{\circnum}[1]{\raisebox{0.15ex}{\scriptsize\textcircled{\raisebox{-0.05ex}{\tiny #1}}}}

\definecolor{coral}{RGB}{255,127,80}
\definecolor{darkgreen}{RGB}{0,100,0}
\definecolor{darkyellow}{RGB}{204,153,0}
\definecolor{salmon}{RGB}{250,128,114}
\definecolor{champagne}{RGB}{247,231,206}
\definecolor{kleinblue}{RGB}{0,47,167}
\definecolor{bauhiniapurple}{RGB}{193,84,193}
\definecolor{darkred}{RGB}{150,0,0}

\newcommand{\darkredtext}[1]{{\color{darkred}#1}}

\newcommand{\thmref}[1]{\hyperref[#1]{\darkredtext{Thm.~\ref*{#1}}}}
\newcommand{\defref}[1]{\hyperref[#1]{\darkredtext{Def.~\ref*{#1}}}}
\newcommand{\propref}[1]{\hyperref[#1]{\darkredtext{Prop.~\ref*{#1}}}}
\newcommand{\assumpref}[1]{\hyperref[#1]{\darkredtext{Assump.~\ref*{#1}}}}
\newcommand{\remarkref}[1]{\hyperref[#1]{\darkredtext{Rem.~\ref*{#1}}}}
\newcommand{\hypref}[1]{\hyperref[#1]{\darkredtext{Hyp.~\ref*{#1}}}}
\newcommand{\conjref}[1]{\hyperref[#1]{\darkredtext{Conj.~\ref*{#1}}}}
\newcommand{\lemref}[1]{\hyperref[#1]{\darkredtext{Lem.~\ref*{#1}}}}
\newcommand{\corref}[1]{\hyperref[#1]{\darkredtext{Cor.~\ref*{#1}}}}
\newcommand{\noteref}[1]{\hyperref[#1]{\darkredtext{Nota.~\ref*{#1}}}}
\newcommand{\claimref}[1]{\hyperref[#1]{\darkredtext{Clm.~\ref*{#1}}}}
\newcommand{\obsref}[1]{\hyperref[#1]{\darkredtext{Obs.~\ref*{#1}}}}

\newcommand{\transparentteal}[1]{\tikz[baseline=(X.base)] \node[fill=teal, fill opacity=0.1, text opacity=1, inner sep=2pt, outer sep=0pt] (X) {#1};}
\newcommand{\figref}[1]{\hyperref[#1]{\transparentteal{Figure~\ref*{#1}}}}

\newcommand{\transparentdarkgreen}[1]{\tikz[baseline=(X.base)] \node[fill=darkgreen, fill opacity=0.1, text opacity=1, inner sep=2pt, outer sep=0pt] (X) {#1};}
\newcommand{\tabref}[1]{\hyperref[#1]{\transparentdarkgreen{Table~\ref*{#1}}}}

\newcommand{\transparentdarkyellow}[1]{\tikz[baseline=(X.base)] \node[fill=darkyellow, fill opacity=0.1, text opacity=1, inner sep=2pt, outer sep=0pt] (X) {#1};}
\newcommand{\secref}[1]{\hyperref[#1]{\transparentdarkyellow{Section~\ref*{#1}}}}

\newcommand{\transparentcoral}[1]{\tikz[baseline=(X.base)] \node[fill=coral, fill opacity=0.1, text opacity=1, inner sep=2pt, outer sep=0pt] (X) {#1};}
\newcommand{\appref}[1]{\hyperref[#1]{\transparentcoral{Appendix~\ref*{#1}}}}

\newcommand{\transparentkleinblue}[1]{\tikz[baseline=(X.base)] \node[fill=kleinblue, fill opacity=0.1, text opacity=1, inner sep=2pt, outer sep=0pt] (X) {#1};}
\newcommand{\algref}[1]{\hyperref[#1]{\transparentkleinblue{Algorithm~\ref*{#1}}}}

\newcommand{\transparentbauhiniapurple}[1]{\tikz[baseline=(X.base)] \node[fill=bauhiniapurple, fill opacity=0.1, text opacity=1, inner sep=2pt, outer sep=0pt] (X) {#1};}
\newcommand{\equref}[1]{\hyperref[#1]{\transparentbauhiniapurple{Equation~\ref*{#1}}}}

\newtheoremstyle{custom}
{1pt} {1pt} {\itshape} {} {\bfseries} {} { } {\thmname{#1} \thmnumber{#2} \thmnote{(#3)} . }
\theoremstyle{custom}

\newtheorem{innerdefinition}{Definition}
\newtheorem{innerproposition}{Proposition}
\newtheorem{innerassumption}{Assumption}
\newtheorem{innerremark}{Remark}
\newtheorem{innertheorem}{Theorem}
\newtheorem{innerhypothesis}{Hypothesis}
\newtheorem{innerconjecture}{Conjecture}
\newtheorem{innerlemma}{Lemma}
\newtheorem{innercorollary}{Corollary}
\newtheorem{innerexample}{Example}
\newtheorem{innernotation}{Notation}
\newtheorem{innerclaim}{Claim}
\newtheorem{innerproblem}{Problem}

\newtheorem{innerobservation}{Observation}

\tcbset{
  theoremframebase/.style={
    enhanced,
    breakable,
    sharp corners,
    boxrule=0pt,
    leftrule=2pt,
    top=6pt,
    bottom=5pt,
    left=5pt,
    right=5pt,
    before skip=4pt,
    after skip=4pt,
  }
}

\newtcolorbox{definitionframe}{theoremframebase, colback=gray!10, colframe=gray!75!black}
\newtcolorbox{propositionframe}{theoremframebase, colback=blue!8, colframe=blue!70!black}
\newtcolorbox{assumptionframe}{theoremframebase, colback=green!8, colframe=green!55!black}
\newtcolorbox{remarkframe}{theoremframebase, colback=yellow!12, colframe=yellow!60!black}
\newtcolorbox{theoremframe}{theoremframebase, colback=red!7, colframe=red!70!black}
\newtcolorbox{hypothesisframe}{theoremframebase, colback=purple!7, colframe=purple!70!black}
\newtcolorbox{conjectureframe}{theoremframebase, colback=orange!8, colframe=orange!70!black}
\newtcolorbox{lemmaframe}{theoremframebase, colback=cyan!8, colframe=cyan!70!black}
\newtcolorbox{corollaryframe}{theoremframebase, colback=magenta!7, colframe=magenta!70!black}
\newtcolorbox{exampleframe}{theoremframebase, colback=lime!8, colframe=lime!45!black}
\newtcolorbox{notationframe}{theoremframebase, colback=pink!7, colframe=pink!70!black}
\newtcolorbox{claimframe}{theoremframebase, colback=violet!7, colframe=violet!70!black}
\newtcolorbox{problemframe}{theoremframebase, colback=salmon!9, colframe=salmon!70!black}
\newtcolorbox{observationframe}{theoremframebase, colback=purple!5, colframe=purple!55!black}

\newenvironment{definition}
{\begin{definitionframe}\begin{innerdefinition}}
{\end{innerdefinition}\end{definitionframe}}

\newenvironment{proposition}
{\begin{propositionframe}\begin{innerproposition}}
{\end{innerproposition}\end{propositionframe}}

\newenvironment{assumption}
{\begin{assumptionframe}\begin{innerassumption}}
{\end{innerassumption}\end{assumptionframe}}

\newenvironment{remark}
{\begin{remarkframe}\begin{innerremark}}
{\end{innerremark}\end{remarkframe}}

\newenvironment{theorem}
{\begin{theoremframe}\begin{innertheorem}}
{\end{innertheorem}\end{theoremframe}}

\newenvironment{lemma}
{\begin{lemmaframe}\begin{innerlemma}}
{\end{innerlemma}\end{lemmaframe}}

\newenvironment{corollary}
{\begin{corollaryframe}\begin{innercorollary}}
{\end{innercorollary}\end{corollaryframe}}

\usepackage[textwidth=3.5cm]{todonotes}

\newcommand{\method}{\textsc{IDM}\xspace}

\newcommand{\sg}{\mathop{\operatorname{sg}}\nolimits}

\usepackage{makecell}

\usepackage{bm}

\definecolor{ForestGreen}{RGB}{34,139,34}
\definecolor{OrangeRed}{RGB}{255,69,0}

%% file: resources/main.tex
\section{Introduction}
\label{sec:intro}

Diffusion and flow-matching models produce high-quality images, but their repeated network evaluations impose substantial inference cost~\citep{ho2020denoising,song2020score,lipman2022flow,liu2022flow}. One and few-step generation reduces this cost by learning a much shorter sampling process. The challenge is to accurately render the details specified in the prompt using only one or a few sampling steps. Long-text rendering makes this challenge concrete: many characters must remain legible and correct within a coherent image~\citep{geng2025x,wu2025qwen}. It therefore provides a demanding test of whether fast generation retains detailed conditional information.

Distribution matching distillation (DMD) trains a fast generator to match a reference distribution~\citep{yin2024one,yin2024improved}. A score describes the local direction of increasing log density. A pretrained teacher supplies the reference score, while an auxiliary fake-score network estimates the score of the current generated distribution. Their difference guides updates to generated samples. These updates are applied sample by sample, but their direction depends on both distributions. Distribution matching can complement trajectory-based training: rCM combines score regularization with continuous-time consistency training~\citep{zheng2025large}.

\begin{figure}[!t]
  \centering
  \includegraphics[width=\linewidth]{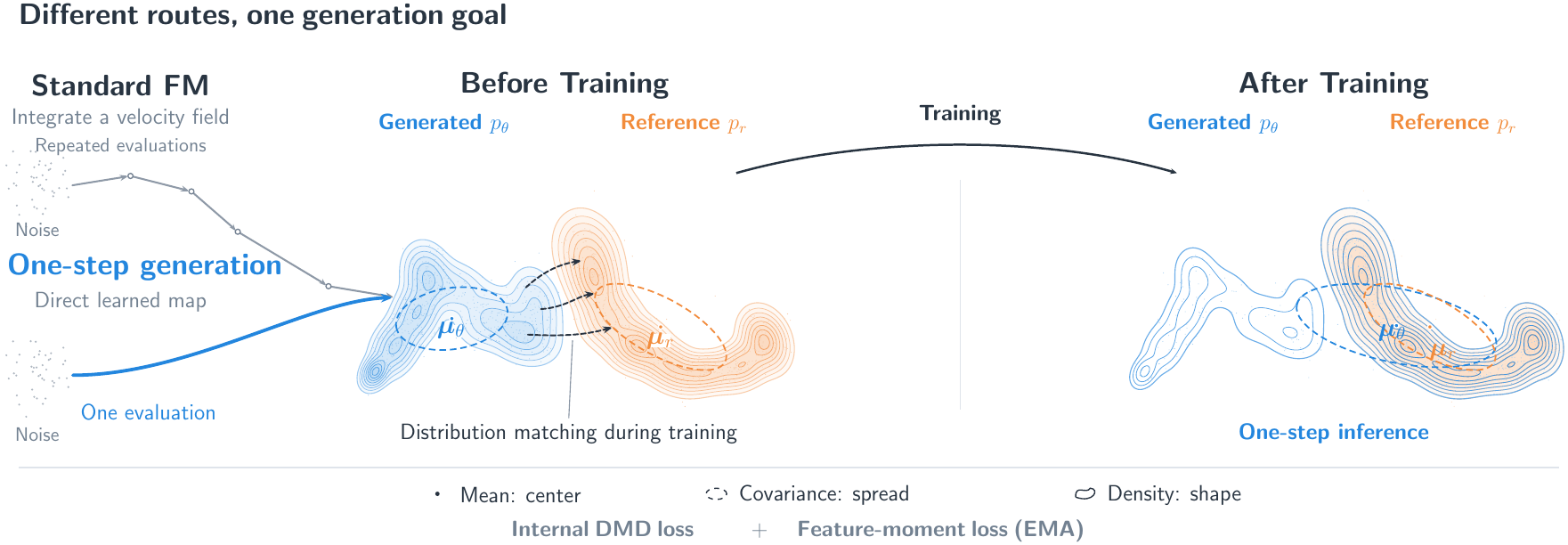}
  \caption{\textbf{Overview of \method for one-step generation.}
  Multi-step sampling refines samples iteratively, whereas one-step generation learns a direct noise-to-data map.
  \method combines an internal DMD loss with a feature-moment loss in the generator's EMA feature space.
  The right panel illustrates idealized alignment with single-evaluation inference.}
  \label{fig:overview}
  \vspace{-2em}
\end{figure}

TwinFlow and APEX obtain this score difference from internally separated velocity branches of the generator~\citep{cheng2025twinflow,liu2026self}. The model-score estimate is still learned, but no separate fake-score network is needed. We adopt this internal construction and ask whether the generator can also supply an explicit constraint on statistics across samples. Although the exact DMD field contains full distributional information, learned score estimates and finite training can leave discrepancies that an additional statistical objective may help reduce.

To motivate such a constraint, we revisit the reverse-KL score difference underlying DMD. Building on the Gaussian-plus-residual score representation~\citep{wang2023hidden,wang2024unreasonable}, we separate a component determined by means and covariances from the remaining score structure in~\secref{sec:decomposition}. This identifies statistics that can be constrained explicitly, while showing why moments alone do not determine the full distribution. We therefore retain the full DMD correction and add a moment loss, testing its benefit through controlled ablations in a shared training setting.

We introduce \textbf{Intrinsic Distribution Matching (\method)}, which augments internal DMD with moment matching in the generator's own feature space. Its exponential-moving-average (EMA) copy provides a shared representation for generated and reference samples; we match the means and per-coordinate standard deviations of these features. Unlike objectives that obtain features from an external encoder or teacher~\citep{yang2026fdloss,zhang2026teacher}, this design reuses the generator for both score estimation and feature extraction. Both training signals are obtained from the generator and its EMA copy, without a separate discriminator, score model, or external feature encoder.

We evaluate \method using Z-Image-Turbo 6B and Qwen-Image 20B with full-parameter and LoRA tuning in~\secref{sec:exp}. LongText-Bench evaluates the Chinese and English long-text rendering fidelity emphasized above, with subsets denoted LongText-CN and LongText-EN~\citep{geng2025x}. At two network evaluations (2~NFE), \method on Z-Image-Turbo achieves state-of-the-art scores among the compared methods: 94.25\% and 90.20\%, improving over Z-Image-Turbo++ by 2.63 and 0.32 percentage points, respectively in~\tabref{tab:ablation_summary_2nfe}.

\textbf{Our main contributions are:}
\begin{enumerate}[label=\alph*., nosep, leftmargin=16pt]
  \item \textbf{Intrinsic distributional supervision.} We augment internal DMD with an explicit loss on the means and per-coordinate standard deviations of generator EMA features. This adds statistical supervision across samples without a separate external feature encoder.
  \item \textbf{Distribution-gradient explanation.} Building on the Gaussian-plus-residual score representation, we motivate an explicit moment constraint alongside the full DMD correction.

  \item \textbf{Few-step training and evaluation.} We demonstrate full-parameter and LoRA training on 6B and 20B backbones. Ablations assess the added moment loss, EMA features, and regression choices, while LongText-Bench evaluates fine-grained text fidelity.
\end{enumerate}

\clearpage

\begin{figure}[!t]
\centering
\includegraphics[width=\linewidth,height=0.95\textheight,keepaspectratio]{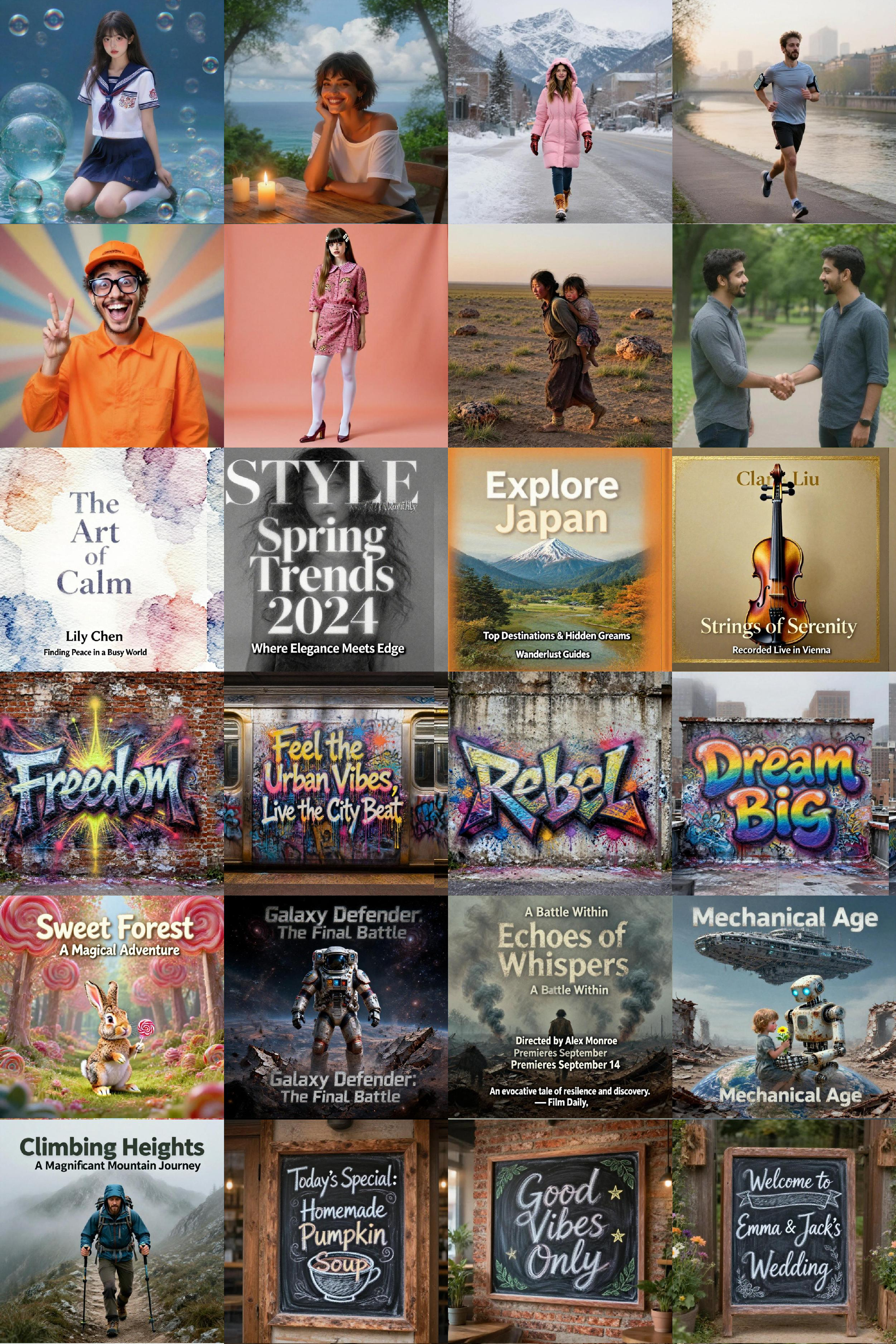}
\caption{\textbf{\method (Z-Image), 2~NFE.} Text rendering and portrait generation.}
\label{fig:qualitative_showcase_1}
\end{figure}

\clearpage

\section{Background}
\label{sec:background}

\paragraph{Continuous Generative Models.}
Diffusion and flow-matching models transform noise into data through repeated network evaluations~\citep{ho2020denoising,song2020score,lipman2022flow}. Distillation learns a shorter process~\citep{salimans2022progressive,song2023consistency}. We use rectified flow to define the generator and then explain the DMD supervision used to train it.

\paragraph{Rectified Flow and Few-Step Generation.}
Let $p_z=\cN(\mathbf{0},\mI_d)$ be the noise distribution and $p_r$ the reference-data distribution. Rectified flow uses the path $\xx_{r,t}=(1-t)\xx_r+t\mepsilon$, with independently sampled $\xx_r\sim p_r$ and $\mepsilon\sim p_z$~\citep{liu2022flow}. Thus $t=0$ denotes clean data and $t=1$ pure noise.
A network $\mmF_\theta(\xx_{r,t},t,\cc)$ regresses the pair velocity $\mepsilon-\xx_r$, where $\cc\in\cC$ is a text-conditioning vector. At the squared-error optimum it recovers the marginal velocity, the conditional average of pair velocities at the same noisy input and condition.
Few-step generation integrates this field from $t=1$ to $t=0$ with a small number of forward-Euler steps. At one network evaluation, the generator is $G_\theta(\zz,\cc)=\zz-\mmF_\theta(\zz,1,\cc)$ for $\zz\sim p_z$.

\paragraph{Distribution Matching Distillation.}
The generator induces a distribution $p_\theta$. DMD aims to match it to $p_r$ through the reverse-KL gradient~\citep{wang2023prolificdreamer,yin2024one}. At a fixed condition and noise level $t\in(0,1)$, let $p_{\theta,t}$ and $p_{r,t}$ be their noised distributions. A score $\ss_{i,t}=\nabla_{\xx_t}\log p_{i,t}$ points locally toward increasing log density. Evaluated at the same noised model sample, $\ss_{r,t}-\ss_{\theta,t}$ points toward a lower model-to-reference density ratio in~\secref{sec:canonical_field}.
The reference score is supplied by a frozen pretrained teacher. Since the generator provides samples but no tractable score, DMD trains a separate \emph{fake-score network} on its outputs. This network estimates the generated distribution; the generator itself produces the images. The fake-score estimate must track the changing generator. The reference score guides samples toward regions supported by the target, while the model score accounts for the generator's current concentration of probability mass. Both enter the distributional update, so the model score must still be learned when a separate network is removed.

\paragraph{Score--Velocity Identity.}
For independent $\xx\sim p$ and $\mepsilon\sim p_z$, set $\xx_t=(1-t)\xx+t\mepsilon$. For $t\in(0,1)$, the score $\ss_t:=\nabla_{\xx_t}\log p_t$ and marginal velocity $\vv_t^*(\xx_t):=\mathbb E[\mepsilon-\xx\mid\xx_t]$ satisfy
\begin{equation}
  \ss_t(\xx_t)=-\frac{\xx_t+(1-t)\vv_t^*(\xx_t)}{t}\,.
  \label{eq:score_velocity}
\end{equation}
Thus a velocity estimate also provides a score estimate in~\appref{app:score_velocity}. Subtracting two velocities at the same input gives a scaled score difference.

\paragraph{Internal Score Estimation.}
TwinFlow and APEX use two internally separated velocity branches of one generator to estimate the reference and model scores~\citep{cheng2025twinflow,liu2026self}. This removes the separate fake-score network while retaining model-score learning. Following APEX, we use this construction for DMD in~\secref{sec:micro}. We next examine the statistics represented by its target field to motivate additional supervision from the generator's features.

\section{A Distribution-Gradient View}
\label{sec:unified}

The score difference in DMD defines a direction for changing the generated distribution. Separating the contribution of means and covariances from the remaining score structure clarifies what moment matching can constrain and why the full DMD correction is still needed.

\subsection{Reverse-KL Field}
\label{sec:setup}
\label{sec:canonical_field}

The one-step generator maps $\zz\sim p_z$ to $\xx_\theta=G_\theta(\zz,\cc)$, inducing $p_\theta(\cdot\mid\cc)$; $p_r(\cdot\mid\cc)$ is the reference law. Here $\xx$ is a sample in the generator's output representation, and conditions follow $p_c$ independently of $\zz$.
For differentiable positive densities, define $\ss_i(\xx\mid\cc):=\nabla_\xx\log p_i(\xx\mid\cc)$, $i\in\{r,\theta\}$. At fixed $\cc$, the reverse KL is
\begin{equation}
  \cD_{\mathrm{KL}}(p_\theta\|p_r)=\int p_\theta(\xx\mid\cc)\log\frac{p_\theta(\xx\mid\cc)}{p_r(\xx\mid\cc)}\,\dm\xx.
\end{equation}
Its sample-space distribution gradient is
\begin{equation}
  \gv_{\mathrm{KL}}(\xx\mid\cc)=\ss_\theta(\xx\mid\cc)-\ss_r(\xx\mid\cc).
  \label{eq:reverse_kl_field}
\end{equation}
The descent direction is $-\gv_{\mathrm{KL}}=\ss_r-\ss_\theta$, pointing locally toward a lower model-to-reference density ratio. It changes as the generator distribution changes.
Applying the same identity to $p_{\theta,t}$ and $p_{r,t}$ gives $\gv_{\mathrm{KL},t}=\ss_{\theta,t}-\ss_{r,t}$, the noisy field estimated in DMD. We state the identities below for clean laws under the regularity assumptions in \appref{app:unified_grad}; the decomposition also applies to the noisy laws at each fixed $t$, with their own means, covariances, and score residuals.

\begin{figure}[!t]
  \centering
  \includegraphics[width=\linewidth]{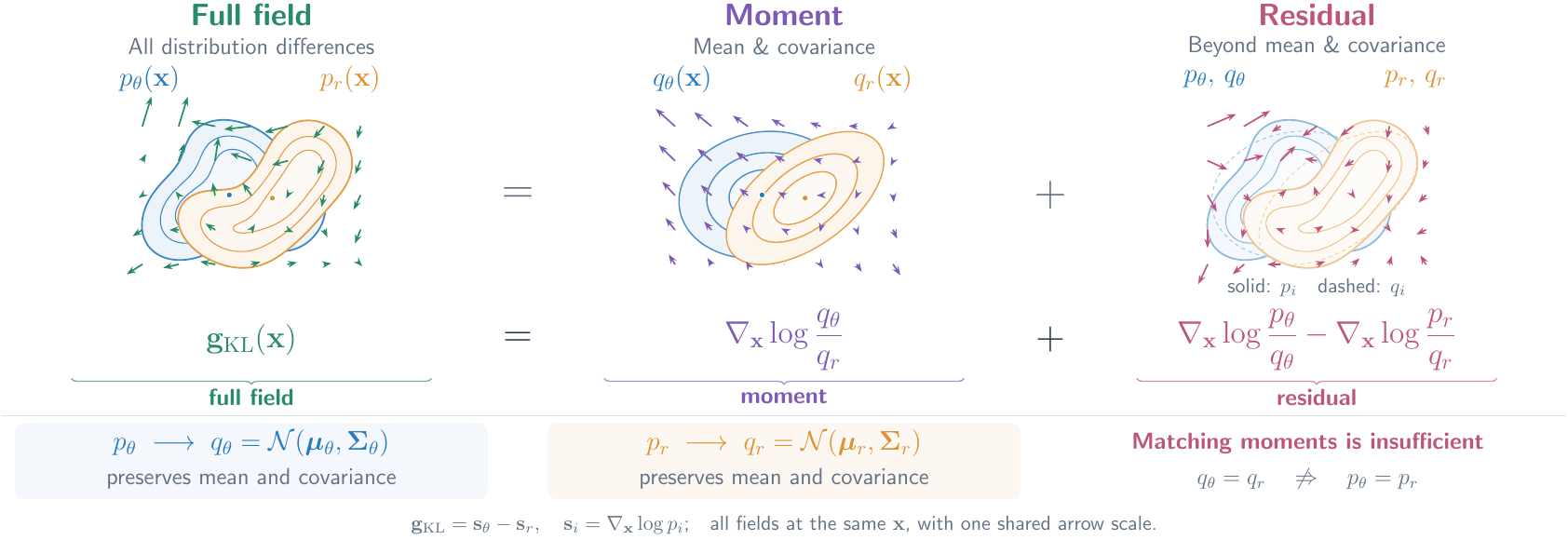}
  \caption{\textbf{Reverse-KL Field Decomposition.}
  The full field $\gv_{\mathrm{KL}}=\ss_\theta-\ss_r$ (\textbf{left}) is the sum of a Gaussian moment component (\textbf{middle}) and a non-Gaussian residual component (\textbf{right}). Solid contours show $p_i$ and dashed contours its moment-matched Gaussian $q_i$, for $i\in\{r,\theta\}$.
  Equal means and covariances eliminate the Gaussian component but can leave a nonzero residual. All fields use the same evaluation points and arrow scale.}
  \label{fig:decomposition_geometry}
\end{figure}

\subsection{Parameter Gradient}

The field above specifies a sample update. To relate it to training, consider a discrepancy $\cD(p_\theta,p_r)$ with a differentiable first-variation potential, which describes how redistributing probability mass changes $\cD$. Its spatial gradient specifies the field used in the sample update:
\begin{equation}
  \gv_{\cD}(\xx\mid\cc)
  :=\nabla_\xx\frac{\delta\cD(p_\theta,p_r)}{\delta p_\theta}(\xx\mid\cc).
  \label{eq:distribution_gradient}
\end{equation}
Under the assumptions in \appref{app:unified_grad},
\begin{equation}
  \nabla_\theta\cD(p_\theta,p_r)
  =\EEb{\zz\sim p_z}{
    \underbrace{[J_\theta G_\theta(\zz,\cc)]^\top}_{\text{Jacobian transpose}}
    \underbrace{\gv_{\cD}(G_\theta(\zz,\cc)\mid\cc)}_{\text{distribution gradient}}
  }\,.
  \label{eq:universal_grad}
\end{equation}
Here $J_\theta G_\theta:=\partial G_\theta/\partial\theta$; backpropagation applies its transpose, and training averages over $\cc$. For reverse KL, the first variation is $\log(p_\theta/p_r)+1$, recovering Eq.~\ref{eq:reverse_kl_field}. Thus the score difference gives the distributional direction that is passed back through the generator. This separates the distributional change that reduces the discrepancy from the parameter update that realizes this change. The decomposition below concerns the former and does not depend on a particular generator architecture.

\subsection{Decomposition of the Reverse-KL Field}
\label{sec:decomposition}

To identify which statistics this field contains, let $p$ be a smooth positive density with finite second moment, mean $\mu_p$, and positive-definite covariance $\Sigma_p$. Its moment-matched Gaussian is $q_p:=\cN(\mu_p,\Sigma_p)$, with score $\ss_{q_p}(\xx)=-\Sigma_p^{-1}(\xx-\mu_p)$. The score of $p$ decomposes exactly as
\begin{equation}
  \ss_p(\xx)
  =\underbrace{\ss_{q_p}(\xx)}_{=-\Sigma_p^{-1}(\xx-\mu_p)}
  +\underbrace{\boldsymbol{\eta}_p(\xx)}_{\text{non-Gaussian score residual}}\,,
  \label{eq:score_decomp_main}
\end{equation}
where $\boldsymbol{\eta}_p:=\ss_p-\ss_{q_p}$ contains the score structure beyond this Gaussian approximation~\citep{wang2023hidden,wang2024unreasonable}.
At fixed $\cc$, use $q_i:=q_{p_i}$ and $\boldsymbol{\eta}_i:=\boldsymbol{\eta}_{p_i}$ for the conditional laws $p_i(\cdot\mid\cc)$, $i\in\{r,\theta\}$. Subtracting their decompositions gives
\begin{equation}
  \gv_{\mathrm{KL}}(\xx\mid\cc)
  =
  \underbrace{\left(\ss_{q_\theta}(\xx\mid\cc)-\ss_{q_r}(\xx\mid\cc)\right)}_{\gv_{\mathrm{KL}}^{\mathrm{G}}:\;\text{Gaussian moment component}}
  +
  \underbrace{\left(\boldsymbol{\eta}_\theta(\xx\mid\cc)-\boldsymbol{\eta}_r(\xx\mid\cc)\right)}_{\gv_{\mathrm{KL}}^{\mathrm{nG}}:\;\text{non-Gaussian residual component}}.
  \label{eq:grad_decomp_main}
\end{equation}
The Gaussian component depends only on means and covariances. Two distributions can share these moments yet have different shapes, leaving a nonzero residual in~\figref{fig:decomposition_geometry}. Replacing $p_i$ by $p_{i,t}$ gives the same decomposition for the noisy field used in DMD.

Under the integrability and boundary conditions in \appref{app:orthogonality},
\begin{equation}
  \mathbb E_p[\boldsymbol{\eta}_p(\xx)]=\mathbf{0},
  \qquad
  \mathbb E_p[\boldsymbol{\eta}_p(\xx)\xx^\top]=\mathbf{0}_{d\times d}.
  \label{eq:orthogonality_main}
\end{equation}
Each score residual is orthogonal to constant and linear functions under its own distribution; their difference need not be orthogonal under either distribution.
The moment-matched Gaussian is used to analyze the score; the generated samples are not assumed to follow a Gaussian distribution. The residual does not identify a single higher-order moment: it collects all score structure left after subtracting the Gaussian score for the corresponding distribution.

\paragraph{Implication for training.}
Exact DMD already contains both components. In practice, its learned score estimates and finite training need not align all statistics accurately. The decomposition identifies moments as quantities that can be constrained explicitly, while the residual explains why moment matching alone is insufficient.
We therefore keep the full DMD correction and add a loss on feature means and standard deviations. Features from the generator's EMA copy provide a shared representation without introducing an external encoder. This is an additional feature-space objective motivated by the decomposition, not a separate estimate of its sample-space Gaussian or residual component. The ablations in \secref{sec:exp_ablation} test whether this explicit constraint improves generation.

\FloatBarrier

\section{Intrinsic Distribution Matching}
\label{sec:combined}
Intrinsic Distribution Matching (\method) retains internal DMD supervision and adds the feature-moment constraint motivated in \secref{sec:decomposition}. The DMD loss $\cL_{\text{score}}$ uses the full internal score-difference estimate; $\cL_{\text{pop}}$ explicitly matches statistics of generator EMA features. We first define how the DMD correction supervises velocity and endpoint predictions, then describe the feature constraint, the empirical moment estimates, and the gradients used in joint training.
\subsection{DMD Loss}
\label{sec:micro}
At a fixed condition $\cc$, let $p_{r,t}$ and $p_{\theta,t}$ denote the reference and model marginals under the linear noising construction of \secref{sec:background}, with scores $\ss_{r,t},\ss_{\theta,t}$ and exact marginal velocities $\vv^*_{r,t},\vv^*_{\theta,t}$.
Applying Eq.~\ref{eq:score_velocity} to both at the same point $\xx_t$ gives, for $t\in(0,1)$:
\begin{equation}
  \dd_t^*(\xx_t)
  := \vv^*_{\theta,t}(\xx_t) - \vv^*_{r,t}(\xx_t)
  = \frac{t}{1{-}t}\left(\ss_{r,t}(\xx_t) - \ss_{\theta,t}(\xx_t)\right)\!,
  \label{eq:d_micro}
\end{equation}
so $\gv_{\mathrm{KL},t}:=\ss_{\theta,t}-\ss_{r,t}=-\frac{1-t}{t}\,\dd_t^*$ is the reverse-KL field on these noisy marginals in Eq.~\ref{eq:reverse_kl_field}.
\paragraph{Internal velocity branches.}

Following APEX~\citep{liu2026self}, we separate the reference and model velocity branches by negating the condition vector:
\begin{equation}
  \widehat\vv_{r,t}(\xx_t\mid\cc):=\mmF_\theta(\xx_t,t,\cc),\qquad
  \widehat\vv_{\theta,t}(\xx_t\mid\cc):=\mmF_\theta(\xx_t,t,-\cc),
  \label{eq:branch_instantiation}
\end{equation}
Here $-\cc$ selects the model branch for condition $\cc$, and $\widehat\dd_t(\xx_t\mid\cc):=\widehat\vv_{\theta,t}(\xx_t\mid\cc)-\widehat\vv_{r,t}(\xx_t\mid\cc)$ is their pointwise difference.
This interpretation requires both branches to approximate their marginal velocities.
\paragraph{Velocity and endpoint regression.}
The same DMD correction can supervise velocity predictions or their corresponding clean-sample endpoints. We compare these regression forms and their combinations in \secref{sec:exp_ablation} to select the training configuration.

For fixed input $\xx_t$ at time $t\in(0,1]$, the endpoint prediction is $\widehat\xx_0(\xx_t,t,\cc):=\xx_t-t\mmF_\theta(\xx_t,t,\cc)$.
At $t=1$ and $\xx_t=\zz$, this equals $G_\theta(\zz,\cc)$; predictions at intermediate times need not follow the one-step output law $p_\theta$.
Let $\sg$ denote stop-gradient. Set $\widehat\xx_0^{\sg}:=\sg(\widehat\xx_0(\xx_t,t,\cc))$ and re-noise it as $\widehat\xx_{0,t'}^{\sg}:=(1-t')\widehat\xx_0^{\sg}+t'\mepsilon$, where the re-noising time $t'\in(0,1)$ and $\mepsilon\sim p_z$ are sampled independently of each other and of the regression input.
Evaluate the branch difference $\widehat\dd_{t'}(\widehat\xx_{0,t'}^{\sg}\mid\cc)$ at this re-noised prediction. Its weighted, normalized value is detached to obtain $\widehat\dd_{t'}^{\mathrm{DMD}}$; without processing, the correction is $\sg\!\left(\widehat\dd_{t'}(\widehat\xx_{0,t'}^{\sg}\mid\cc)\right)$.
The DMD targets are
\begin{equation}
\begin{aligned}
\vv_t^{\mathrm{DMD}}
&:=\sg\!\left(\mmF_\theta(\xx_t,t,\cc)-\widehat\dd_{t'}^{\mathrm{DMD}}\right),\\
\xx_0^{\mathrm{DMD}}
&:=\sg\!\left(\xx_t-t\vv_t^{\mathrm{DMD}}\right).
\end{aligned}
\label{eq:dmd_targets}
\end{equation}
Subtracting this correction from the velocity target adds $t\widehat\dd_{t'}^{\mathrm{DMD}}$ to the detached clean-sample prediction. Thus the same correction defines both targets.
The combined squared regression loss, with both targets detached, averages over conditions, regression inputs and times, and re-noising draws:
\begin{equation}
\cL_{\text{score}}(\theta)
:=\mathbb E\!\left[
\underbrace{\|\mmF_\theta(\xx_t,t,\cc)-\vv_t^{\mathrm{DMD}}\|_2^2}_{\ell_v(t)}
+\underbrace{\|\widehat\xx_0(\xx_t,t,\cc)-\xx_0^{\mathrm{DMD}}\|_2^2}_{\ell_x(t)}
\right].
\label{eq:l_score}
\end{equation}
At fixed $(\xx_t,t)$, these squared losses satisfy $\ell_x(t)=t^2\ell_v(t)$, so equal weights give an effective factor $1+t^2$ on velocity regression in~\appref{app:micro_theory}. A separate discussion of conditional-mean regression and covariance appears in \appref{app:var_collapse}.
\begin{figure}[!t]
  \centering
  \includegraphics[width=\linewidth]{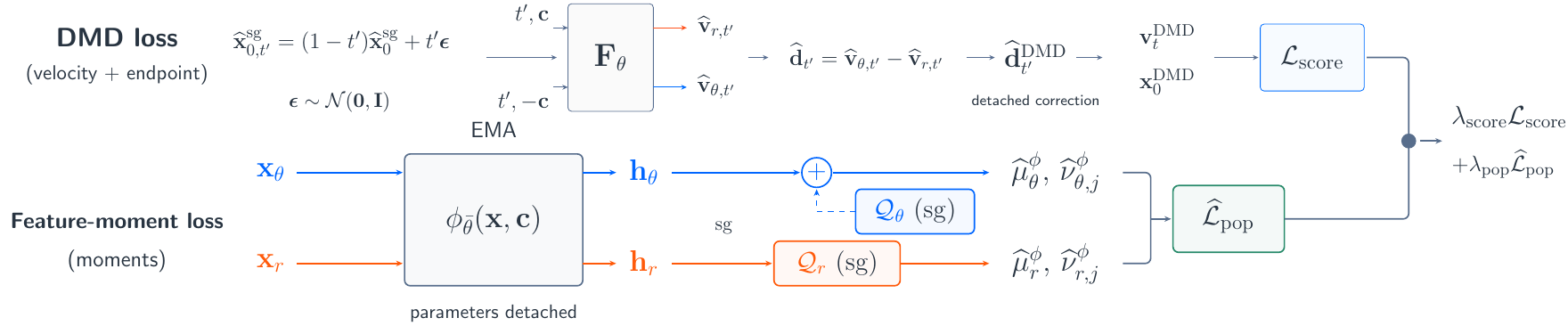}
  \captionsetup{belowskip=16pt}
  \caption{
\textbf{Intrinsic losses of \method.}
\textbf{Top, DMD loss:} re-noise the predicted clean sample and use the internal velocity difference to construct detached velocity and endpoint targets for $\cL_{\text{score}}$ in Eq.~\ref{eq:l_score}.
\textbf{Bottom, feature-moment loss:} match means and per-coordinate standard deviations of EMA features $\hh_i:=\phi_{\bar\theta}(\xx_i,\cc)$ using current model features and detached model/reference queues $\cQ_\theta,\cQ_r$ in Eq.~\ref{eq:moment_loss}.
Base objective omitted.
  }
  \label{fig:method}
\end{figure}

\subsection{Feature-Moment Loss}
\label{sec:macro}

Alongside the learned DMD correction, we compare generated and reference samples in the feature space of the generator's EMA copy. The EMA supplies the representation, while reference samples provide the target statistics for moment matching.
\paragraph{Intrinsic feature map.}
\label{sec:whose_features}

$\phi:=\phi_{\bar\theta}:\R^d\times\cC\to\R^{d_\phi}$ mean-pools the last-layer hidden activations of the generator's EMA copy into one feature vector per sample. Generated and reference samples use the noising construction of \secref{sec:background} at high noise levels, with noising arguments suppressed in $\phi$. Here $\bar\theta$ denotes EMA parameters and $d_\phi$ the feature dimension; each update fixes $\bar\theta$ while retaining input derivatives.
We compare this choice with frozen-encoder~\citep{yang2026fdloss}, teacher-feature~\citep{zhang2026teacher}, and live self-feature alternatives in \secref{sec:exp_ablation}.
\par
\paragraph{Matching feature moments.}
\label{sec:moment_general}

Let $\mu_i^\phi$ and $\nu_{i,j}^\phi$ denote the mean and per-coordinate variance of $\phi(\xx,\cc)$ under the joint law $(\xx,\cc)\sim p_i(\cdot\mid\cc)p_c$, also averaging over feature-extraction noise, for $i\in\{r,\theta\}$. These statistics pool across conditions rather than matching conditional distributions.
We match these pooled statistics with
\begin{equation}
  \cL_{\text{pop}} = \|\mu_\theta^\phi - \mu_r^\phi\|_2^2 + \sum_{j=1}^{d_\phi} \left(\sqrt{\nu_{\theta,j}^\phi} - \sqrt{\nu_{r,j}^\phi}\right)^2\!.
  \label{eq:moment_loss}
\end{equation}
This is the squared 2-Wasserstein distance between the diagonal Gaussians $\cN(\mu_i^\phi,\diag(\nu_i^\phi))$; it is non-negative and vanishes exactly when their feature means and per-coordinate standard deviations match. At the feature level, the mean term shifts the average toward the reference mean, while the standard-deviation term adjusts the spread around that average. These roles are explicit in the empirical gradient in \propref{prop:mt_grad}. Their effect on images is mediated by the fixed EMA feature Jacobian and the generator, so the objective does not prescribe a pixel-wise correction.
A nonlinear feature map can capture properties not determined by raw-sample means and covariances. Matching its moments therefore constrains statistics in a learned representation. The loss still uses only pooled feature means and per-coordinate standard deviations; it does not estimate the sample-space Gaussian component or match the full feature distribution.
\par
Queues $\cQ_\theta$ and $\cQ_r$ store up to $K=16384$ detached model and reference features, respectively. Model moments combine $B$ current-batch features with $\cQ_\theta$; reference moments use $\cQ_r$, giving $B+K$ and $K$ features with full queues. These moments define $\widehat{\cL}_{\text{pop}}$.
Queue entries come from earlier model and EMA states, so the moments approximate current population statistics. \propref{prop:mt_grad} gives the exact empirical-loss gradient at fixed queues in~\appref{app:macro_theory}.

\subsection{Training Objective}
\label{sec:objective}
We use the RCGM loss~\citep{sun2025anystep}, which supervises generation across time intervals through recursive consistent velocity-field estimation, as $\cL_{\text{base}}$. The main generator objective is
\begin{equation}
\cL_{\text{\method}}(\theta)
=\cL_{\text{base}}(\theta)
+\lambda_{\text{score}}\cL_{\text{score}}(\theta)
+\lambda_{\text{pop}}\widehat{\cL}_{\text{pop}} .
\label{eq:l_idm}
\end{equation}
Each iteration generates current samples and extracts their high-noise EMA features. The current features retain their computation graph and are combined with the detached queues to evaluate $\widehat{\cL}_{\text{pop}}$. Its gradient passes through the fixed EMA network to the current samples and generator; queued features receive no gradients. Across iterations, we refresh the queues with detached features and update the EMA as $\bar\theta\leftarrow0.99\bar\theta+0.01\theta$. The weights $\lambda_{\text{score}}$ and $\lambda_{\text{pop}}$ remain fixed in~\secref{sec:exp_setup}. Feature extraction and queues are used only during training.
Freezing the EMA parameters does not freeze the current features: their values change with the generated images. Detaching these features would also remove the moment-loss gradient to the generator. The detached queues provide historical statistics without retaining computation graphs from previous iterations.
Auxiliary flow matching $\cL_{\text{fake}}$ trains the internal model-velocity branch on detached generated samples and also updates the shared parameters in~\appref{app:internal_dmd_training}.
One-step inference generates $\xx_\theta=G_\theta(\zz,\cc)$ in one forward pass; the same checkpoint also supports few-step sampling.

\section{Experiments}
\label{sec:exp}
Our experiments assess the added value of $\cL_{\text{pop}}$ over internal DMD, the suitability of generator EMA features relative to an external encoder, and the choice of velocity and endpoint regression. System-level comparisons evaluate the complete training setting; controlled ablations examine these individual choices under the shared data, backbone, and sampling budget. Throughout this section, $\cL_{\text{pop}}$ refers to the empirical loss $\widehat{\cL}_{\text{pop}}$.

\begin{table*}[t!]
\centering
\caption{\textbf{System-level comparison of efficiency and quality.}
Speeds are on a single A100 (\textbf{BF16}).
Throughput is samples/s (batch=10); latency is seconds (batch=1).
GenEval measures compositional prompt following; FID/CLIP complement the efficiency comparison. LongText results appear in \protect\tabref{tab:ablation_summary_2nfe}.
\dag~indicates that distinct models are required for different NFEs.
\ddag~marks a noncanonical operating point or a secondary comparison source.
\S~indicates that the LoRA weights are merged into the base weights before timing.
\textbf{Params:} backbone size in billions; LoRA entries additionally give rank $r$ and trainable parameters in parentheses. Blue/red backbone counts indicate full/LoRA tuning, respectively.
}
\small
\vspace{-1em}
\label{tab:main_comparison}{
\resizebox{\linewidth}{!}{
\begin{tabular}{l|l|ccc|cccc}
\toprule
\multicolumn{2}{c|}{\multirow{2}{*}{\textbf{Methods}}} & \multirow{2}{*}{\textbf{NFEs}}                    & \textbf{Throughput}   & \textbf{Latency}         & \textbf{Params}          & \multirow{2}{*}{\textbf{FID~$\downarrow$}}                                                                      & \multirow{2}{*}{\textbf{CLIP~$\uparrow$}} & \multirow{2}{*}{\textbf{GenEval~$\uparrow$}}                                     \\
\multicolumn{2}{c|}{}                                  &                                                   & \textbf{(samples/s)}  & \textbf{(s)}             & \textbf{(B)}             &                                                                                                                 &                                           &                                                                                  \\
\midrule
\multicolumn{1}{c|}{\multirow{39}{*}{\rotatebox{90}{\textbf{Few-Step Generative Models}}}}
& SDXL-LCM~\cite{luo2023latent}                     & 2                     & 2.89                     & 0.40                     & \textcolor{color3}{\textbf{0.9}}                                                                                & 18.11                                     & 27.51                                        & 0.44                              \\
& PixArt-LCM~\cite{chen2024pixartdelta}             & 2                     & 3.52                     & 0.31                     & \textcolor{color3}{\textbf{0.6}}                                                                                & 10.33                                     & 27.24                                        & 0.42                              \\
& SD3.5-Turbo~\citep{esser2024scaling}\textsuperscript{\ddag} & 2                     & 1.61                     & 0.68                     & \textcolor{color3}{\textbf{8.0}}                                                                                & 51.47                                     & 25.59                                        & 0.53                              \\
& PCM~\cite{wang2024phased}\dag                     & 2                     & 2.62                     & 0.56                     & \textcolor{color3}{\textbf{0.9}}                                                                                & 14.70                                     & 27.66                                        & 0.55                              \\
& SDXL-DMD2~\cite{yin2024improved}\dag              & 2                     & 2.89                     & 0.40                     & \textcolor{color3}{\textbf{0.9}}                                                                                & \textit{7.61}                             & 28.87                                        & 0.58                              \\
& FLUX-schnell~\citep{FLUX}                         & 2                     & 0.92                     & 1.15                     & \textcolor{color3}{\textbf{12.0}}                                                                               & 7.75                                      & 28.25                                        & \textit{0.71}                     \\
& Sana-Sprint~\citep{chen2025sana}                  & 2                     & 6.46                     & 0.25                     & \textcolor{color3}{\textbf{0.6}}                                                                                & \underline{6.54}                          & \textit{28.40}                               & \underline{0.76}                  \\
& Sana-Sprint~\citep{chen2025sana}                  & 2                     & 5.68                     & 0.24                     & \textcolor{color3}{\textbf{1.6}}                                                                                & 6.50                                      & \underline{28.45}                            & 0.77                              \\
& Qwen-Image-Lightning~\citep{Qwen-Image-Lightning}\textsuperscript{\ddag} & 2                     & 3.15                     & 0.48                     & \textcolor{color5}{\textbf{20.0}} (\textcolor{color2}{r=64},\textcolor{color3}{\textbf{0.4}})\textsuperscript{\S}                     & 6.76                                      & 28.37                                        & 0.85                              \\
& TwinFlow (Qwen-Image)~\citep{cheng2025twinflow}   & 2                     & 3.15                     & 0.48                     & \textcolor{color5}{\textbf{20.0}} (\textcolor{color2}{r=64},\textcolor{color3}{\textbf{0.4}})\textsuperscript{\S}                     & --                                        & --                                           & 0.87                              \\
& APEX (Qwen-Image)~\citep{liu2026self}   & 2 & 3.21 & 0.49 & \textcolor{color5}{\textbf{20.0}} (\textcolor{color2}{r=32},\textcolor{color3}{\textbf{0.2}})\textsuperscript{\S} & 6.72                  & \textit{28.71}           & 0.87          \\
& APEX (Qwen-Image)~\citep{liu2026self}   & 2 & 3.17 & 0.47 & \textcolor{color5}{\textbf{20.0}} (\textcolor{color2}{r=64},\textcolor{color3}{\textbf{0.4}})\textsuperscript{\S} & 6.51                  & \textit{28.42}           & 0.89          \\
& APEX (Qwen-Image)~\citep{liu2026self}   & 2 & 3.30 & 0.45 & \textcolor{color3}{\textbf{20.0}}                                                             & 6.44                  & 28.51                    & \textbf{0.90} \\
\cmidrule{2-9}
& \bf \cellcolor{blue!10}\method (Z-Image) & \cellcolor{blue!10} 2 & \cellcolor{blue!10} 5.18 & \cellcolor{blue!10} 0.30 & \cellcolor{blue!10} \textcolor{color5}{\textbf{6.0}} (\textcolor{color2}{r=32},\textcolor{color3}{\textbf{0.2}})\textsuperscript{\S}              & \cellcolor{blue!10} 6.98   & \cellcolor{blue!10} 28.41   & \cellcolor{blue!10} 0.76 \\
& \bf \cellcolor{blue!10}\method (Z-Image) & \cellcolor{blue!10} 2 & \cellcolor{blue!10} 5.18 & \cellcolor{blue!10} 0.30 & \cellcolor{blue!10} \textcolor{color5}{\textbf{6.0}} (\textcolor{color2}{r=64},\textcolor{color3}{\textbf{0.4}})\textsuperscript{\S}              & \cellcolor{blue!10} 6.79   & \cellcolor{blue!10} 28.13   & \cellcolor{blue!10} 0.77 \\
& \bf \cellcolor{blue!10}\method (Z-Image)             & \cellcolor{blue!10} 2 & \cellcolor{blue!10} 5.18 & \cellcolor{blue!10} 0.30 & \cellcolor{blue!10} \textcolor{color3}{\textbf{6.0}}                                         & \cellcolor{blue!10} 6.71   & \cellcolor{blue!10} 28.22   & \cellcolor{blue!10} 0.78 \\
& \bf \cellcolor{blue!10}\method (Qwen-Image) & \cellcolor{blue!10} 2 & \cellcolor{blue!10} 3.28 & \cellcolor{blue!10} 0.44 & \cellcolor{blue!10} \textcolor{color5}{\textbf{20.0}} (\textcolor{color2}{r=32},\textcolor{color3}{\textbf{0.2}})\textsuperscript{\S} & \cellcolor{blue!10} 6.62   & \cellcolor{blue!10} 28.76   & \cellcolor{blue!10} 0.89 \\
& \bf \cellcolor{blue!10}\method (Qwen-Image) & \cellcolor{blue!10} 2 & \cellcolor{blue!10} 3.28 & \cellcolor{blue!10} 0.44 & \cellcolor{blue!10} \textcolor{color5}{\textbf{20.0}} (\textcolor{color2}{r=64},\textcolor{color3}{\textbf{0.4}})\textsuperscript{\S} & \cellcolor{blue!10} 6.41   & \cellcolor{blue!10} 28.47   & \cellcolor{blue!10} 0.90 \\
& \bf \cellcolor{blue!10}\method (Qwen-Image)                         & \cellcolor{blue!10} 2 & \cellcolor{blue!10} 3.28 & \cellcolor{blue!10} 0.44 & \cellcolor{blue!10} \textcolor{color3}{\textbf{20.0}}                                          & \cellcolor{blue!10} 6.34   & \cellcolor{blue!10} 28.56   & \cellcolor{blue!10} 0.91 \\
\cmidrule{2-9}
& SDXL-LCM~\cite{luo2023latent}                     & 1                     & 3.36                     & 0.32                     & \textcolor{color3}{\textbf{0.9}}                                                                                & 50.51                                     & 24.45                                        & 0.28                              \\
& PixArt-LCM~\cite{chen2024pixartdelta}             & 1                     & 4.26                     & 0.25                     & \textcolor{color3}{\textbf{0.6}}                                                                                & 73.35                                     & 23.99                                        & 0.41                              \\
& PixArt-DMD~\citep{chenpixartsigma,yin2024one}\dag & 1                     & 4.26                     & 0.25                     & \textcolor{color3}{\textbf{0.6}}                                                                                & 9.59                                      & 26.98                                        & 0.45                              \\
& SD3.5-Turbo~\citep{esser2024scaling}\textsuperscript{\ddag} & 1                     & 2.48                     & 0.45                     & \textcolor{color3}{\textbf{8.0}}                                                                                & 52.40                                     & 25.40                                        & 0.51                              \\
& PCM~\cite{wang2024phased}\dag                     & 1                     & 3.16                     & 0.40                     & \textcolor{color3}{\textbf{0.9}}                                                                                & 30.11                                     & 26.47                                        & 0.42                              \\
& SDXL-DMD2~\cite{yin2024improved}\dag              & 1                     & 3.36                     & 0.32                     & \textcolor{color3}{\textbf{0.9}}                                                                                & \underline{7.10}                          & 28.93                                        & 0.59                              \\
& FLUX-schnell~\citep{FLUX}                         & 1                     & 1.58                     & 0.68                     & \textcolor{color3}{\textbf{12.0}}                                                                               & \textit{7.26}                             & \underline{28.49}                            & \textit{0.69}                     \\
& Sana-Sprint~\citep{chen2025sana}                  & 1                     & 7.22                     & 0.21                     & \textcolor{color3}{\textbf{0.6}}                                                                                & 7.04                                      & 28.04                                        & 0.72                              \\
& Sana-Sprint~\citep{chen2025sana}                  & 1                     & 6.71                     & 0.21                     & \textcolor{color3}{\textbf{1.6}}                                                                                & 7.69                                      & \textit{28.27}                               & 0.76                              \\
& Qwen-Image-Lightning~\citep{Qwen-Image-Lightning}\textsuperscript{\ddag} & 1                     & 3.29                     & 0.40                     & \textcolor{color5}{\textbf{20.0}} (\textcolor{color2}{r=64},\textcolor{color3}{\textbf{0.4}})\textsuperscript{\S}                     & 7.06                                      & 28.35                                        & 0.85                              \\
& TwinFlow (Qwen-Image)~\citep{cheng2025twinflow}                & 1                     & 3.29                     & 0.40                     & \textcolor{color5}{\textbf{20.0}} (\textcolor{color2}{r=64},\textcolor{color3}{\textbf{0.4}})\textsuperscript{\S}                     & --                                        & --                                           & 0.86                              \\
& APEX (Qwen-Image)~\citep{liu2026self}   & 1 & 3.29 & 0.39 & \textcolor{color5}{\textbf{20.0}} (\textcolor{color2}{r=32},\textcolor{color3}{\textbf{0.2}})\textsuperscript{\S} & 7.22                  & \textit{28.62}           & 0.88          \\
& APEX (Qwen-Image)~\citep{liu2026self}   & 1 & 3.27 & 0.39 & \textcolor{color5}{\textbf{20.0}} (\textcolor{color2}{r=64},\textcolor{color3}{\textbf{0.4}})\textsuperscript{\S} & 7.14                  & \textit{28.45}           & \textbf{0.89} \\
& APEX (Qwen-Image)~\citep{liu2026self}   & 1 & 3.50 & 0.34 & \textcolor{color3}{\textbf{20.0}}                                                             & 6.87                  & 28.66                    & \textbf{0.89} \\
\cmidrule{2-9}
& \bf \cellcolor{blue!10}\method (Z-Image) & \cellcolor{blue!10} 1 & \cellcolor{blue!10} 5.41 & \cellcolor{blue!10} 0.25 & \cellcolor{blue!10} \textcolor{color5}{\textbf{6.0}} (\textcolor{color2}{r=32},\textcolor{color3}{\textbf{0.2}})\textsuperscript{\S}              & \cellcolor{blue!10} 7.49   & \cellcolor{blue!10} 28.33   & \cellcolor{blue!10} 0.75 \\
& \bf \cellcolor{blue!10}\method (Z-Image) & \cellcolor{blue!10} 1 & \cellcolor{blue!10} 5.41 & \cellcolor{blue!10} 0.25 & \cellcolor{blue!10} \textcolor{color5}{\textbf{6.0}} (\textcolor{color2}{r=64},\textcolor{color3}{\textbf{0.4}})\textsuperscript{\S}              & \cellcolor{blue!10} 7.40   & \cellcolor{blue!10} 28.16   & \cellcolor{blue!10} 0.76 \\
& \bf \cellcolor{blue!10}\method (Z-Image)                           & \cellcolor{blue!10} 1 & \cellcolor{blue!10} 5.41 & \cellcolor{blue!10} 0.25 & \cellcolor{blue!10} \textcolor{color3}{\textbf{6.0}}                                         & \cellcolor{blue!10} 7.13   & \cellcolor{blue!10} 28.37   & \cellcolor{blue!10} 0.77 \\
& \bf \cellcolor{blue!10}\method (Qwen-Image) & \cellcolor{blue!10} 1 & \cellcolor{blue!10} 3.44 & \cellcolor{blue!10} 0.35 & \cellcolor{blue!10} \textcolor{color5}{\textbf{20.0}} (\textcolor{color2}{r=32},\textcolor{color3}{\textbf{0.2}})\textsuperscript{\S} & \cellcolor{blue!10} 7.12   & \cellcolor{blue!10} 28.67   & \cellcolor{blue!10} 0.88 \\
& \bf \cellcolor{blue!10}\method (Qwen-Image) & \cellcolor{blue!10} 1 & \cellcolor{blue!10} 3.44 & \cellcolor{blue!10} 0.35 & \cellcolor{blue!10} \textcolor{color5}{\textbf{20.0}} (\textcolor{color2}{r=64},\textcolor{color3}{\textbf{0.4}})\textsuperscript{\S} & \cellcolor{blue!10} 7.04   & \cellcolor{blue!10} 28.50   & \cellcolor{blue!10} 0.89 \\
& \bf \cellcolor{blue!10}\method (Qwen-Image)                         & \cellcolor{blue!10} 1 & \cellcolor{blue!10} 3.44 & \cellcolor{blue!10} 0.35 & \cellcolor{blue!10} \textcolor{color3}{\textbf{20.0}}                                          & \cellcolor{blue!10} 6.77   & \cellcolor{blue!10} 28.71   & \cellcolor{blue!10} 0.90 \\
\bottomrule
\end{tabular}}
\vspace{-2em}
}
\end{table*}

\subsection{Experimental Setup}
\label{sec:exp_setup}
\paragraph{Backbones and tuning.}
We initialize \method from Z-Image-Turbo 6B~\citep{team2025zimage} and Qwen-Image 20B~\citep{wu2025qwen}, using full-parameter tuning and LoRA with ranks 32 and 64. We denote the resulting models by \method (Z-Image) and \method (Qwen-Image); the former uses the Turbo checkpoint rather than the undistilled Z-Image base model.
\paragraph{Training data.}
The reference data defining $p_r$ combine ShareGPT-4o~\citep{chen2025sharegpt}, BLIP-3o~\citep{chen2025blip3}, and samples synthesized with Qwen-Image-2512 and Z-Image-Turbo: 600K from FLUX-Reason-6M prompts~\citep{fang2025flux} and 5M from text-rendering prompts. Compared with the 200K poster-prompt samples used by APEX~\citep{liu2026self}, this larger text-focused collection provides more training examples for dense text rendering.

\paragraph{Training configuration.}
All models use BF16, FlashAttention4, and FSDP2 with global batch size 64.
Training runs for 15K steps on Z-Image-Turbo 6B and 10K steps for Qwen-Image 20B LoRA, using 32$\times$H100 80GB and 8$\times$A100 80GB GPUs.
The default weights are $\lambda_{\text{score}}=1$ and $\lambda_{\text{pop}}=0.5$, fixed throughout training.
Each trained checkpoint is evaluated at both 1 and 2~NFE without changing the training loss weights for either sampling budget.
\paragraph{Evaluation.}
We follow the same standard benchmark protocols used for the APEX comparisons, reporting each method at its stated NFE budget. Published reference results retain their training data and configuration; the internal \method ablations use the shared configuration above.
We evaluate compositional prompt following with GenEval~\citep{ghosh2023geneval}, FID and CLIP score on MJHQ-30K~\citep{li2024playgroundv2.5}, detailed prompt following on DPG-Bench~\citep{hu2024ella}, and world knowledge on WISE~\citep{niu2025wise}.
Z-Image comparisons and ablations additionally report OneIG~\citep{chang2026oneig} and LongText-Bench~\citep{geng2025x}; LongText-CN and LongText-EN measure Chinese and English long-text rendering, respectively.

\FloatBarrier

\subsection{Main Comparison}
\label{sec:exp_main}

\paragraph{Long-text rendering.}
At 2~NFE, \method (Z-Image) achieves 94.25\% on LongText-CN and 90.20\% on LongText-EN, the highest scores among the compared methods at this budget in~\tabref{tab:ablation_summary_2nfe}. These improve over Z-Image-Turbo++ by 2.63 and 0.32 percentage points and over APEX by 3.87 and 1.44 points, respectively. This directly supports the fine-grained text-rendering objective introduced above. The table places external comparisons alongside the controlled ablations, allowing the contribution of each loss to be examined.

\paragraph{General quality and efficiency.}
With full-parameter tuning, \method achieves GenEval scores of 0.78 on Z-Image and 0.91 on Qwen-Image at 2~NFE in~\tabref{tab:main_comparison}. On Qwen-Image under the same NFE and tuning setting, APEX scores 0.90; FID is 6.34 for \method and 6.44 for APEX, with latencies of 0.44 and 0.45 seconds. Thus, the quality improvement comes at comparable sampling latency. At 1~NFE, Qwen-Image \method reaches 0.90 GenEval at 0.35 seconds per sample, giving a faster sampling option with the same trained checkpoint.
\par
The same full-parameter checkpoints improve from 1 to 2~NFE on both backbones: Z-Image FID decreases from 7.13 to 6.71 and Qwen-Image FID from 6.77 to 6.34, while each gains 0.01 GenEval. Thus the learned objectives remain useful when the sampling budget allows an additional network evaluation.
With rank-64 LoRA, Qwen-Image \method trains 0.4B parameters and reaches 0.90 GenEval at 2~NFE, compared with 0.91 under full tuning. This setting retains most of the compositional performance while updating substantially fewer parameters. DPG-Bench and WISE provide complementary evaluations of prompt following and world knowledge. Figure~\ref{fig:nfe_comparison} presents qualitative comparisons across sampling budgets.

\tabref{tab:geneval}, \tabref{tab:wise}, and~\tabref{tab:dpg} provide the per-category results for GenEval, WISE, and DPG-Bench.

\begin{table*}[!t]\centering
\caption{\textbf{Quantitative evaluation on GenEval.} Per-category and overall scores at the NFE shown in each model label; unlabeled published baselines use their reported default sampling budget. Z-Image is the base model ($\sim$100 NFE), whereas Z-Image-Turbo is the 8-NFE distilled model.}
\resizebox{0.8\linewidth}{!}{
\begin{tabular}{l|cccccc|c}
\toprule
\multirow{2}{*}{\textbf{Model}}                          & \textbf{Single} & \textbf{Two} & \multirow{2}{*}{\textbf{Counting}} & \multirow{2}{*}{\textbf{Colors}} & \multirow{2}{*}{\textbf{Position}} & \textbf{Attribute} & \multirow{2}{*}{\textbf{Overall$\uparrow$}} \\
& \bf Object      & \bf Object   &                                    &                                  &                                    & \bf Binding        &                                             \\
\midrule
Show-o~\citep{xie2024show}                               & 0.95            & 0.52         & 0.49                               & 0.82                             & 0.11                               & 0.28               & 0.53                                        \\
Emu3-Gen~\citep{wang2024emu3}                            & 0.98            & 0.71         & 0.34                               & 0.81                             & 0.17                               & 0.21               & 0.54                                        \\
PixArt-$\alpha$~\citep{chen2024pixartalpha}              & 0.98            & 0.50         & 0.44                               & 0.80                             & 0.08                               & 0.07               & 0.48                                        \\
SD3 Medium~\citep{esser2024scaling}                      & 0.98            & 0.74         & 0.63                               & 0.67                             & 0.34                               & 0.36               & 0.62                                        \\
FLUX.1 [Dev]~\citep{flux2024}                            & 0.98            & 0.81         & 0.74                               & 0.79                             & 0.22                               & 0.45               & 0.66                                        \\
SD3.5 Large~\citep{esser2024scaling}\textsuperscript{\ddag} & 0.98         & 0.89         & 0.73                               & 0.83                             & 0.34                               & 0.47               & 0.71                                        \\
JanusFlow~\citep{ma2025janusflow}                        & 0.97            & 0.59         & 0.45                               & 0.83                             & 0.53                               & 0.42               & 0.63                                        \\
Lumina-Image 2.0~\citep{qin2025lumina}                   & -               & 0.87         & 0.67                               & -                                & -                                  & 0.62               & 0.73                                        \\
Janus-Pro-7B~\citep{chen2025janus}                       & 0.99            & 0.89         & 0.59                               & 0.90                             & 0.79                               & 0.66               & 0.80                                        \\
HiDream-I1-Full~\citep{cai2025hidream}                   & \bf 1.00        & 0.98         & 0.79                               & 0.91                             & 0.60                               & 0.72               & 0.83                                        \\
GPT Image 1 [High]~\citep{gptimage}                      & 0.99            & 0.92         & 0.85                               & 0.92                             & 0.75                               & 0.61               & 0.84                                        \\
Seedream 3.0~\citep{gao2025seedream}                     & 0.99            & \bf 0.96     & 0.91                               & \bf 0.93                         & 0.47                               & 0.80               & 0.84                                        \\
BAGEL~\citep{deng2025bagel}                              & 0.98            & 0.95         & 0.84                               & 0.95                             & 0.78                               & 0.77               & 0.88                                        \\
Qwen-Image~\citep{wu2025qwen}                            & 0.99            & 0.92         & 0.89                               & 0.88                             & 0.76                               & 0.77               & \underline{0.87}                            \\
Hyper-BAGEL~\citep{lu2025hyperbagel}                     & 0.97            & 0.86         & 0.75                               & 0.90                             & 0.67                               & 0.62               & 0.80                                        \\
Qwen-Image-Lightning~\citep{Qwen-Image-Lightning,liu2026self}\textsuperscript{\ddag} & 0.99 & 0.89 & 0.85                               & 0.87                             & 0.75                               & 0.76               & 0.85                                        \\
Z-Image~\citep{team2025zimage} (base; $\sim$100 NFE)      & 1.00            & 0.94         & 0.78                               & 0.93                             & 0.62                               & 0.77               & 0.84                                        \\
Z-Image-Turbo~\citep{team2025zimage,liu2026high} & 1.00            & 0.89         & 0.72                               & 0.85                             & 0.46                               & 0.59               & 0.75                                        \\
DMD2~\citep{yin2024improved,liu2026high} (2 NFE) & 1.00           & 0.90         & 0.68                               & 0.84                             & 0.49                               & 0.66               & 0.76                                        \\
TwinFlow~\citep{cheng2025twinflow,liu2026high} (2 NFE) & 1.00       & 0.87         & 0.68                               & 0.85                             & 0.40                               & 0.55               & 0.72                                        \\
Z-Image-Turbo++~\citep{liu2026high} (2 NFE)      & 1.00            & 0.88         & 0.69                               & 0.85                             & 0.49                               & 0.65               & 0.76                                        \\
TwinFlow (Qwen-Image)~\citep{cheng2025twinflow}  (1 NFE)              & 1.00            & 0.91         & 0.84                               & 0.90                             & 0.75                               & 0.74               & 0.86                                        \\
APEX (Qwen-Image) (LoRA\&\textcolor{color2}{r=32}) (1 NFE)    & 0.99            & 0.95         & 0.85                               & 0.90                             & 0.79                               & 0.78               & 0.88                                        \\
APEX (Qwen-Image) (LoRA\&\textcolor{color2}{r=64}) (1 NFE)    & 0.99            & 0.94         & 0.88                               & 0.90                             & \bf 0.85                           & \bf 0.78           & \bf 0.89                                    \\
APEX (Qwen-Image) (\textcolor{color3}{Full}) (1 NFE)           & 0.99            & 0.92         & 0.83                               & 0.91                             & 0.86                               & 0.81               & \textbf{0.89}                               \\
\midrule
\rowcolor{blue!10}
\bf \method (Z-Image) (LoRA\&\textcolor{color2}{r=32}) (2 NFE) & 1.00 & 0.90 & 0.69 & 0.87 & 0.47 & 0.62 & 0.76 \\
\rowcolor{blue!10}
\bf \method (Z-Image) (LoRA\&\textcolor{color2}{r=64}) (2 NFE) & 1.00 & 0.90 & 0.70 & 0.87 & 0.49 & 0.64 & 0.77 \\
\rowcolor{blue!10}
\bf \method (Z-Image) (\textcolor{color3}{Full}) (2 NFE)              & 1.00 & 0.91 & 0.71 & 0.88 & 0.51 & 0.66 & 0.78 \\
\rowcolor{blue!10}
\bf \method (Z-Image) (LoRA\&\textcolor{color2}{r=32}) (1 NFE) & 1.00 & 0.89 & 0.68 & 0.86 & 0.45 & 0.61 & 0.75 \\
\rowcolor{blue!10}
\bf \method (Z-Image) (LoRA\&\textcolor{color2}{r=64}) (1 NFE) & 1.00 & 0.90 & 0.69 & 0.87 & 0.47 & 0.63 & 0.76 \\
\rowcolor{blue!10}
\bf \method (Z-Image) (\textcolor{color3}{Full}) (1 NFE)                               & 1.00 & 0.90 & 0.70 & 0.87 & 0.49 & 0.64 & 0.77 \\
\rowcolor{blue!10}
\bf \method (Qwen-Image) (LoRA\&\textcolor{color2}{r=32}) (2 NFE) & 1.00 & 0.95 & 0.87 & 0.92 & 0.79 & 0.83 & 0.89 \\
\rowcolor{blue!10}
\bf \method (Qwen-Image) (LoRA\&\textcolor{color2}{r=64}) (2 NFE) & 1.00 & 0.95 & 0.88 & 0.92 & 0.81 & 0.84 & 0.90 \\
\rowcolor{blue!10}
\bf \method (Qwen-Image) (\textcolor{color3}{Full}) (2 NFE)                               & 1.00 & 0.96 & 0.89 & 0.93 & 0.83 & 0.86 & 0.91 \\
\rowcolor{blue!10}
\bf \method (Qwen-Image) (LoRA\&\textcolor{color2}{r=32}) (1 NFE) & 1.00 & 0.94 & 0.86 & 0.91 & 0.77 & 0.81 & 0.88 \\
\rowcolor{blue!10}
\bf \method (Qwen-Image) (LoRA\&\textcolor{color2}{r=64}) (1 NFE) & 1.00 & 0.95 & 0.87 & 0.92 & 0.79 & 0.83 & 0.89 \\
\rowcolor{blue!10}
\bf \method (Qwen-Image) (\textcolor{color3}{Full}) (1 NFE)                               & 1.00 & 0.95 & 0.88 & 0.92 & 0.81 & 0.84 & 0.90 \\
\bottomrule
\end{tabular}\label{tab:geneval}
}
\end{table*}

\begin{table*}[!t]
\centering
\caption{\textbf{Quantitative evaluation on WISE~\citep{niu2025wise}.} Per-domain and overall world-knowledge grounding scores.}
\resizebox{0.8\textwidth}{!}{
\begin{tabular}{l|cccccc|c}
\toprule
\textbf{Model}                                             & \bf Cultural & \bf Time & \bf Space & \bf Biology & \bf Physics & \bf Chemistry & \bf Overall$\uparrow$ \\
\midrule
SD v1.5~\citep{rombach2022high}                            & 0.34         & 0.35     & 0.32      & 0.28        & 0.29        & 0.21          & 0.32                  \\
SDXL~\citep{podell2023sdxl}                                & 0.43         & 0.48     & 0.47      & 0.44        & 0.45        & 0.27          & 0.43                  \\
SD3.5-Large~\citep{esser2024scaling}\textsuperscript{\ddag} & 0.44        & 0.50     & 0.58      & 0.44        & 0.52        & 0.31          & 0.46                  \\
PixArt-$\alpha$~\citep{chen2024pixartalpha}                & 0.45         & 0.50     & 0.48      & 0.49        & 0.56        & 0.34          & 0.47                  \\
Playground-v2.5~\citep{li2024playgroundv2.5}               & 0.49         & 0.58     & 0.55      & 0.43        & 0.48        & 0.33          & 0.49                  \\
FLUX.1 [Dev]~\citep{flux2024}                              & 0.48         & 0.58     & 0.62      & 0.42        & 0.51        & 0.35          & 0.50                  \\
Janus~\citep{wu2025janus}                                  & 0.16         & 0.26     & 0.35      & 0.28        & 0.30        & 0.14          & 0.23                  \\
VILA-u-7B-256~\citep{wu2024vila}                           & 0.26         & 0.33     & 0.37      & 0.35        & 0.39        & 0.23          & 0.31                  \\
Show-o~\citep{xie2024show}                                 & 0.28         & 0.40     & 0.48      & 0.30        & 0.46        & 0.30          & 0.35                  \\
Janus-Pro-7B~\citep{chen2025janus}                         & 0.30         & 0.37     & 0.49      & 0.36        & 0.42        & 0.26          & 0.35                  \\
Emu3-Gen~\citep{wang2024emu3}                              & 0.34         & 0.45     & 0.48      & 0.41        & 0.45        & 0.27          & 0.39                  \\
MetaQuery-XL~\citep{pan2025transfer}                       & 0.56         & 0.55     & 0.62      & 0.49        & 0.63        & 0.41          & 0.55                  \\
BAGEL~\citep{deng2025bagel}                                & 0.44         & 0.55     & 0.68      & 0.44        & 0.60        & 0.39          & 0.52                  \\
GPT-Image-1~\citep{gptimage}                               & 0.81         & 0.71     & 0.89      & 0.83        & 0.79        & 0.74          & \textbf{0.80}         \\
TwinFlow (Qwen-Image; 1 NFE)~\citep{cheng2025twinflow}     & --           & --       & --        & --          & --          & --            & 0.54                  \\
APEX (Qwen-Image) (LoRA\&\textcolor{color2}{r=32}) (1 NFE) & --          & --       & --        & --          & --          & --            & --                    \\
APEX (Qwen-Image) (LoRA\&\textcolor{color2}{r=64}) (1 NFE) & --          & --       & --        & --          & --          & --            & --                    \\
APEX (Qwen-Image) (\textcolor{color3}{Full}) (1 NFE)      & 0.53         & 0.54     & 0.66      & 0.48        & 0.61        & 0.41          & 0.54                  \\
\midrule
\rowcolor{blue!10}
\bf \method (Z-Image) (LoRA\&\textcolor{color2}{r=32}) (2 NFE) & 0.45        & 0.46       & 0.56        & 0.41          & 0.52          & 0.37            & 0.46                    \\
\rowcolor{blue!10}
\bf \method (Z-Image) (LoRA\&\textcolor{color2}{r=64}) (2 NFE) & 0.46        & 0.47       & 0.58        & 0.42          & 0.54          & 0.37            & 0.47                    \\
\rowcolor{blue!10}
\bf \method (Z-Image) (\textcolor{color3}{Full}) (2 NFE)    & 0.47           & 0.48       & 0.59        & 0.43          & 0.55          & 0.38            & 0.48                    \\
\rowcolor{blue!10}
\bf \method (Z-Image) (LoRA\&\textcolor{color2}{r=32}) (1 NFE) & 0.43        & 0.44       & 0.55        & 0.40          & 0.51          & 0.36            & 0.45                    \\
\rowcolor{blue!10}
\bf \method (Z-Image) (LoRA\&\textcolor{color2}{r=64}) (1 NFE) & 0.44        & 0.45       & 0.56        & 0.41          & 0.52          & 0.36            & 0.46                    \\
\rowcolor{blue!10}
\bf \method (Z-Image) (\textcolor{color3}{Full}) (1 NFE)    & 0.46           & 0.47       & 0.58        & 0.42          & 0.54          & 0.37            & 0.47                    \\
\rowcolor{blue!10}
\bf \method (Qwen-Image) (LoRA\&\textcolor{color2}{r=32}) (2 NFE) & 0.53     & 0.54       & 0.65        & 0.48          & 0.60          & 0.41            & 0.54                      \\
\rowcolor{blue!10}
\bf \method (Qwen-Image) (LoRA\&\textcolor{color2}{r=64}) (2 NFE) & 0.54     & 0.55       & 0.67        & 0.49          & 0.62          & 0.41            & 0.55                      \\
\rowcolor{blue!10}
\bf \method (Qwen-Image) (\textcolor{color3}{Full}) (2 NFE) & 0.55           & 0.56       & 0.68        & 0.50          & 0.63          & 0.42            & 0.56                      \\
\rowcolor{blue!10}
\bf \method (Qwen-Image) (LoRA\&\textcolor{color2}{r=32}) (1 NFE) & 0.51     & 0.52       & 0.64        & 0.47          & 0.59          & 0.40            & 0.52                      \\
\rowcolor{blue!10}
\bf \method (Qwen-Image) (LoRA\&\textcolor{color2}{r=64}) (1 NFE) & 0.53     & 0.54       & 0.65        & 0.48          & 0.60          & 0.41            & 0.54                      \\
\rowcolor{blue!10}
\bf \method (Qwen-Image) (\textcolor{color3}{Full}) (1 NFE) & 0.54           & 0.55       & 0.67        & 0.49          & 0.62          & 0.41            & 0.55                      \\
\bottomrule
\end{tabular}\label{tab:wise}
}
\end{table*}

\FloatBarrier

\begin{table*}[!t]
\centering
\caption{\textbf{Quantitative evaluation on DPG-Bench.} NFE notation follows \protect\tabref{tab:geneval}.}
\resizebox{0.8\linewidth}{!}{
\begin{tabular}{l|ccccc|c}
\toprule
\textbf{Model}                                             & \bf Global     & \bf Entity     & \bf Attribute  & \bf Relation   & \bf Other      & \bf Overall$\uparrow$ \\
\midrule
SD v1.5~\citep{rombach2022high}                            & 74.63          & 74.23          & 75.39          & 73.49          & 67.81          & 63.18                 \\
PixArt-$\alpha$~\citep{chen2024pixartalpha}                & 74.97          & 79.32          & 78.60          & 82.57          & 76.96          & 71.11                 \\
Lumina-Next~\citep{zhuo2024lumina}                         & 82.82          & 88.65          & 86.44          & 80.53          & 81.82          & 74.63                 \\
SDXL~\citep{podell2023sdxl}                                & 83.27          & 82.43          & 80.91          & 86.76          & 80.41          & 74.65                 \\
Hunyuan-DiT~\citep{li2024hunyuan-dit}                      & 84.59          & 80.59          & 88.01          & 74.36          & 86.41          & 78.87                 \\
Janus~\citep{wu2025janus}                                  & 82.33          & 87.38          & 87.70          & 85.46          & 86.41          & 79.68                 \\
PixArt-$\Sigma$~\citep{chen2024pixartsigma}                & 86.89          & 82.89          & 88.94          & 86.59          & 87.68          & 80.54                 \\
Emu3-Gen~\citep{wang2024emu3}                              & 85.21          & 86.68          & 86.84          & 90.22          & 83.15          & 80.60                 \\
Janus-Pro-1B~\citep{chen2025janus}                         & 87.58          & 88.63          & 88.17          & 88.98          & 88.30          & 82.63                 \\
DALL-E 3~\citep{Dalle-3}                                   & 90.97          & 89.61          & 88.39          & 90.58          & 89.83          & 83.50                 \\
FLUX.1 [Dev]~\citep{flux2024}                              & 74.35          & 90.00          & 88.96          & 90.87          & 88.33          & 83.84                 \\
SD3 Medium~\citep{esser2024scaling}                       & 87.90          & 91.01          & 88.83          & 80.70          & 88.68          & 84.08                 \\
SD3.5-Turbo~\citep{esser2024scaling}\textsuperscript{\ddag} & 79.03        & 80.12          & 86.13          & 84.73          & 91.86          & 78.29                 \\
SD3.5-Large~\citep{esser2024scaling}\textsuperscript{\ddag} & 83.21        & 84.27          & 88.99          & 87.35          & 93.28          & 80.35                 \\
FLUX.1-schnell~\cite{FLUX}                                 & 84.94          & 86.62          & 90.82          & 88.35          & 93.45          & 82.00                 \\
Janus-Pro-7B~\citep{chen2025janus}                         & 86.90          & 88.90          & 89.40          & 89.32          & 89.48          & 84.19                 \\
HiDream-I1-Full~\citep{cai2025hidream}                     & 76.44          & 90.22          & 89.48          & 93.74          & 91.83          & 85.89                 \\
Lumina-Image 2.0~\citep{qin2025lumina}                     & -              & 91.97          & 90.20          & \textbf{94.85} & -              & 87.20                 \\
Seedream 3.0~\citep{gao2025seedream}                       & \textbf{94.31} & \textbf{92.65} & 91.36          & 92.78          & 88.24          & 88.27                 \\
GPT Image 1 [High]~\citep{gptimage}                        & 88.89          & 88.94          & 89.84          & 92.63          & 90.96          & 85.15                 \\
Qwen-Image~\citep{wu2025qwen}                              & 91.32          & 91.56          & \textbf{92.02} & 94.31          & \textbf{92.73} & 88.32                 \\
Z-Image~\citep{team2025zimage} (base; $\sim$100 NFE)       & 93.39          & 91.22          & 93.16          & 92.22          & 91.52          & 88.14                 \\
Z-Image-Turbo~\citep{team2025zimage,liu2026high} & 91.29          & 89.59          & 90.14          & 92.16          & 88.68          & 84.86                 \\
DMD2~\citep{yin2024improved,liu2026high} (2 NFE) & 83.59          & 91.10          & 88.88          & 93.81          & 85.60          & 85.40                 \\
TwinFlow~\citep{cheng2025twinflow,liu2026high} (2 NFE) & 81.46    & 91.83          & 89.02          & 94.16          & 87.60          & 85.98                 \\
Z-Image-Turbo++~\citep{liu2026high} (2 NFE)       & 83.89          & 91.84          & 89.30          & 94.58          & 89.60          & 85.86                 \\
TwinFlow (Qwen-Image)~\citep{cheng2025twinflow}  (1 NFE)    & --             & --             & --             & --             & --             & 86.52                 \\
APEX (Qwen-Image) (LoRA\&\textcolor{color2}{r=32}) (1 NFE)  & 93.12          & 90.95          & 91.38          & 90.65          & 91.73          & 86.17                 \\
APEX (Qwen-Image) (LoRA\&\textcolor{color2}{r=64}) (1 NFE)  & 92.46          & 91.14          & 90.71          & 91.30          & 91.98          & 85.77                 \\
APEX (Qwen-Image) (\textcolor{color3}{Full}) (1 NFE)        & 93.25          & 89.76          & 90.65          & 91.17          & 90.75          & 84.59                 \\
\midrule
\rowcolor{blue!10}
\bf \method (Z-Image) (LoRA\&\textcolor{color2}{r=32}) (2 NFE) & 80.04          & 92.78          & 89.47          & 94.43          & 89.46          & 86.03                 \\
\rowcolor{blue!10}
\bf \method (Z-Image) (LoRA\&\textcolor{color2}{r=64}) (2 NFE) & 80.46          & 92.94          & 89.59          & 94.52          & 89.74          & 86.44                 \\
\rowcolor{blue!10}
\bf \method (Z-Image) (\textcolor{color3}{Full}) (2 NFE) & 80.85          & 93.08          & 89.70          & 94.60          & 90.00          & 86.82                 \\
\rowcolor{blue!10}
\bf \method (Z-Image) (LoRA\&\textcolor{color2}{r=32}) (1 NFE) & 79.56          & 92.60          & 89.34          & 94.32          & 89.14          & 85.56                 \\
\rowcolor{blue!10}
\bf \method (Z-Image) (LoRA\&\textcolor{color2}{r=64}) (1 NFE) & 79.96          & 92.75          & 89.45          & 94.41          & 89.40          & 85.95                 \\
\rowcolor{blue!10}
\bf \method (Z-Image) (\textcolor{color3}{Full}) (1 NFE)      & 80.33          & 92.89          & 89.55          & 94.49          & 89.65          & 86.31                 \\
\rowcolor{blue!10}
\bf \method (Qwen-Image) (LoRA\&\textcolor{color2}{r=32}) (2 NFE) & 93.29       & 90.06          & 91.02          & 91.38          & 92.02          & 86.44                 \\
\rowcolor{blue!10}
\bf \method (Qwen-Image) (LoRA\&\textcolor{color2}{r=64}) (2 NFE) & 93.68       & 90.20          & 91.13          & 91.46          & 92.28          & 86.82                 \\
\rowcolor{blue!10}
\bf \method (Qwen-Image) (\textcolor{color3}{Full}) (2 NFE)   & 94.03          & 90.33          & 91.23          & 91.53          & 92.52          & 87.16                 \\
\rowcolor{blue!10}
\bf \method (Qwen-Image) (LoRA\&\textcolor{color2}{r=32}) (1 NFE) & 92.74       & 89.86          & 90.86          & 91.26          & 91.66          & 85.91                 \\
\rowcolor{blue!10}
\bf \method (Qwen-Image) (LoRA\&\textcolor{color2}{r=64}) (1 NFE) & 93.12       & 90.00          & 90.97          & 91.34          & 91.92          & 86.28                 \\
\rowcolor{blue!10}
\bf \method (Qwen-Image) (\textcolor{color3}{Full}) (1 NFE)   & 93.46          & 90.12          & 91.07          & 91.41          & 92.14          & 86.61                 \\
\bottomrule
\end{tabular}\label{tab:dpg}
}
\end{table*}

\subsection{Key Ablations}
\label{sec:exp_ablation}
All controlled ablations use \method (Z-Image) with full-parameter tuning at 2~NFE and the shared training configuration above. The default combines equal-weight DMD velocity and endpoint regression with $\cL_{\text{pop}}$ using generator EMA features.

\begin{table}[t!]
\centering
\caption{\textbf{Overall ablation results across benchmarks (2~NFE).} Scores are percentages. Row~\circnum{1} is the full configuration and the shared control; rows~\circnum{2}--\circnum{11} vary the losses, DMD regression weights, or feature map. The $v$-only and $x$-only controls retain $\cL_{\text{pop}}$. APEX~\citep{liu2026self} is the published Z-Image reference from its comparison with TwinFlow; rows~\circnum{12}--\circnum{16} are from \citet{liu2026high}, with row~\circnum{16} evaluated at 8~NFE.}
\label{tab:ablation_summary_2nfe}
\resizebox{0.9\linewidth}{!}{
\begin{tabular}{clcccccc}
\toprule
Idx & Model & NFE & OneIG & GenEval & DPG-Bench & LongText-CN & LongText-EN \\
\midrule
\firsttablegroup{8}{Distribution Matching Loss Ablation}
\circnum{2}  & \phantom{00}w/o $\cL_{\text{score}}$ (DMD loss) & 2 & 52.58 & 76.90 & 86.15 & 93.33 & 89.62 \\
\circnum{3}  & \phantom{00}w/o $\cL_{\text{pop}}$ (feature-moment loss) & 2 & 52.87 & 77.46 & 86.54 & 93.87 & 89.79 \\
\circnum{4}  & \phantom{00}w/o $\cL_{\text{score}}$ and $\cL_{\text{pop}}$ & 2 & 52.25 & 76.24 & 85.68 & 92.70 & 88.88 \\
\addlinespace[1pt]
\tablegroup{8}{DMD Regression Ablation}
\circnum{5}  & \phantom{00}$v$-loss only  & 2 & 52.43 & 76.64 & 85.93 & 93.03 & 89.17 \\
\circnum{6}  & \phantom{00}$x$-loss only  & 2 & 52.78 & 77.26 & 86.42 & 93.71 & 89.74 \\
\circnum{7}  & \phantom{00}2$\times$$x$-loss+$v$-loss & 2 & 53.01 & 77.71 & 86.75 & 94.17 & 90.13 \\
\circnum{8}  & \phantom{00}$x$-loss+2$\times$$v$-loss & 2 & 52.95 & 77.59 & 86.66 & 94.05 & 90.04 \\
\addlinespace[1pt]
\tablegroup{8}{Feature-Map Ablation}
\circnum{9}  & \phantom{00}DINOv2            & 2 & 53.00 & 77.75 & 86.71 & 94.13 & 90.09 \\
\circnum{10}  & \phantom{00}Teacher features    & 2 & 52.82 & 77.36 & 86.49 & 93.79 & 89.81 \\
\circnum{11}  & \phantom{00}Live self-features & 2 & 52.91 & 77.51 & 86.60 & 93.95 & 89.95 \\
\addlinespace[1pt]
\tablegroup{8}{Comparative}
\circnum{1}  & IDM & 2 & 53.06 & 77.81 & 86.82 & 94.25 & 90.20 \\
\circnum{15} & Z-Image-Turbo     & 2 & 50.94 & 76.53 & 85.78 & 78.48 & 72.89 \\
\circnum{13} & DMD2              & 2 & 50.70 & 76.12 & 85.40 & 85.30 & 80.99 \\
\circnum{14} & TwinFlow & 2 & 51.38 & 72.41 & 85.98 & 78.65 & 71.99 \\
\circnum{A}  & APEX & 2 & 52.60 & 77.23 & 88.64 & 90.38 & 88.76 \\
\circnum{12} & Z-Image-Turbo++   & 2 & 52.50 & 75.70 & 85.86 & 91.62 & 89.88 \\
\addlinespace[1pt]
\tablegroup{8}{Higher-NFE baseline}
\circnum{16} & Z-Image-Turbo & 8 & 52.84 & 75.01 & 84.86 & 92.56 & 91.74 \\
\bottomrule
\end{tabular}
}
\vspace{-1em}
\end{table}

\paragraph{Added value of feature-moment matching.}
Adding $\cL_{\text{pop}}$ to the shared DMD control raises GenEval from 77.46 to 77.81, LongText-CN from 93.87 to 94.25, and LongText-EN from 89.79 to 90.20, with gains on all five aggregate metrics. Removing $\cL_{\text{score}}$ instead gives 76.90 GenEval; removing both losses gives 76.24. The other two aggregate metrics also improve: OneIG rises from 52.87 to 53.06 and DPG-Bench from 86.54 to 86.82. The GenEval breakdown further identifies the gains from $\cL_{\text{pop}}$: color-attribute binding rises from 65.25 to 66.25, position from 50.75 to 51.25, and counting from 70.94 to 71.25, with unchanged single- and two-object scores in~\tabref{tab:ablation_geneval_2nfe}. Removing DMD gives a larger GenEval decrease (0.91 points) than removing moment matching (0.35 points). Moment matching improves this DMD control while the score correction remains central.

\paragraph{Generator EMA features.}
Moment matching with DINOv2 or generator EMA features improves all five metrics over the internal DMD control. From the same control without $\cL_{\text{pop}}$, DINOv2 raises GenEval from 77.46 to 77.75 (+0.29 points), while EMA reaches 77.81 (+0.35). Their LongText-CN scores are close: 94.13 and 94.25. Live self-features and teacher features reach 77.51 and 77.36 GenEval. EMA features add moment supervision beyond internal DMD, with performance close to DINOv2 and no external feature encoder.

\paragraph{Velocity and endpoint regression.}
To select how the DMD correction supervises predictions, we compare $v$ regression, $x$ regression, and their combinations while retaining $\cL_{\text{pop}}$.
\par
Both single-regression controls therefore include the added moment constraint, so this block tests how to apply DMD within the combined objective. Since $\ell_x(t)=t^2\ell_v(t)$ in Eq.~\ref{eq:l_score}, these choices change time weighting and update magnitude. Equal-weight $v{+}x$ reaches 77.81 GenEval, versus 76.64 for $v$ only and 77.26 for $x$ only; $2x{+}v$ and $x{+}2v$ reach 77.71 and 77.59. Equal weights lead on the five aggregate metrics and are our empirical default, although $x{+}2v$ gives a higher position score, 52.00 versus 51.25. The mixed losses apply alternative weightings to the same detached correction; these results do not establish independent distributional information in the two regression terms.

\paragraph{Relation to APEX and the internal controls.}
\method builds on APEX's internal velocity-difference construction~\citep{liu2026self}, adding endpoint regression and EMA feature-moment matching in~\figref{fig:idm_ablation_ladder}. The published APEX reference uses a smaller training collection, so the comparison reflects the complete training settings. \method scores 90.20 on LongText-EN versus APEX's 88.76, while APEX leads on DPG-Bench, 88.64 versus 86.82. Within the \method training setting, the loss-removal and feature-map controls separate these objective choices from differences in training data.

\begin{figure}[!t]
\centering
\includegraphics[width=\linewidth]{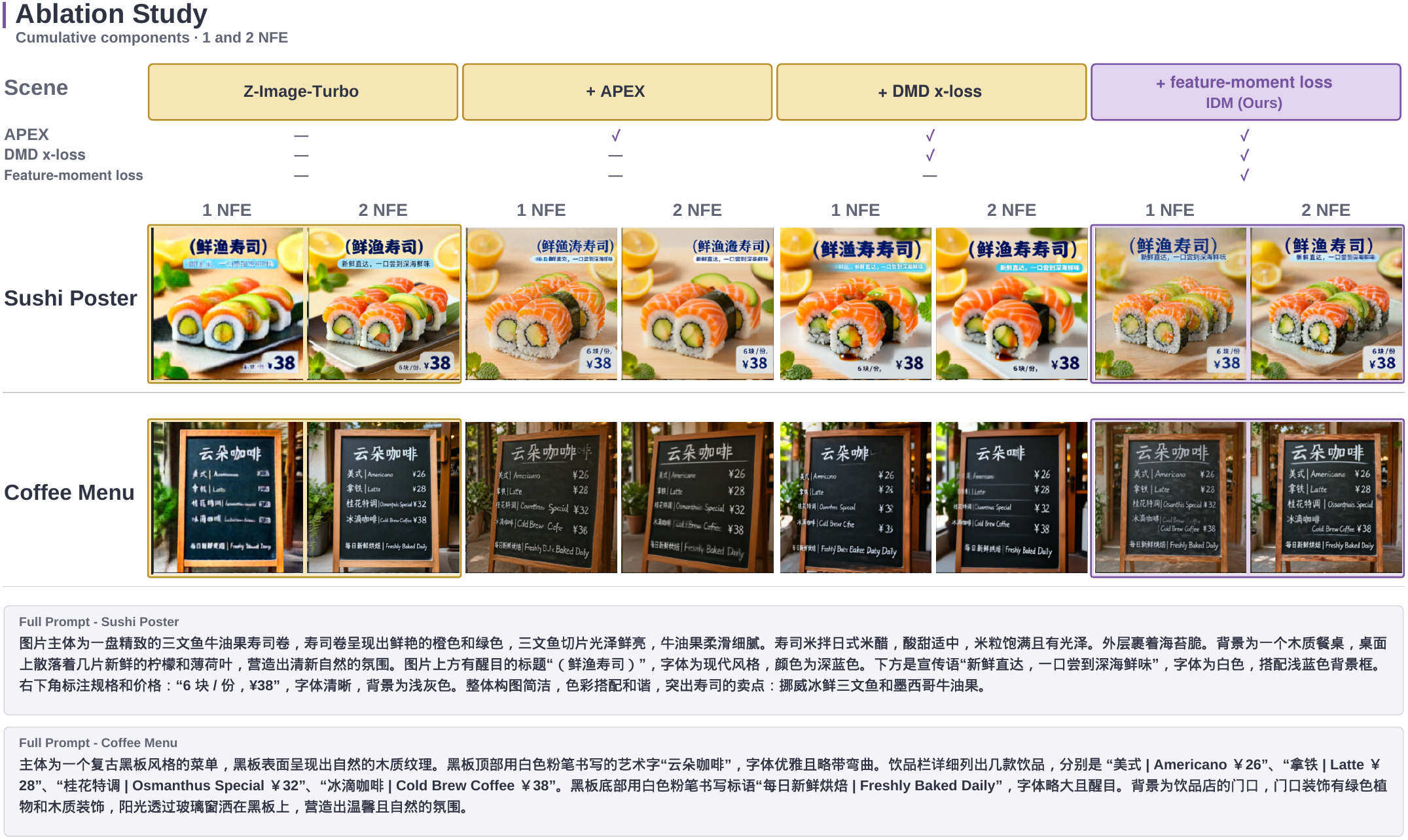}
\caption{\textbf{Reference methods and \method loss configurations.} The two prompts are evaluated at 1 and 2~NFE. The $v{+}x$ control omits $\cL_{\text{pop}}$, while full \method includes EMA feature-moment matching. APEX is an external reference trained on a smaller collection; the \method loss ablations share the same training setting.}
\label{fig:idm_ablation_ladder}
\vspace{-2em}
\end{figure}

\paragraph{Qualitative comparison.}
In the examples of \figref{fig:idm_ablation_ladder}, full \method shows more consistent preservation of dense text across the displayed sampling budgets, particularly in the multi-line menu. The menu combines a prominent title with several rows of small characters, making it a demanding test of text rendering beyond a few large words. Its entries pair Chinese and English drink names, including Americano and Latte, with prices. Each row combines both scripts with numbers and currency symbols, all within the compact layout of a chalkboard menu. The $v{+}x$ control without $\cL_{\text{pop}}$ produces more saturated colors, most visibly in the orange and yellow regions of the sushi poster. The LongText results quantify text fidelity across the benchmark.

\FloatBarrier

\tabref{tab:ablation_geneval_2nfe}, \tabref{tab:ablation_dpg_2nfe}, and~\tabref{tab:ablation_oneig_2nfe} expand the benchmark scores in \tabref{tab:ablation_summary_2nfe} into their individual categories.

\begin{table}[!htbp]
\centering
\caption{\textbf{Detailed GenEval ablation results.} Scores are percentages; Overall is the mean of the six category scores. Configurations, row groups, and sources follow \protect\tabref{tab:ablation_summary_2nfe}.}
\label{tab:ablation_geneval_2nfe}
\resizebox{0.9\linewidth}{!}{
\begin{tabular}{clcccccccc}
\toprule
Idx & Model & NFE & Overall & Two obj. & Color attr. & Position & Counting & Colors & Single obj. \\
\midrule
\firsttablegroup{10}{Distribution Matching Loss Ablation}
\circnum{2}  & \phantom{00}w/o $\cL_{\text{score}}$ (DMD loss) & 2 & 76.90 & 90.40 & 64.50 & 49.25 & 70.31 & 87.23 & 99.69 \\
\circnum{3}  & \phantom{00}w/o $\cL_{\text{pop}}$ (feature-moment loss) & 2 & 77.46 & 90.91 & 65.25 & 50.75 & 70.94 & 87.23 & 99.69 \\
\circnum{4}  & \phantom{00}w/o $\cL_{\text{score}}$ and $\cL_{\text{pop}}$ & 2 & 76.24 & 90.15 & 63.75 & 49.25 & 68.44 & 86.17 & 99.69 \\
\addlinespace[1pt]
\tablegroup{10}{DMD Regression Ablation}
\circnum{5}  & \phantom{00}$v$-loss only  & 2 & 76.64 & 90.40 & 64.00 & 50.00 & 69.06 & 86.70 & 99.69 \\
\circnum{6}  & \phantom{00}$x$-loss only  & 2 & 77.26 & 90.66 & 65.50 & 50.50 & 70.00 & 87.23 & 99.69 \\
\circnum{7}  & \phantom{00}2$\times$$x$-loss+$v$-loss & 2 & 77.71 & 91.16 & 65.75 & 51.25 & 70.63 & 87.77 & 99.69 \\
\circnum{8}  & \phantom{00}$x$-loss+2$\times$$v$-loss & 2 & 77.59 & 91.16 & 66.00 & 52.00 & 70.00 & 86.70 & 99.69 \\
\addlinespace[1pt]
\tablegroup{10}{Feature-Map Ablation}
\circnum{9}  & \phantom{00}DINOv2            & 2 & 77.75 & 91.16 & 65.75 & 51.75 & 70.94 & 87.23 & 99.69 \\
\circnum{10}  & \phantom{00}Teacher features    & 2 & 77.36 & 90.66 & 65.50 & 50.75 & 70.31 & 87.23 & 99.69 \\
\circnum{11}  & \phantom{00}Live self-features & 2 & 77.51 & 90.91 & 66.00 & 51.25 & 70.00 & 87.23 & 99.69 \\
\addlinespace[1pt]
\tablegroup{10}{Comparative}
\circnum{1}  & IDM & 2 & 77.81 & 90.91 & 66.25 & 51.25 & 71.25 & 87.50 & 99.69 \\
\circnum{15} & Z-Image-Turbo     & 2 & 76.53 & 89.39 & 57.75 & 50.50 & 74.38 & 87.77 & 99.38 \\
\circnum{13} & DMD2              & 2 & 76.12 & 90.40 & 65.50 & 49.00 & 67.50 & 84.31 & 100.00 \\
\circnum{14} & TwinFlow & 2 & 72.41 & 87.12 & 54.75 & 40.00 & 67.81 & 85.11 & 99.69 \\
\circnum{A}  & APEX & 2 & 77.23 & 90.66 & 65.50 & 50.25 & 70.31 & 86.97 & 99.69 \\
\circnum{12} & Z-Image-Turbo++   & 2 & 75.70 & 87.63 & 64.50 & 48.75 & 69.06 & 84.57 & 99.69 \\
\addlinespace[1pt]
\tablegroup{10}{Higher-NFE baseline}
\circnum{16} & Z-Image-Turbo & 8 & 75.01 & 88.89 & 58.75 & 45.75 & 71.88 & 85.11 & 99.69 \\
\bottomrule
\end{tabular}
}
\end{table}

\begin{table}[!htbp]
\centering
\caption{\textbf{Detailed DPG-Bench ablation results.} Scores are percentages. Overall follows the standard prompt-level aggregation over the benchmark question hierarchy, rather than the mean of the five category columns. Configurations, row groups, and sources follow \protect\tabref{tab:ablation_summary_2nfe}.}
\label{tab:ablation_dpg_2nfe}
\resizebox{0.9\linewidth}{!}{
\begin{tabular}{clccccccc}
\toprule
Idx & Model & NFE & Overall & Attribute & Global & Relation & Other & Entity \\
\midrule
\firsttablegroup{9}{Distribution Matching Loss Ablation}
\circnum{2}  & \phantom{00}w/o $\cL_{\text{score}}$ (DMD loss) & 2 & 86.15 & 89.33 & 80.23 & 94.32 & 89.54 & 92.58 \\
\circnum{3}  & \phantom{00}w/o $\cL_{\text{pop}}$ (feature-moment loss) & 2 & 86.54 & 89.54 & 79.94 & 94.48 & 89.41 & 92.88 \\
\circnum{4}  & \phantom{00}w/o $\cL_{\text{score}}$ and $\cL_{\text{pop}}$ & 2 & 85.68 & 89.07 & 80.01 & 94.12 & 89.37 & 92.23 \\
\addlinespace[1pt]
\tablegroup{9}{DMD Regression Ablation}
\circnum{5}  & \phantom{00}$v$-loss only  & 2 & 85.93 & 89.21 & 79.45 & 94.22 & 89.05 & 92.42 \\
\circnum{6}  & \phantom{00}$x$-loss only  & 2 & 86.42 & 89.48 & 79.56 & 94.43 & 89.15 & 92.78 \\
\circnum{7}  & \phantom{00}2$\times$$x$-loss+$v$-loss & 2 & 86.75 & 89.68 & 81.62 & 94.57 & 90.44 & 93.04 \\
\circnum{8}  & \phantom{00}$x$-loss+2$\times$$v$-loss & 2 & 86.66 & 89.64 & 80.79 & 94.55 & 89.92 & 92.99 \\
\addlinespace[1pt]
\tablegroup{9}{Feature-Map Ablation}
\circnum{9}  & \phantom{00}DINOv2            & 2 & 86.71 & 89.66 & 82.26 & 94.56 & 90.85 & 92.87 \\
\circnum{10}  & \phantom{00}Teacher features    & 2 & 86.49 & 89.52 & 80.36 & 94.46 & 89.68 & 92.83 \\
\circnum{11}  & \phantom{00}Live self-features & 2 & 86.60 & 89.58 & 80.97 & 94.51 & 90.08 & 92.91 \\
\addlinespace[1pt]
\tablegroup{9}{Comparative}
\circnum{1}  & IDM & 2 & 86.82 & 89.70 & 80.85 & 94.60 & 90.00 & 93.08 \\
\circnum{15} & Z-Image-Turbo     & 2 & 85.78 & 88.78 & 84.19 & 94.16 & 89.60 & 91.73 \\
\circnum{13} & DMD2              & 2 & 85.40 & 88.88 & 83.59 & 93.81 & 85.60 & 91.10 \\
\circnum{14} & TwinFlow & 2 & 85.98 & 89.02 & 81.46 & 94.16 & 87.60 & 91.83 \\
\circnum{A}  & APEX & 2 & 88.64 & 90.20 & 88.35 & 95.10 & 90.80 & 92.88 \\
\circnum{12} & Z-Image-Turbo++   & 2 & 85.86 & 89.30 & 83.89 & 94.58 & 89.60 & 91.84 \\
\addlinespace[1pt]
\tablegroup{9}{Higher-NFE baseline}
\circnum{16} & Z-Image-Turbo & 8 & 84.86 & 90.14 & 91.29 & 92.16 & 88.68 & 89.59 \\
\bottomrule
\end{tabular}
}
\end{table}

\begin{table}[!htbp]
\centering
\caption{\textbf{Detailed OneIG ablation results.} Scores are percentages; Overall is the mean of the five category scores. Configurations, row groups, and sources follow \protect\tabref{tab:ablation_summary_2nfe}.}
\label{tab:ablation_oneig_2nfe}
\resizebox{0.9\linewidth}{!}{
\begin{tabular}{clccccccc}
\toprule
Idx & Model & NFE & Overall & Alignment & Text & Diversity & Style & Reasoning \\
\midrule
\firsttablegroup{9}{Distribution Matching Loss Ablation}
\circnum{2}  & \phantom{00}w/o $\cL_{\text{score}}$ (DMD loss) & 2 & 52.58 & 86.13 & 97.09 & 12.74 & 38.04 & 28.90 \\
\circnum{3}  & \phantom{00}w/o $\cL_{\text{pop}}$ (feature-moment loss) & 2 & 52.87 & 86.31 & 97.73 & 12.81 & 38.16 & 29.34 \\
\circnum{4}  & \phantom{00}w/o $\cL_{\text{score}}$ and $\cL_{\text{pop}}$ & 2 & 52.25 & 85.94 & 96.37 & 12.67 & 37.90 & 28.37 \\
\addlinespace[1pt]
\tablegroup{9}{DMD Regression Ablation}
\circnum{5}  & \phantom{00}$v$-loss only  & 2 & 52.43 & 86.04 & 96.76 & 12.71 & 37.97 & 28.67 \\
\circnum{6}  & \phantom{00}$x$-loss only  & 2 & 52.78 & 86.25 & 97.53 & 12.79 & 38.12 & 29.21 \\
\circnum{7}  & \phantom{00}2$\times$$x$-loss+$v$-loss & 2 & 53.01 & 86.39 & 98.06 & 12.84 & 38.22 & 29.54 \\
\circnum{8}  & \phantom{00}$x$-loss+2$\times$$v$-loss & 2 & 52.95 & 86.36 & 97.94 & 12.83 & 38.20 & 29.42 \\
\addlinespace[1pt]
\tablegroup{9}{Feature-Map Ablation}
\circnum{9}  & \phantom{00}DINOv2            & 2 & 53.00 & 86.58 & 97.72 & 12.88 & 38.36 & 29.46 \\
\circnum{10}  & \phantom{00}Teacher features    & 2 & 52.82 & 86.28 & 97.62 & 12.80 & 38.14 & 29.26 \\
\circnum{11}  & \phantom{00}Live self-features & 2 & 52.91 & 86.33 & 97.82 & 12.82 & 38.18 & 29.40 \\
\addlinespace[1pt]
\tablegroup{9}{Comparative}
\circnum{1}  & IDM & 2 & 53.06 & 86.42 & 98.15 & 12.85 & 38.24 & 29.65 \\
\circnum{15} & Z-Image-Turbo     & 2 & 50.94 & 83.92 & 92.85 & 13.85 & 36.25 & 27.86 \\
\circnum{13} & DMD2              & 2 & 50.70 & 85.21 & 90.82 & 12.75 & 36.77 & 27.93 \\
\circnum{14} & TwinFlow & 2 & 51.38 & 85.26 & 90.61 & 15.49 & 38.03 & 27.52 \\
\circnum{A}  & APEX & 2 & 52.60 & 86.10 & 97.05 & 12.70 & 38.00 & 29.15 \\
\circnum{12} & Z-Image-Turbo++   & 2 & 52.50 & 85.65 & 97.09 & 11.75 & 37.87 & 30.12 \\
\addlinespace[1pt]
\tablegroup{9}{Higher-NFE baseline}
\circnum{16} & Z-Image-Turbo & 8 & 52.84 & 84.03 & 99.32 & 13.85 & 36.81 & 30.19 \\
\bottomrule
\end{tabular}
}
\end{table}

\FloatBarrier

\begin{figure}[!t]
\centering
\includegraphics[width=\linewidth,height=0.82\textheight,keepaspectratio]{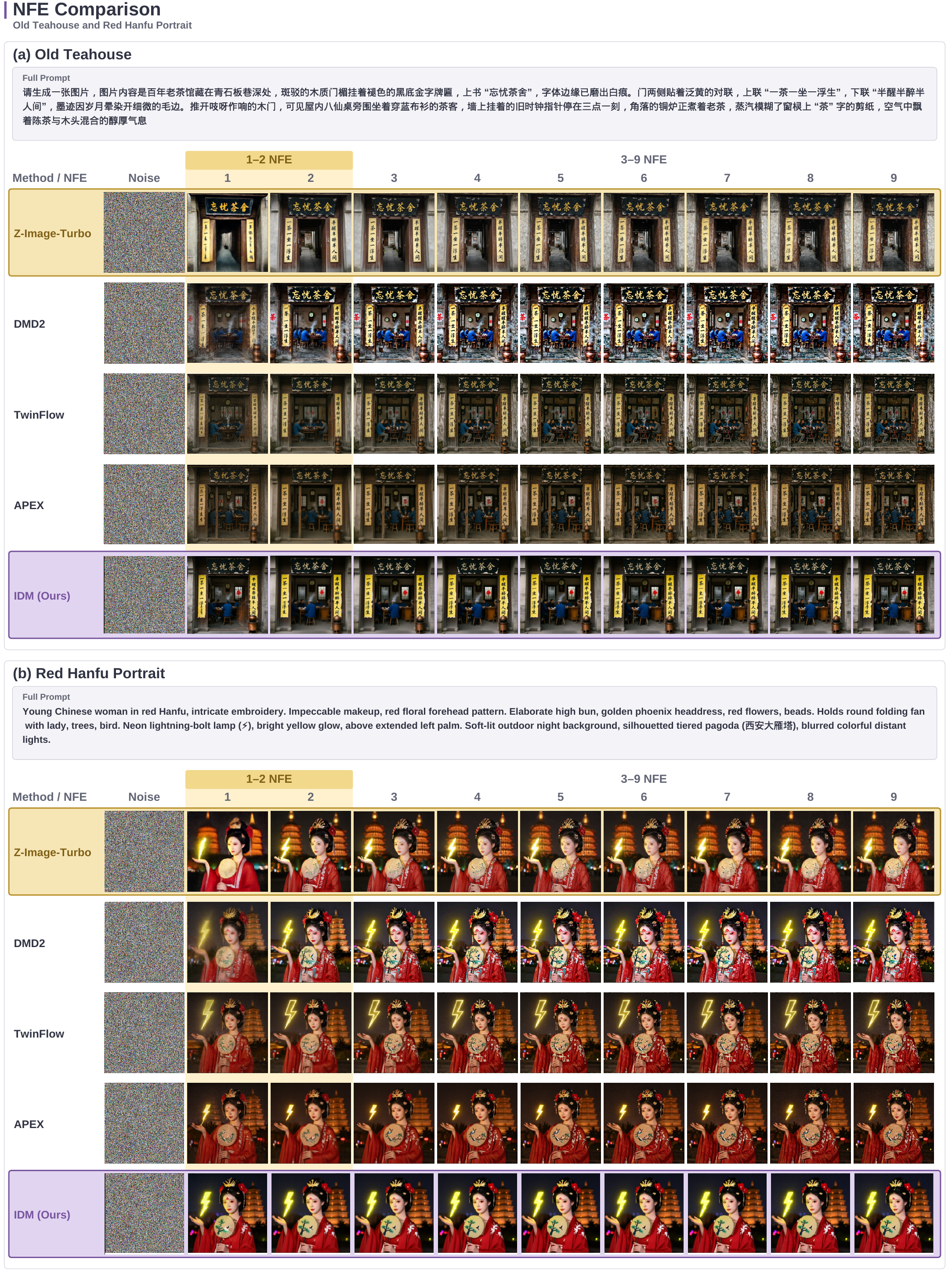}
\caption{\textbf{Sampling-budget comparison on Z-Image.} Final outputs from Z-Image-Turbo, DMD2, TwinFlow, APEX, and \method at the indicated NFE budgets, grouped by two prompts. Each column represents a sampling budget, not an intermediate state of one trajectory.}
\label{fig:nfe_comparison}
\end{figure}

\subsection{Qualitative Results}
\label{sec:qualitative}

\figref{fig:nfe_comparison} compares sampling budgets. \figref{fig:qualitative_showcase_1} shows text rendering and portraits; \figref{fig:qualitative_showcase_2} provides additional samples across styles and content.

\clearpage

\begin{figure}[!t]
\centering
\begin{minipage}[t]{0.48\linewidth}
\centering
\includegraphics[width=\linewidth]{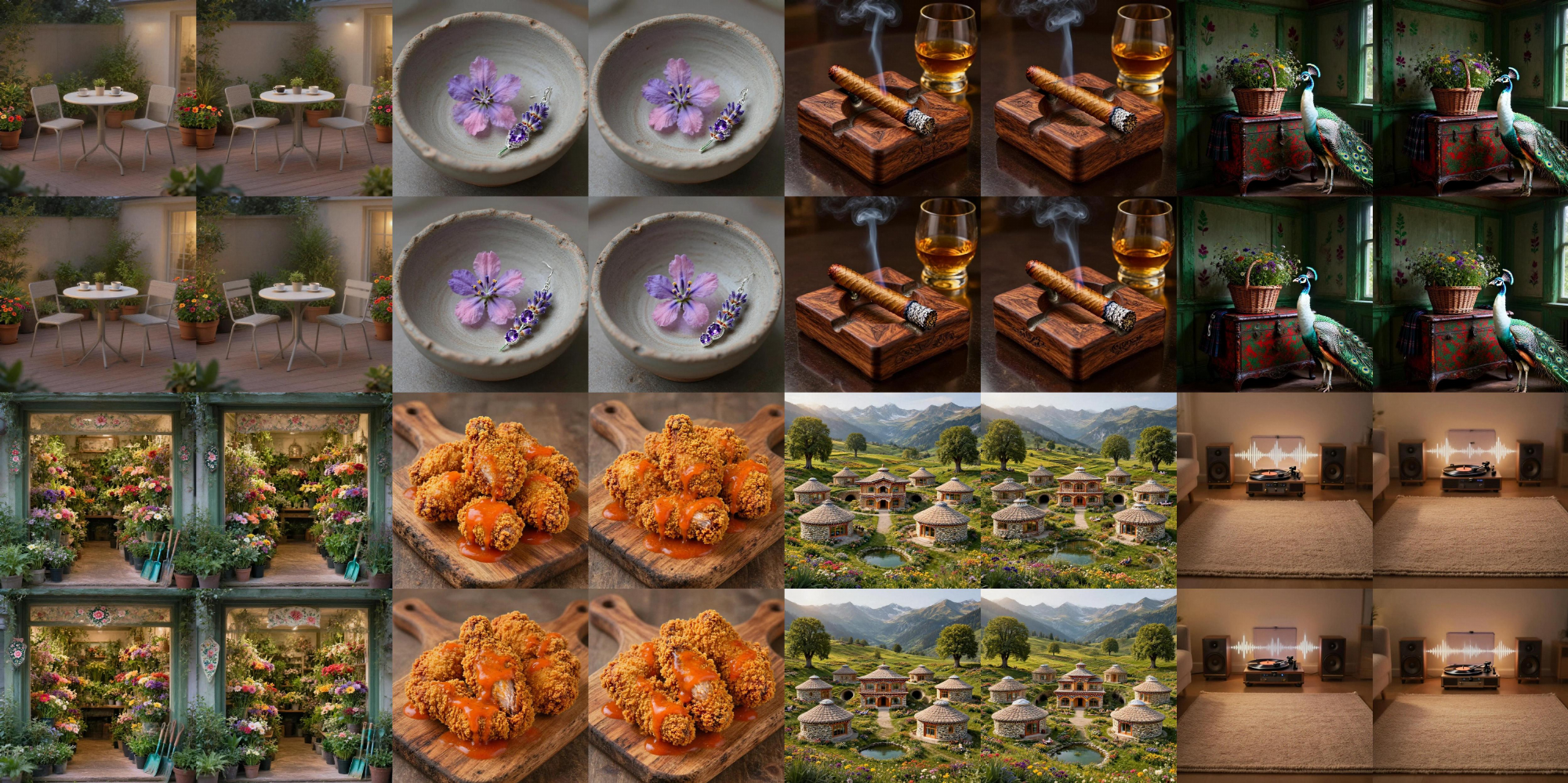}
\end{minipage}\hfill
\begin{minipage}[t]{0.48\linewidth}
\centering
\includegraphics[width=\linewidth]{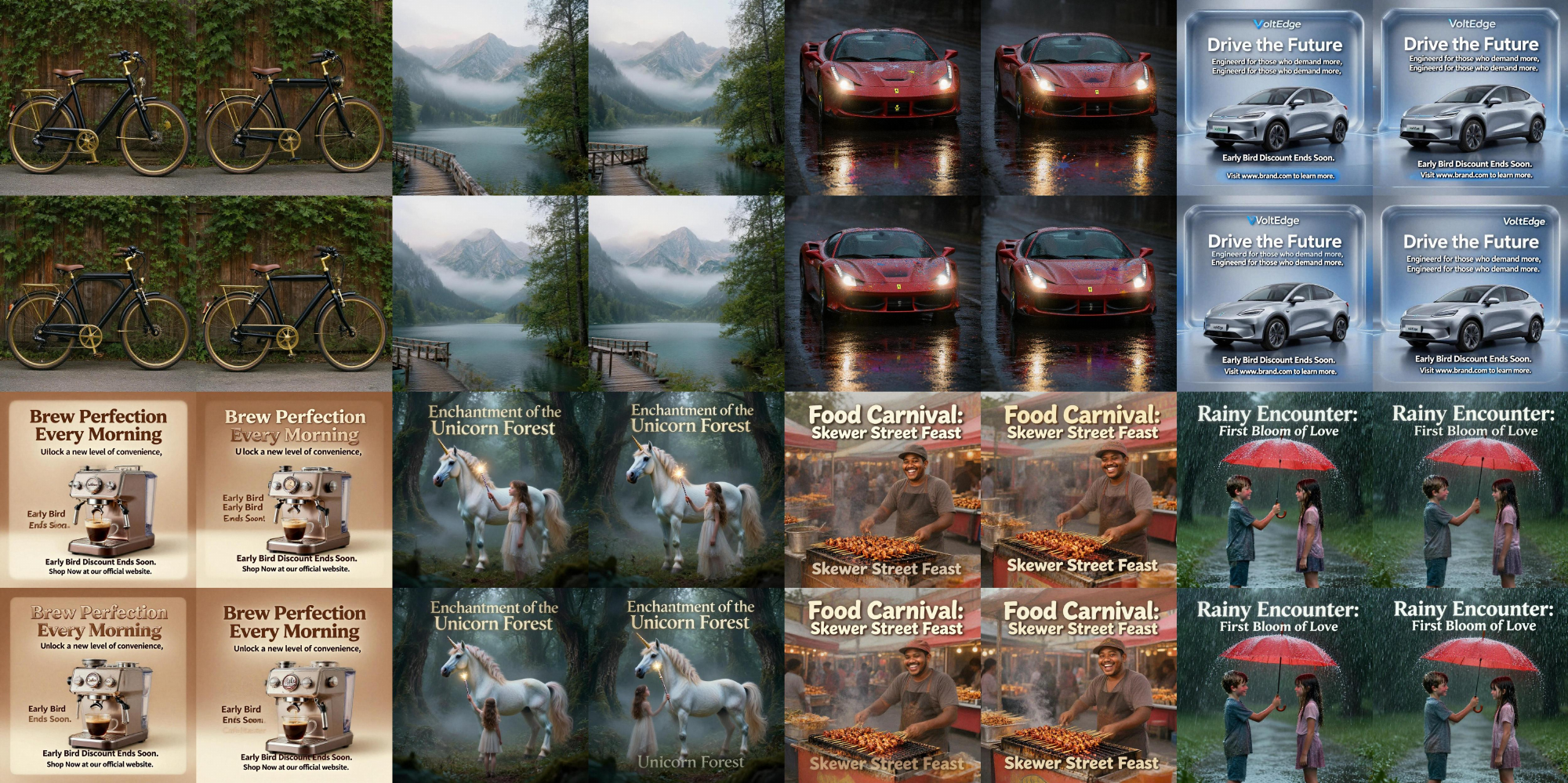}
\end{minipage}
\begin{minipage}[t]{0.48\linewidth}
\centering
\includegraphics[width=\linewidth]{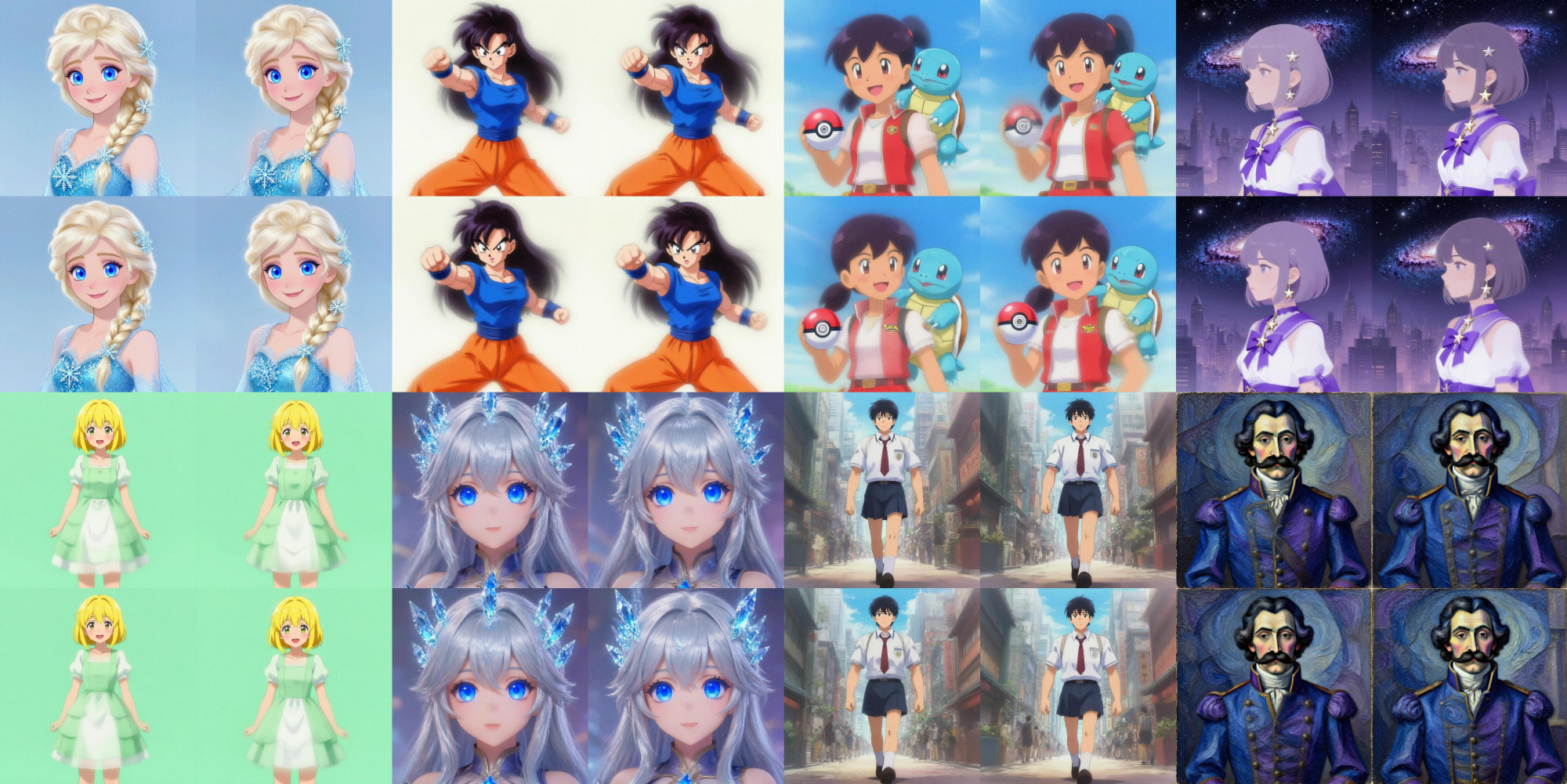}
\end{minipage}\hfill
\begin{minipage}[t]{0.48\linewidth}
\centering
\includegraphics[width=\linewidth]{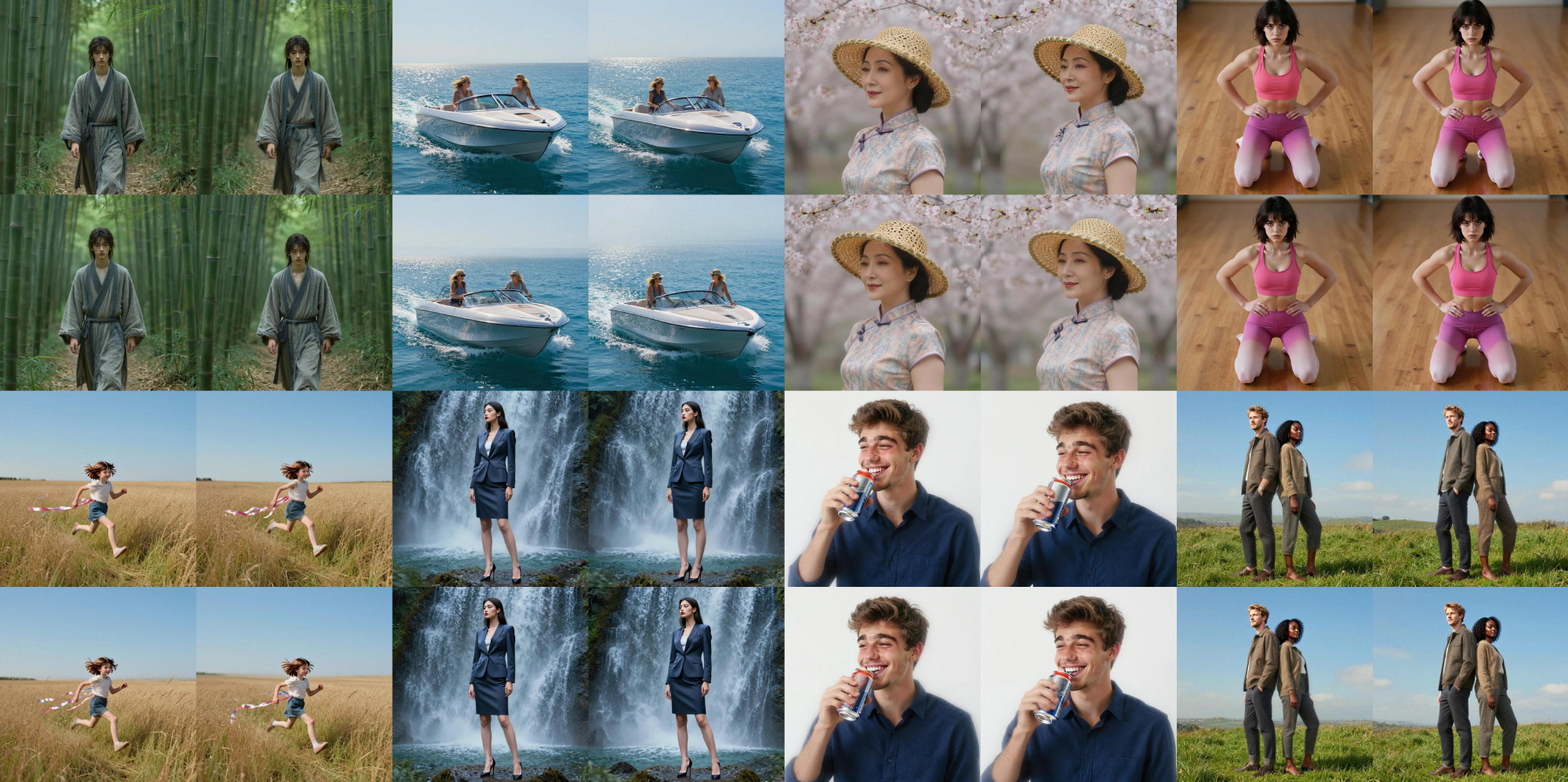}
\end{minipage}
\begin{minipage}[t]{0.48\linewidth}
\centering
\includegraphics[width=\linewidth]{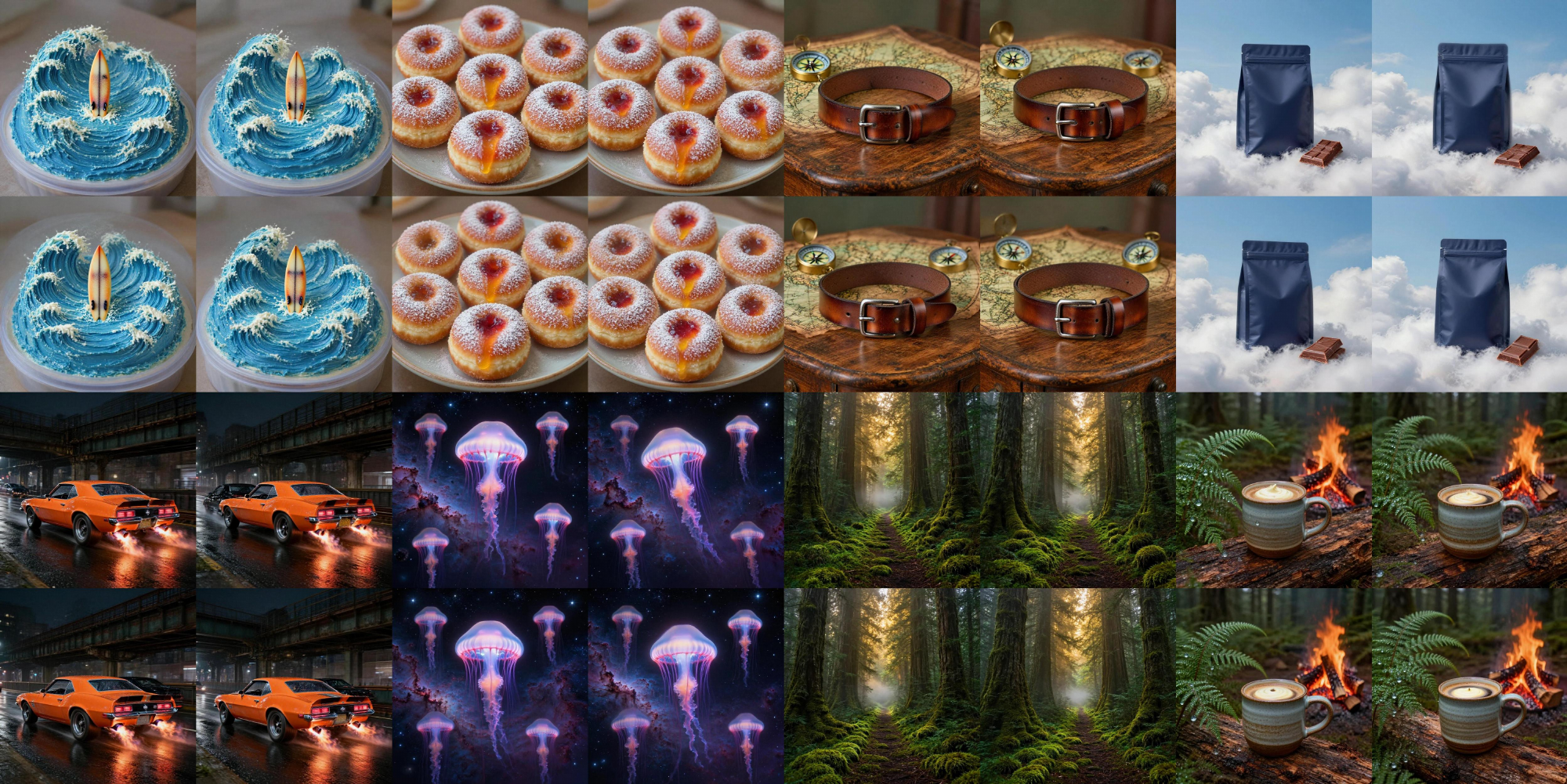}
\end{minipage}\hfill
\begin{minipage}[t]{0.48\linewidth}
\centering
\includegraphics[width=\linewidth]{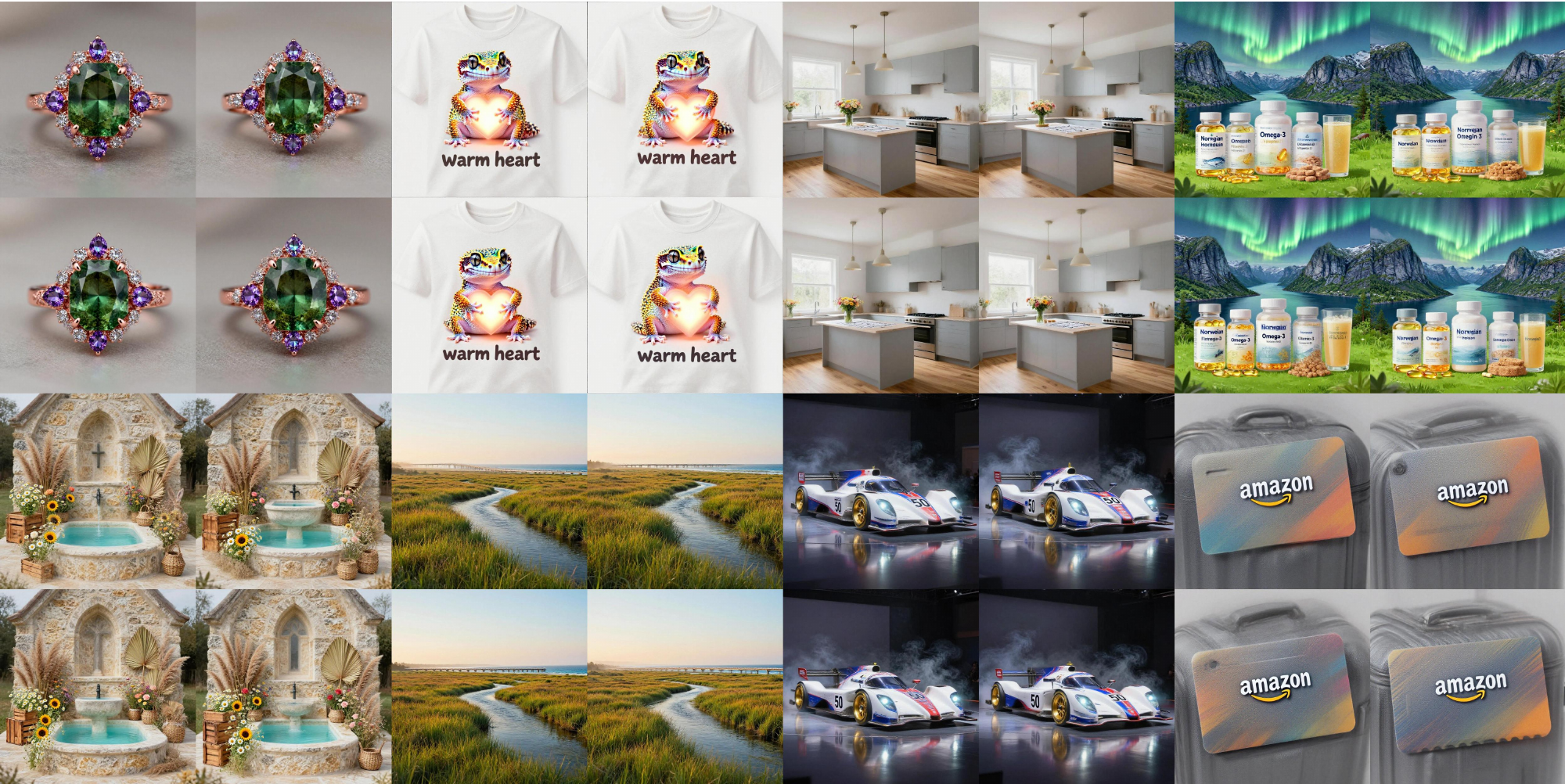}
\end{minipage}
\begin{minipage}[t]{0.48\linewidth}
\centering
\includegraphics[width=\linewidth]{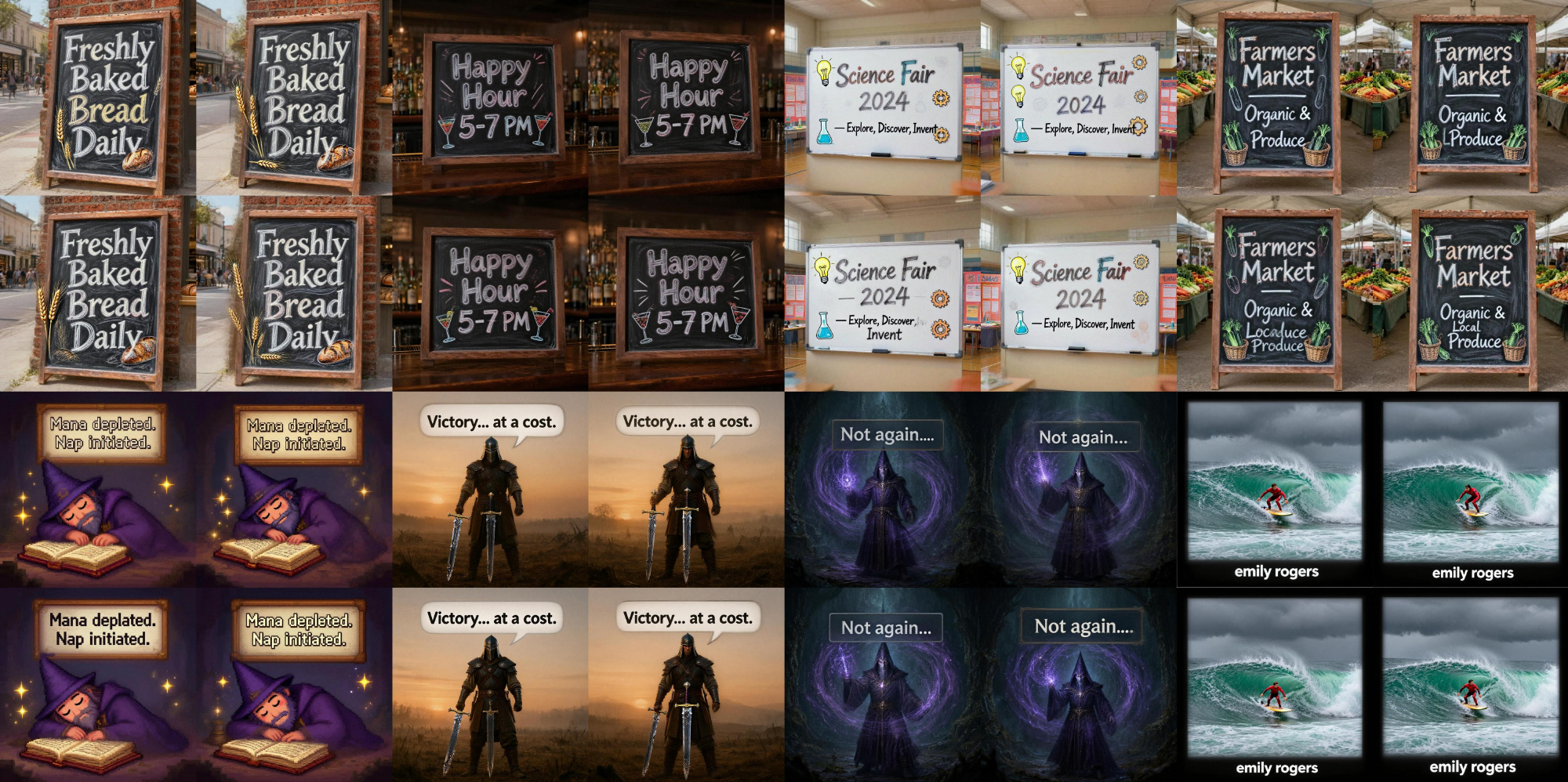}
\end{minipage}\hfill
\begin{minipage}[t]{0.48\linewidth}
\centering
\includegraphics[width=\linewidth]{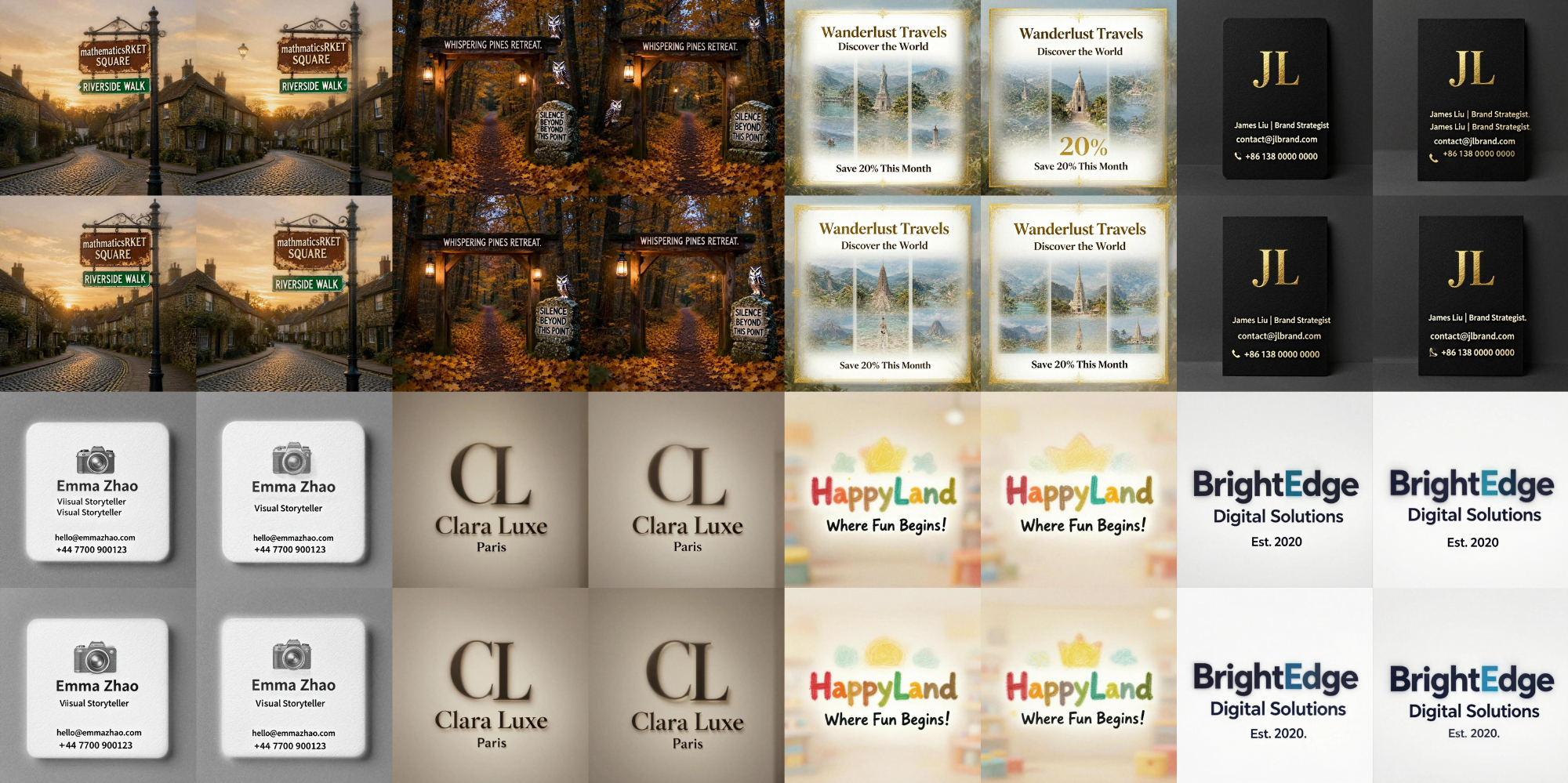}
\end{minipage}
\caption{\textbf{\method (Z-Image), 2~NFE.} Additional samples across styles and content. Every panel uses the same model and sampling budget; the grid is a gallery of outputs rather than a method or NFE comparison.}
\label{fig:qualitative_showcase_2}
\end{figure}

\FloatBarrier

\section{Conclusion}
\label{sec:conclusion}

We presented \textbf{Intrinsic Distribution Matching (\method)}, which builds on internal DMD and adds an explicit constraint on generator EMA feature statistics. The reverse-KL field decomposition motivates retaining full score-difference supervision while additionally matching feature means and standard deviations. Controlled ablations show gains over the internal DMD baseline from EMA feature moments, with performance close to DINOv2. On Z-Image-Turbo 6B, \method reaches 94.25\% and 90.20\% on Chinese and English LongText-Bench, the highest scores among the compared two-NFE methods. Experiments on backbones up to 20B parameters support using the generator's predictions and representations together for few-step text-to-image generation.

%% file: resources/appendix.tex
\setcounter{figure}{0}
\renewcommand{\thefigure}{A\arabic{figure}}
\renewcommand{\theHfigure}{appendix.\arabic{figure}}
\setcounter{table}{0}
\renewcommand{\thetable}{A\arabic{table}}
\renewcommand{\theHtable}{appendix.\arabic{table}}

\newpage
{
    \hypersetup{linkcolor=black}
    \parskip=0em
    \renewcommand{\contentsname}{Appendix Contents}
    \tableofcontents
}
\newpage

\input{resources/appendix_related_work.tex}
\clearpage

\section{Theoretical Proofs}
\label{app:proofs}

\paragraph{Conventions.} The distribution-gradient and score identities are stated at fixed condition $\cc\sim p_c$, with conditioning omitted when unambiguous; training averages the resulting conditional gradients over $\cc$. In contrast, feature moments pool samples from the joint law $p_i(\xx\mid\cc)p_c(\cc)$. A subscript $t$ denotes a noisy marginal, a star an exact marginal quantity, and a hat a learned or empirical estimate. Thus $\cL_{\text{pop}}$ denotes the population moment loss and $\widehat{\cL}_{\text{pop}}$ its empirical counterpart. For full queues of equal size $K$, the model estimator uses $N_\phi=B+K$ features, where $B$ is the global current-batch size; the reference estimator uses the $K$ queued reference features.

\subsection{Unified Gradient Structure}
\label{app:unified_grad}

\subsubsection{First variation form}
\label{app:unified_grad_first_variation}

We give regularity assumptions and prove the general identity behind \eqref{eq:universal_grad}.
We use the standard first-variation calculation for functionals of a density obeying a continuity equation, in the sense of gradient flows of probability measures~\citep{ambrosio2007gradient}. The pushforward step converts the sample-space integral into an expectation over the noise prior and clarifies its dependence on generator parameters.

\begin{definition}[Distribution Gradient]
\label{def:dist_grad}
For a distributional discrepancy $\cD(p_\theta, p_r)$ admitting a differentiable first-variation potential $\psi(\xx) := \delta\cD(p_\theta,p_r)/\delta p_\theta(\xx)$, the \emph{sample-space distribution gradient} is
\begin{equation}
  \gv_{\cD}(\xx) := \nabla_\xx \psi(\xx).
\end{equation}
\end{definition}

\begin{assumption}[Regularity for the first variation]
  \label{ass:regularity}
  Fix $\cc$ and let $G_\theta(\cdot,\cc)$ push $p_z$ to $p_\theta(\cdot\mid\cc)$. The reference law $p_r$ is fixed, and the discrepancy depends on $\theta$ only through $p_\theta$.
  For each direction $\Delta\theta$, define
  $\theta_\epsilon=\theta+\epsilon\Delta\theta$ and
  $p_{\theta_\epsilon}:=(G_{\theta_\epsilon}(\cdot,\cc))_\#p_z$.
  \begin{enumerate}[label=(\alph*),leftmargin=*]
    \item $G_\theta$ is differentiable in $\theta$,
    \item $p_{\theta_\epsilon}$ admits a smooth density near $\epsilon=0$, and differentiation can be exchanged with the relevant integrals.
    \item For any first-variation potential $\psi$ used below, the boundary term in integration by parts vanishes; for example, $\psi(\xx)p_\theta(\xx)\uu(\xx)$ decays sufficiently fast as $\|\xx\|\to\infty$.
  \end{enumerate}
\end{assumption}

\paragraph{On condition (b): when does the model law admit a density?}
\label{app:density}
Condition (b) is an assumption rather than a consequence of the generator construction.
For a locally Lipschitz deterministic map $G_\theta(\cdot,\cc):\R^d\to\R^d$, the area formula implies that a Jacobian $\partial G_\theta/\partial\zz$ singular on a set of positive $p_z$-measure produces a singular component in the image law. The full model law then has no Lebesgue density, so its clean score $\ss_\theta=\nabla_\xx\log p_\theta$ is not defined as assumed.
This is not a pathological case. A generator that collapses a direction of the seed space, for instance one that has learned to ignore part of $\zz$, produces exactly this degeneracy.
It is the same manifold-support obstruction identified by \citet{arjovsky2017wasserstein} for implicit generators, which there motivates replacing $f$-divergences by a weaker metric; we instead retain the divergence and carry the density assumption explicitly.

The clean-distribution results in \secref{sec:unified} therefore assume smooth, strictly positive densities on a common domain, together with the integrability and boundary conditions required by each result. The factorization \eqref{eq:universal_grad}, reverse-KL field \eqref{eq:reverse_kl_field}, and decomposition \eqref{eq:grad_decomp_main} are exact within that setting.
This restriction is distinct from the noisy-marginal setting: adding independent nondegenerate Gaussian noise gives a smooth positive density for every $t\in(0,1)$, even when the clean law is singular. The score--velocity identity in \appref{app:score_velocity} therefore does not require a clean density; its finite-second-moment assumption suffices for the stated velocity regression. A noisy reverse-KL objective additionally requires the relevant divergence and gradient integrals to exist.
The practical losses in \secref{sec:combined} remain computable without evaluating clean scores. Singular clean laws invalidate the clean-field interpretation, but do not by themselves invalidate the noisy score fields or feature-moment objective.

\begin{proof}[Proof of \eqref{eq:universal_grad}]
Fix an arbitrary direction $\Delta\theta$ and write $\theta_\epsilon=\theta+\epsilon\Delta\theta$. For $\xx=G_\theta(\zz,\cc)$, the generated sample has the first-order expansion:
\begin{equation}
  \label{eq:sample_expansion}
  \xx_{\theta_\epsilon}
  = G_{\theta+\epsilon\Delta\theta}(\zz,\cc)
  = \xx + \epsilon \uu_\zz + o(\epsilon),
\qquad
  \uu_\zz := J_\theta G_\theta(\zz,\cc)\Delta\theta .
\end{equation}
The corresponding velocity field is the conditional mean displacement of samples currently located at $\xx$:
\begin{equation}
\uu(\xx)
:=
  \mathbb{E}\!\left[\uu_\zz \mid G_\theta(\zz,\cc)=\xx\right].
  \label{eq:eulerian_velocity}
\end{equation}

We first identify the density perturbation induced by this sample displacement. For any smooth compactly supported test function $\varphi$, the pushforward definition and \eqref{eq:sample_expansion} give
\begin{align}
  \frac{d}{d\epsilon}
  \int \varphi(\xx)p_{\theta_\epsilon}(\xx)\,\dm\xx\bigg|_{\epsilon=0}
  &=
  \frac{d}{d\epsilon}
  \mathbb{E}_{\zz}
  \left[
    \varphi(G_{\theta+\epsilon\Delta\theta}(\zz,\cc))
  \right]\bigg|_{\epsilon=0} \\&=
  \frac{d}{d\epsilon}
  \mathbb{E}_{\zz}
  \left[
    \varphi(\xx + \epsilon \uu_\zz)
  \right]\bigg|_{\epsilon=0} \\
  &=
  \mathbb{E}_{\zz}
  \left[
    \nabla_\xx\varphi(G_\theta(\zz,\cc))^\top\uu_\zz
  \right] \\
  &=
  \int
  \nabla_\xx\varphi(\xx)^\top p_\theta(\xx)\uu(\xx)\,\dm\xx .
  \label{eq:weak_continuity}
\end{align}
The last equality is exactly the conditional-expectation definition of $\uu(\xx)$. Since \eqref{eq:weak_continuity} holds for all such $\varphi$, the density path satisfies the continuity equation in the weak sense:
\begin{equation}
  \partial_\epsilon p_{\theta_\epsilon}(\xx)\big|_{\epsilon=0}
=
  -\nabla_\xx\cdot(p_\theta(\xx)\uu(\xx)).
  \label{eq:continuity}
\end{equation}

Now consider a distributional discrepancy $\cD(p_\theta, p_r)$ with first variation
\begin{equation}
  \psi(\xx) := \frac{\delta \cD(p_\theta, p_r)}{\delta p_\theta}(\xx),
\end{equation}
meaning that, along the density path $p_{\theta_\epsilon}$,
\begin{align}
  \frac{d}{d\epsilon}\cD(p_{\theta_\epsilon}, p_r)\bigg|_{\epsilon=0}
&=
  \int \psi(\xx)\,\partial_\epsilon p_{\theta_\epsilon}(\xx)\big|_{\epsilon=0}\,\dm\xx \\
  &=
  -\int \psi(\xx)\,\nabla_\xx\cdot(p_\theta(\xx)\uu(\xx))\,\dm\xx \\
&=
  \int \nabla_\xx\psi(\xx)\cdot p_\theta(\xx)\,\uu(\xx)\,\dm\xx ,
  \label{eq:first_variation_ip}
\end{align}
where the last equality is integration by parts, using the vanishing boundary term in \assumpref{ass:regularity}(c). By \eqref{eq:eulerian_velocity}, any integrable vector field
$\ww(\xx)$ satisfies the following pushforward-expectation identity (converting an expectation under $p_\theta$ to one under the noise prior):
\begin{equation}
  \int p_\theta(\xx)\,\ww(\xx)^\top\uu(\xx)\,\dm\xx
  =
  \mathbb{E}_{\zz}
  \left[
    \ww(G_\theta(\zz,\cc))^\top\uu_\zz
  \right].
  \label{eq:latent_pullback}
\end{equation}
Taking $\ww=\nabla_\xx\psi$ in \eqref{eq:latent_pullback} gives
\begin{equation}
  \int p_\theta(\xx)\,\nabla_\xx\psi(\xx)\cdot\uu(\xx)\,\dm\xx
  =
  \mathbb{E}_{\zz}
\left[
\nabla_\xx\psi(G_\theta(\zz,\cc))^\top
\uu_\zz
\right]
  =
  \mathbb{E}_{\zz}
\left[
\nabla_\xx\psi(G_\theta(\zz,\cc))^\top
J_\theta G_\theta(\zz,\cc)\Delta\theta
\right].
  \label{eq:latent_first_variation}
\end{equation}
Combining \eqref{eq:first_variation_ip} and
\eqref{eq:latent_first_variation}, the directional derivative in the direction
$\Delta\theta$ is
\begin{equation}
  \frac{d}{d\epsilon}\cD(p_{\theta_\epsilon}, p_r)\bigg|_{\epsilon=0}
  =
  \left\langle
  \mathbb{E}_{\zz}
  \left[
    \bigl[J_\theta G_\theta(\zz,\cc)\bigr]^\top
    \nabla_\xx\psi(G_\theta(\zz,\cc))
  \right],
  \Delta\theta
  \right\rangle .
\end{equation}
Since $\Delta\theta$ is arbitrary,
\begin{equation}
  \nabla_\theta \cD(p_\theta, p_r)
=
  \mathbb{E}_{\zz}
\left[
\bigl[J_\theta G_\theta(\zz,\cc)\bigr]^\top
    \underbrace{\nabla_\xx\psi(G_\theta(\zz,\cc))}_{\gv_{\cD}(G_\theta(\zz,\cc))}
\right].
  \label{eq:general_unified_grad}
\end{equation}
Thus the sample-space distribution gradient in \eqref{eq:universal_grad} is
$\gv_{\cD}(\xx)=\nabla_\xx\psi(\xx)$ whenever the discrepancy admits the first variation $\psi=\delta\cD(p_\theta, p_r)/\delta p_\theta$.
\end{proof}

\subsubsection{$f$-divergence specialization}

Here $\rho(\xx):=p_\theta(\xx)/p_r(\xx)$ denotes the density ratio.

The following template is the $f$-divergence distribution-matching result of \citet{xu2025one}, whose variational form goes back to $f$-GAN~\citep{nowozin2016f}.
We restate it in the present notation, and add in \appref{app:unified_grad_special_cases} the distinction between an identified $f$-divergence first variation and a detached optimal-discriminator update field, which the two coincide on only for particular generators.

\begin{theorem}[Unified Gradient Template for $f$-divergences]
\label{thm:unified}
Let $G_\theta(\zz,\cc)=\xx$ be the generator at fixed $\cc$ and $h:(0,\infty)\to\R$ be a twice continuously differentiable convex divergence generator with
$h(1)=0$, and define
\begin{equation}
  D_h(p_\theta\|p_r)
  =
  \int p_r(\xx)h(\rho(\xx))\,\dm\xx,
  \qquad
  \rho(\xx):=\frac{p_\theta(\xx)}{p_r(\xx)}.
\end{equation}
Assume \assumpref{ass:regularity} applies to
$\cL[p_\theta]=D_h(p_\theta\|p_r)$, and assume $p_\theta$ and $p_r$ are smooth and strictly positive on an open set containing $\operatorname{supp}(p_\theta)$. Then the distribution gradient is
\begin{equation}
  \gv_h(\xx)
  =
  h''(\rho(\xx))\,\rho(\xx)\,(\ss_\theta(\xx)-\ss_r(\xx)),
\end{equation}
and
\begin{equation}
  \nabla_\theta D_h(p_\theta\|p_r)
  =
  \EEb{\zz \sim p_z}
  {
    \bigl[J_\theta G_\theta(\zz,\cc)\bigr]^\top
    \left[
      w(\xx)(\ss_\theta(\xx)-\ss_r(\xx))
    \right]_{\xx=G_\theta(\zz,\cc)}
  },
  \label{eq:unified_thm}
\end{equation}
where $w(\xx)=h''(\rho(\xx))\rho(\xx)$.
\end{theorem}

\begin{proof}
We first compute the first variation with respect to the model density $p_\theta$, while keeping $p_r$ fixed. Let $p_\epsilon=p_\theta+\epsilon\,\delta p$ be an admissible smooth perturbation with $\int \delta p(\xx)\,\dm\xx=0$, and define $\rho_\epsilon(\xx):=p_\epsilon(\xx)/p_r(\xx)$. Since $p_r$ is fixed, we have
\begin{equation}
  \rho_\epsilon(\xx)
  =
  \rho(\xx)
  +
  \epsilon\,\frac{\delta p(\xx)}{p_r(\xx)}.
\end{equation}
Therefore the derivative is
\begin{align}
  \frac{d}{d\epsilon}D_h(p_\epsilon\|p_r)\bigg|_{\epsilon=0}
  &=
  \int
  p_r(\xx)\,
  \left[
    \frac{d}{d\epsilon} h\left(\rho_\epsilon(\xx)\right)\bigg|_{\epsilon=0}\,
  \right]
  \,\dm\xx \\
  &=
  \int
  p_r(\xx)\,
  \left[
    \frac{d}{d\epsilon} h\left(\rho(\xx) + \epsilon\,\frac{\delta p(\xx)}{p_r(\xx)}\right)\bigg|_{\epsilon=0}\,
  \right]
  \,\dm\xx \\
  &=
  \int
  p_r(\xx)\,
  h'(\rho(\xx))\,
  \frac{\delta p(\xx)}{p_r(\xx)}
  \,\dm\xx \\
  &=
  \int h'(\rho(\xx))\,\delta p(\xx)\,\dm\xx\,.
  \label{eq:first_variation_potential}
\end{align}
By definition, the first-variation potential is any function $\psi_h$ whose pairing with every admissible density perturbation $\delta p$ gives the directional derivative:
\begin{equation}
  \frac{d}{d\epsilon}D_h(p_\theta+\epsilon\,\delta p\|p_r)\bigg|_{\epsilon=0}
  =
  \int \psi_h(\xx)\,\delta p(\xx)\,\dm\xx .
  \label{eq:first_variation_potential_definition}
\end{equation}
Comparing this defining identity \eqref{eq:first_variation_potential_definition} with \eqref{eq:first_variation_potential}, we can identify
\begin{equation}
  \psi_h(\xx)
  :=
  \frac{\delta D_h}{\delta p_\theta}(\xx)
  =
  h'(\rho(\xx)).
\end{equation}
This identification is unique up to an additive constant on the support of admissible perturbations, because $\int \delta p(\xx)\,\dm\xx=0$ and
\begin{equation}
  \int \bigl(\psi_h(\xx)+c_0\bigr)\delta p(\xx)\,\dm\xx
  =
  \int \psi_h(\xx)\delta p(\xx)\,\dm\xx
\end{equation}
for any scalar constant $c_0$ (distinct from the conditioning vector $\cc$).

We now compute the sample-space field. By the chain rule,
\begin{equation}
  \label{eq:chain_rule}
  \nabla_\xx \psi_h(\xx)
  =
  h''(\rho(\xx))\nabla_\xx \rho(\xx).
\end{equation}
It remains to express $\nabla_\xx \rho$ in terms of the two scores. Since $\rho=p_\theta/p_r$ and both densities are positive,
\begin{equation}
  \nabla_\xx \rho(\xx)
  =
  \frac{\nabla_\xx p_\theta(\xx)}{p_r(\xx)}
  -
  \frac{p_\theta(\xx)\nabla_\xx p_r(\xx)}{p_r(\xx)^2}.
\end{equation}
Using $\nabla_\xx p_\theta=p_\theta\ss_\theta$ and $\nabla_\xx p_r=p_r\ss_r$, this becomes
\begin{align}
  \nabla_\xx \rho(\xx)
  &=
  \frac{p_\theta(\xx)\ss_\theta(\xx)}{p_r(\xx)}
  -
  \frac{p_\theta(\xx)p_r(\xx)\ss_r(\xx)}{p_r(\xx)^2} \nonumber \\
  &=
  \rho(\xx)\bigl(\ss_\theta(\xx)-\ss_r(\xx)\bigr).
\end{align}
Substituting this identity into \eqref{eq:chain_rule} gives
\begin{equation}
  \nabla_\xx \psi_h(\xx)
  =
  h''(\rho(\xx))\,\rho(\xx)(\ss_\theta(\xx)-\ss_r(\xx)).
\end{equation}
Therefore the distribution gradient is
\begin{equation}
  \gv_h(\xx)
  =
  \nabla_\xx\psi_h(\xx)
  =
  h''(\rho(\xx))\,\rho(\xx)(\ss_\theta(\xx)-\ss_r(\xx)).
\end{equation}

Finally, applying the general first-variation template \eqref{eq:general_unified_grad} to $\cL[p_\theta]=D_h(p_\theta\|p_r)$ yields
\begin{align}
  \nabla_\theta D_h(p_\theta \| p_r)
  &=
  \EEb{\zz\sim p_z}{
    \bigl[J_\theta G_\theta(\zz,\cc)\bigr]^\top
    \gv_h(G_\theta(\zz,\cc))
  } \nonumber \\
  &=
  \EEb{\zz\sim p_z}{
    \bigl[J_\theta G_\theta(\zz,\cc)\bigr]^\top
    \left[
      h''(\rho(\xx))\,\rho(\xx)
      (\ss_\theta(\xx)-\ss_r(\xx))
    \right]_{\xx=G_\theta(\zz,\cc)}
  }.
\end{align}
With $w(\xx)=h''(\rho(\xx))\rho(\xx)$, this is exactly
\eqref{eq:unified_thm}.
\end{proof}

\subsubsection{Important special cases}
\label{app:unified_grad_special_cases}
The density-ratio-weighted form $w(\xx)(\ss_\theta-\ss_r)$ with $w=h''(\rho)\rho$ is the one-step $f$-divergence template of \citet{xu2025one}, later extended to broader divergence families by \citet{wang2026uni}; reading the optimal-discriminator field of GANs as a weighted score difference follows \citet{yi2023monoflow}.

\paragraph{Reverse KL / DMD.}
For $h(\rho)=\rho\log \rho$, we have $h''(\rho)=1/\rho$, hence $w=\rho h''(\rho)=1$.
The distribution gradient is therefore
\begin{equation}
  \gv_{\mathrm{KL}}(\xx)=\ss_\theta(\xx)-\ss_r(\xx).
\end{equation}
This is the ideal reverse-KL field; the score-difference component of practical DMD estimates its noisy-marginal counterpart with approximate reference and online fake score models.
It is not the complete DMD objective: the original method also uses an LPIPS regression term on pre-generated teacher pairs, and its finite-training online fake-score estimate can be lagged or biased~\citep{yin2024one}.
The corresponding sample-space descent direction is $-\gv_{\mathrm{KL}}=\ss_r-\ss_\theta$.

\paragraph{Scaled Jensen--Shannon / saturating GAN.}
Using twice the conventional equal-mixture JSD generator (the scale is chosen to match the saturating GAN field),
\begin{equation}
  h_{\mathrm{JSD}}(\rho)
  =
  \rho\log \rho-(1+\rho)\log\frac{1+\rho}{2},
\end{equation}
we have
\begin{equation}
  h_{\mathrm{JSD}}''(\rho)=\frac{1}{\rho(1+\rho)},\qquad
  w_{\mathrm{JSD}}(\rho)=\frac{1}{1+\rho}=D^*(\xx),
\end{equation}
where $D^*(\xx)=p_r/(p_r+p_\theta)$ is the optimal real-probability discriminator~\citep{goodfellow2014generative} (here the superscript $^*$ denotes optimality, distinct from its use for exact marginal velocity fields elsewhere).

\paragraph{Detached optimal discriminator generator updates.}
Assume smooth positive densities and an exact optimal discriminator.
We hold the discriminator parameters fixed while retaining derivatives with respect to its input.
Under $D^*(\xx) = p_r/(p_r+p_\theta)$, we have $1{-}D^* = p_\theta/(p_r+p_\theta)$.
For the saturating objective $\ell_{\mathrm{sat}}(\xx) = \log(1{-}D^*(\xx)) = \log p_\theta - \log(p_r+p_\theta)$:
\begin{align}
  \nabla_\xx \ell_{\mathrm{sat}}(\xx)
  &= \ss_\theta(\xx) - \frac{p_r\,\ss_r(\xx) + p_\theta\,\ss_\theta(\xx)}{p_r+p_\theta} \nonumber \\
  &= \ss_\theta - D^*\ss_r - (1{-}D^*)\ss_\theta
  = D^*(\xx)(\ss_\theta(\xx)-\ss_r(\xx))\,.
\end{align}
For the non-saturating objective $\ell_{\mathrm{ns}}(\xx) = -\log D^*(\xx) = \log(p_r+p_\theta) - \log p_r$:
\begin{align}
  \nabla_\xx \ell_{\mathrm{ns}}(\xx)
  &= \frac{p_r\,\ss_r(\xx) + p_\theta\,\ss_\theta(\xx)}{p_r+p_\theta} - \ss_r(\xx) \nonumber \\
  &= D^*\ss_r + (1{-}D^*)\ss_\theta - \ss_r
  = (1{-}D^*(\xx))(\ss_\theta(\xx)-\ss_r(\xx))\,.
\end{align}
The saturating field agrees with the distribution gradient of the scaled JSD above; the non-saturating field is instead a detached-optimal-discriminator generator update.

\paragraph{Special cases.}
\begin{center}
\small
\begin{tabular}{lccc}
\toprule
\textbf{Objective / divergence} & $h(\rho)$ & $h''(\rho)$ & $w = h''(\rho) \cdot \rho$ \\
\midrule
Reverse KL (DMD) & $\rho \log \rho$ & $1/\rho$ & $1$ \\
Scaled JSD (2$\times$JSD) & $\rho\log \rho - (1{+}\rho)\log\frac{1{+}\rho}{2}$ & $\frac{1}{\rho(1{+}\rho)}$ & $\frac{1}{1{+}\rho} = D^*$ \\
NS-GAN detached update & ---$^\dagger$ & --- & $1 - D^*$ \\
Saturating GAN & scaled JSD above & as above & $D^*$ \\
\bottomrule
\end{tabular}
\end{center}
Here $^\dagger$ denotes only the detached optimal-$D^*$ generator update field, not an identified $f$-divergence first variation; the saturating-GAN row is the scaled $2{\times}$JSD field in the preceding row.

\subsection{Score Decomposition of the Reverse-KL Field}
\label{app:decomposition}

\subsubsection{Score Decomposition Identity}
\label{app:decomp_identity}

\begin{definition}[Moment-Matched Gaussian]
\label{def:moment_gaussian}
For any distribution $p$ over $\R^d$ with finite first and second moments and positive-definite covariance, define its \emph{moment-matched Gaussian} as
\begin{equation}
  q_p := \cN(\mu_p, \Sigma_p)\,,
  \qquad
  \mu_p := \mathbb{E}_p[\xx]\,,
  \quad
  \Sigma_p := \operatorname{Cov}_p(\xx)\,.
  \label{eq:moment_gaussian}
\end{equation}
\end{definition}

\begin{definition}[Non-Gaussian Score Residual]
\label{def:non_gaussian_residual}
For any smooth distribution $p$ satisfying \defref{def:moment_gaussian}, with score $\ss_p(\xx) = \nabla_\xx \log p(\xx)$ and moment-matched Gaussian $q_p$, define the non-Gaussian score residual as
\begin{equation}
  \boldsymbol{\eta}_p(\xx) := \ss_p(\xx) - \ss_{q_p}(\xx)\,.
  \label{eq:non_gaussian_residual}
\end{equation}
where $\ss_{q_p}(\xx) = -\Sigma_p^{-1}(\xx - \mu_p)$.
\end{definition}

Therefore, based on \defref{def:moment_gaussian} and \defref{def:non_gaussian_residual}, we can decompose the score function of any distribution $p$ as
\begin{equation}
  \label{eq:score_decomposition_identity}
  \ss_p(\xx) = \ss_{q_p}(\xx) + \boldsymbol{\eta}_p(\xx)\,.
\end{equation}
This is the \emph{score decomposition} in Eq.~\ref{eq:score_decomp_main} of the main text.

\subsubsection{Orthogonality of the Residual}
\label{app:orthogonality}

We use the following scalar-test-function form of Stein's identity \citep[Def.~2.1 and Lem.~2.2]{liu2016kernelized}. If $\varphi:\R^d\to\R$ belongs to the Stein class of $p$, then
\begin{equation}
  \mathbb{E}_p[\ss_p(\xx)\varphi(\xx)+\nabla_\xx\varphi(\xx)]=\mathbf{0}.
\end{equation}
The two identities in \lemref{lem:score_stein_identities} correspond to $\varphi(\xx)=1$ and $\varphi(\xx)=x_l$. We include a coordinate proof for clarity.

\begin{lemma}[Score Stein Identities]
\label{lem:score_stein_identities}
Let $p$ be a strictly positive smooth density on $\R^d$ with score
$\ss_p(\xx)=\nabla_\xx\log p(\xx)$. Assume
\begin{equation}
  \mathbb E_p[\|\ss_p(\xx)\|]<\infty,
  \qquad
  \mathbb E_p[\|\xx\|\,\|\ss_p(\xx)\|]<\infty.
  \label{eq:stein_integrability}
\end{equation}
For $B_R=\{\xx:\|\xx\|\le R\}$, let $n_k$ be the $k$-th coordinate of the outward unit normal on $\partial B_R$.
Assume also that the following boundary condition holds for all coordinates $(k,l)$:
\begin{equation}
  \lim_{R\to\infty}\int_{\partial B_R}p(\xx)n_k\,\dm S=0,
  \qquad
  \lim_{R\to\infty}\int_{\partial B_R}p(\xx)x_l n_k\,\dm S=0.
  \label{eq:stein_boundary_flux}
\end{equation}
Then
\begin{equation}
  \mathbb{E}_p[\ss_p(\xx)]=\mathbf{0},
  \qquad
  \mathbb{E}_p[\ss_p(\xx)\xx^\top]=-\mI_d.
\end{equation}
\end{lemma}

\begin{proof}
Since $p(\xx)\ss_p(\xx)=\nabla_\xx p(\xx)$, the $k$-th coordinate of
$\mathbb{E}_p[\ss_p(\xx)]$ is
\begin{equation}
  \mathbb{E}_p[[\ss_p(\xx)]_k]
  =
  \int \frac{\partial p(\xx)}{\partial x_k}\,\dm\xx.
\end{equation}
By the divergence theorem on $B_R$,
\begin{equation}
  \int_{B_R}\frac{\partial p(\xx)}{\partial x_k}\,\dm\xx
  =
  \int_{\partial B_R}p(\xx)n_k\,\dm S.
\end{equation}
By \eqref{eq:stein_integrability}, the volume integrand is absolutely integrable. Taking $R\to\infty$ and using the first limit in \eqref{eq:stein_boundary_flux} gives
\begin{equation}
  \int \frac{\partial p(\xx)}{\partial x_k}\,\dm\xx=0.
\end{equation}
Since this holds for every $k$,
\begin{equation}
  \mathbb{E}_p[\ss_p(\xx)]=\mathbf{0}.
  \label{eq:stein_mean}
\end{equation}

For the second identity, fix coordinates $(k,l)$. Again by the divergence theorem,
\begin{equation}
  \int_{B_R}
  \frac{\partial}{\partial x_k}\!\left[p(\xx)x_l\right]
  \,\dm\xx
  =
  \int_{\partial B_R}p(\xx)x_l n_k\,\dm S.
\end{equation}
The second condition in \eqref{eq:stein_integrability} makes the expanded derivative integrable. Taking $R\to\infty$ and using the second limit in \eqref{eq:stein_boundary_flux} gives
\begin{equation}
  \int
  \frac{\partial}{\partial x_k}\!\left[p(\xx)x_l\right]
  \,\dm\xx
  =
  0.
\end{equation}
Expanding the derivative gives
\begin{equation}
  \frac{\partial}{\partial x_k}\!\left[p(\xx)x_l\right]
  =
  \frac{\partial p(\xx)}{\partial x_k}x_l
  +
  p(\xx)\delta_{kl}
  =
  p(\xx)[\ss_p(\xx)]_k x_l
  +
  p(\xx)\delta_{kl},
\end{equation}
where $\delta_{kl}$ is the Kronecker delta. Integrating and using $\int p(\xx)\,\dm\xx=1$ yields
\begin{equation}
  \mathbb{E}_p\!\left[[\ss_p(\xx)]_k x_l\right]
  +
  \delta_{kl}
  =
  0.
\end{equation}
Thus $\mathbb{E}_p[[\ss_p(\xx)]_k x_l]=-\delta_{kl}$ for all $(k,l)$, or in matrix form,
\begin{equation}
  \mathbb{E}_p[\ss_p(\xx)\xx^\top] = -\mI_d.
  \label{eq:stein_cov}
\end{equation}
\end{proof}

\begin{proposition}[Orthogonality of the Non-Gaussian Score Residual]
\label{prop:residual_orthogonality}
Let $p$ be a strictly positive smooth density on $\R^d$ with finite second moment and positive-definite covariance $\Sigma_p$. Assume the integrability condition \eqref{eq:stein_integrability} and boundary condition \eqref{eq:stein_boundary_flux}. Finite Fisher information together with the finite second moment is sufficient for \eqref{eq:stein_integrability} by Cauchy--Schwarz. Then the non-Gaussian score residual $\boldsymbol{\eta}_p=\ss_p-\ss_{q_p}$ satisfies
\begin{enumerate}[label=(\roman*), nosep]
  \item $\mathbb{E}_p[\boldsymbol{\eta}_p(\xx)] = \mathbf{0}$.
  \item $\mathbb{E}_p[\boldsymbol{\eta}_p(\xx)\xx^\top] = \mathbf{0}_{d\times d}$.
\end{enumerate}
\end{proposition}

\begin{proof}
\textbf{Step 1: Score Stein identities.}
By \lemref{lem:score_stein_identities},
\begin{equation}
  \mathbb{E}_p[\ss_p(\xx)]=\mathbf{0},
  \qquad
  \mathbb{E}_p[\ss_p(\xx)\xx^\top]=-\mI_d.
\end{equation}

\textbf{Step 2: The moment-matched Gaussian has the same two identities under
$p$.}
By definition, the score of $q_p=\cN(\mu_p,\Sigma_p)$ is
\begin{equation}
  \ss_{q_p}(\xx)=-\Sigma_p^{-1}(\xx-\mu_p).
\end{equation}
Taking expectation under $p$ gives
\begin{equation}
  \mathbb{E}_p[\ss_{q_p}(\xx)]
  =
  -\Sigma_p^{-1}\bigl(\mathbb{E}_p[\xx]-\mu_p\bigr)
  =
  \mathbf{0}.
\end{equation}
For the first-moment identity,
\begin{align}
  \mathbb{E}_p[\ss_{q_p}(\xx)\xx^\top]
  &=
  -\Sigma_p^{-1}\mathbb{E}_p[(\xx-\mu_p)\xx^\top] \nonumber \\
  &=
  -\Sigma_p^{-1}
  \mathbb{E}_p[(\xx-\mu_p)((\xx-\mu_p)+\mu_p)^\top] \nonumber \\
  &=
  -\Sigma_p^{-1}
  \left(
    \mathbb{E}_p[(\xx-\mu_p)(\xx-\mu_p)^\top]
    +
    \mathbb{E}_p[\xx-\mu_p]\mu_p^\top
  \right) \nonumber \\
  &=
  -\Sigma_p^{-1}\Sigma_p
  = -\mI_d\,.
\end{align}

\textbf{Step 3: Subtract the Gaussian identities from the true-score
identities.}
Since $\boldsymbol{\eta}_p=\ss_p-\ss_{q_p}$, the identities in
\lemref{lem:score_stein_identities} and Step~2 imply
\begin{equation}
  \mathbb{E}_p[\boldsymbol{\eta}_p(\xx)]
  =
  \mathbb{E}_p[\ss_p(\xx)]
  -
  \mathbb{E}_p[\ss_{q_p}(\xx)]
  =
  \mathbf{0}.
\end{equation}
Similarly,
\begin{equation}
  \mathbb{E}_p[\boldsymbol{\eta}_p(\xx)\xx^\top]
  =
  \mathbb{E}_p[\ss_p(\xx)\xx^\top]
  -
  \mathbb{E}_p[\ss_{q_p}(\xx)\xx^\top]
  =
  -\mI_d-(-\mI_d)
  =
  \mathbf{0}_{d\times d}.
\end{equation}
Finally, combining (i) and (ii) yields the equivalent centered form:
\begin{equation}
  \mathbb{E}_p[\boldsymbol{\eta}_p(\xx)(\xx-\mu_p)^\top]
  =
  \mathbb{E}_p[\boldsymbol{\eta}_p(\xx)\xx^\top]
  -
  \mathbb{E}_p[\boldsymbol{\eta}_p(\xx)]\mu_p^\top
  =
  \mathbf{0}_{d\times d}.
\end{equation}
\end{proof}

\pagebreak[3]
\subsubsection{Reverse-KL Field Decomposition}
\label{app:grad_decomposition}

\begin{theorem}[Reverse-KL Field Decomposition]
\label{thm:grad_decomp}
Assume $p_\theta$ and $p_r$ are strictly positive smooth densities satisfying \defref{def:moment_gaussian}, and set $q_i:=q_{p_i}$ and $\boldsymbol{\eta}_i:=\boldsymbol{\eta}_{p_i}$ for $i\in\{r,\theta\}$. For the reverse-KL distribution gradient $\gv_{\mathrm{KL}}(\xx) = \ss_\theta(\xx) - \ss_r(\xx)$,
\begin{equation}
  \gv_{\mathrm{KL}}(\xx)
  =
  \underbrace{\left(\ss_{q_\theta}(\xx) - \ss_{q_r}(\xx)\right)}_{\gv_{\mathrm{KL}}^{\mathrm{G}}(\xx):\;\text{Gaussian moment component}}
  +
  \underbrace{\left(\boldsymbol{\eta}_\theta(\xx) - \boldsymbol{\eta}_r(\xx)\right)}_{\gv_{\mathrm{KL}}^{\mathrm{nG}}(\xx):\;\text{non-Gaussian residual component}}
  \label{eq:grad_decomp}
\end{equation}
where $\gv_{\mathrm{KL}}^{\mathrm{G}}(\xx) = -\Sigma_\theta^{-1}(\xx - \mu_\theta) + \Sigma_r^{-1}(\xx - \mu_r)$ depends only on the first two moments, and $\gv_{\mathrm{KL}}^{\mathrm{nG}}(\xx) = \boldsymbol{\eta}_\theta(\xx) - \boldsymbol{\eta}_r(\xx)$ captures the remaining score structure.
This exact decomposition is stated for the unweighted reverse-KL field. For a general $f$-divergence field $w(\xx)(\ss_\theta-\ss_r)$, the algebraic decomposition is multiplied by $w(\xx)$, but the residual orthogonality identities need not be preserved after weighting.
\end{theorem}

\begin{proof}
Apply the score decomposition in \eqref{eq:score_decomposition_identity}
to both $\ss_\theta$ and $\ss_r$:
\begin{align}
  \ss_\theta(\xx) - \ss_r(\xx)
  &= \left[\ss_{q_\theta}(\xx) + \boldsymbol{\eta}_\theta(\xx)\right] - \left[\ss_{q_r}(\xx) + \boldsymbol{\eta}_r(\xx)\right] \nonumber \\
  &= \underbrace{\left(\ss_{q_\theta}(\xx) - \ss_{q_r}(\xx)\right)}_{\gv_{\mathrm{KL}}^{\mathrm{G}}} + \underbrace{\left(\boldsymbol{\eta}_\theta(\xx) - \boldsymbol{\eta}_r(\xx)\right)}_{\gv_{\mathrm{KL}}^{\mathrm{nG}}}\,.
\end{align}
\end{proof}

\subsubsection{Diagonal Form and Connection to Estimators}
\label{app:decomp_connection}

In the diagonal case $\Sigma_i = \diag(\sigma_{i,1}^2, \ldots, \sigma_{i,d}^2)$, where $\sigma_{i,j}>0$ is the standard deviation of coordinate $j$ under $p_i$ for $i\in\{r,\theta\}$,
\begin{equation}
  g^{\mathrm{G}}_{\mathrm{KL},j}(\xx) := [\gv_{\mathrm{KL}}^{\mathrm{G}}(\xx)]_j = [-\Sigma_\theta^{-1}(\xx - \mu_\theta) + \Sigma_r^{-1}(\xx - \mu_r)]_j = -\frac{x_j - \mu_{\theta,j}}{\sigma_{\theta,j}^2} + \frac{x_j - \mu_{r,j}}{\sigma_{r,j}^2}\,.
  \label{eq:gG_diagonal}
\end{equation}
Requiring $g^{\mathrm{G}}_{\mathrm{KL},j}(\xx) = 0$ for all $\xx$ forces the coefficient of $x_j$ to vanish:
\begin{equation}
  \frac{1}{\sigma_{r,j}^2} - \frac{1}{\sigma_{\theta,j}^2} = 0
  \;\;\Longrightarrow\;\;
  \sigma_{\theta,j} = \sigma_{r,j}\,.
\end{equation}
Substituting back into the constant term gives $\mu_{\theta,j} = \mu_{r,j}$.
In the same coordinate space, this is the fixed-point condition for diagonal Gaussian moment matching; Eq.~\ref{eq:moment_loss} imposes the analogous condition on pooled feature moments.
In $\phi$-space, the moment-matching gradient in~\propref{prop:mt_grad} defines a pooled feature-space regularizer motivated by $\gv_{\mathrm{KL}}^{\mathrm{G}}$.
At fixed feature-map parameters and queues, let $\hh_k=\phi_{\bar\theta}(\xx_k,\cc_k)$ be a current model feature. Its empirical sample-gradient contribution is
\begin{equation}
  \frac{\partial\widehat{\cL}_{\text{pop}}}{\partial\xx_k}
  = \bigl[J_\xx\phi_{\bar\theta}(\xx_k,\cc_k)\bigr]^\top
    \frac{\partial\widehat{\cL}_{\text{pop}}}{\partial\hh_k},
\end{equation}
where $J_\xx\phi_{\bar\theta}:=\partial\phi_{\bar\theta}/\partial\xx$, and the feature derivative is given in \propref{prop:mt_grad}. Thus the feature-moment loss is a pooled feature-space moment regularizer motivated by $\gv_{\mathrm{KL}}^{\mathrm{G}}$, not an estimate of it.

The internal velocity difference in the DMD loss in~\secref{sec:micro} yields an estimate of the full noisy reverse-KL field through Eq.~\ref{eq:d_micro}, including both moment and residual components.
The feature-moment loss adds an explicit constraint on pooled feature statistics alongside this pointwise signal.
Whether this signal has lower empirical gradient variance is an experimental question, not a consequence of the decomposition.

\subsubsection{Variational characterization via $L^2$ projection}
\label{app:l2_projection}

The score decomposition \eqref{eq:score_decomp_main} admits a variational interpretation that explains why the moment-matched Gaussian is not an arbitrary choice but rather the unique optimal affine approximation.

Assume that $p$ satisfies the regularity conditions of \propref{prop:residual_orthogonality} and additionally has finite Fisher information:
\begin{equation}
  \mathbb{E}_p[\|\ss_p(\xx)\|^2] = \mathbb{E}_p[\|\nabla_\xx \log p(\xx)\|^2] < \infty.
  \label{eq:finite_fisher}
\end{equation}
This ensures that the $L^2(p)$ optimization problem below is well-posed.

Consider the problem of approximating the score function $\ss_p$ by an affine function in $L^2(p)$:
\begin{equation}
\min_{\mA\in\mathbb{R}^{d\times d}, \bb\in\mathbb{R}^d}\;
  \mathbb{E}_p[\|\ss_p(\xx) - (\mA\xx+\bb)\|^2].
\end{equation}
The first-order optimality conditions (normal equations) for this least-squares problem are:
\begin{align}
  \mathbb{E}_p[\ss_p(\xx) - (\mA\xx+\bb)] &= \mathbf{0}, \label{eq:l2_normal_1} \\
  \mathbb{E}_p[(\ss_p(\xx) - (\mA\xx+\bb))\xx^\top] &= \mathbf{0}. \label{eq:l2_normal_2}
\end{align}
These are exactly the orthogonality conditions \eqref{eq:orthogonality_main}.

To solve for $\mA$ and $\bb$, expand the first condition:
\begin{equation}
  \mathbb{E}_p[\ss_p(\xx)] = \mA\mathbb{E}_p[\xx] + \bb.
\end{equation}
By Stein's identity in~\lemref{lem:score_stein_identities}, $\mathbb{E}_p[\ss_p(\xx)] = \mathbf{0}$ and $\mathbb{E}_p[\xx] = \mu_p$, so
\begin{equation}
  \bb = -\mA\mu_p.
  \label{eq:l2_b_solution}
\end{equation}
Substituting into the second condition and expanding:
\begin{align}
  \mathbb{E}_p[\ss_p(\xx)\xx^\top] &= \mA\mathbb{E}_p[\xx\xx^\top] + \bb\mathbb{E}_p[\xx]^\top \\
  &= \mA\mathbb{E}_p[\xx\xx^\top] - \mA\mu_p\mu_p^\top \\
  &= \mA(\mathbb{E}_p[\xx\xx^\top] - \mu_p\mu_p^\top) \\
  &= \mA\Sigma_p.
\end{align}
Again by Stein's identity, $\mathbb{E}_p[\ss_p(\xx)\xx^\top] = -\mI_d$, thus
\begin{equation}
  \mA = -\Sigma_p^{-1}.
  \label{eq:l2_A_solution}
\end{equation}
Combining \eqref{eq:l2_b_solution} and \eqref{eq:l2_A_solution} recovers
\begin{equation}
  \ss_{q_p}(\xx) = \mA\xx + \bb = -\Sigma_p^{-1}(\xx-\mu_p),
\end{equation}
which is exactly the score of the moment-matched Gaussian $q_p = \cN(\mu_p, \Sigma_p)$.

Therefore, $\ss_{q_p}$ is the orthogonal projection of $\ss_p$ onto the space of affine functions in $L^2(p)$, and the residual $\boldsymbol{\eta}_p := \ss_p - \ss_{q_p}$ is orthogonal to all affine functions. This yields a Pythagorean energy decomposition:
\begin{equation}
  \mathbb{E}_p[\|\ss_p\|^2] = \mathbb{E}_p[\|\ss_{q_p}\|^2] +
  \mathbb{E}_p[\|\boldsymbol{\eta}_p\|^2].
  \label{eq:pythagorean}
\end{equation}

\paragraph{Implications for the field decomposition.}
This projection property holds within each distribution separately: $\boldsymbol{\eta}_\theta$ is the $L^2(p_\theta)$-orthogonal residual, and $\boldsymbol{\eta}_r$ is the $L^2(p_r)$-orthogonal residual. The Pythagorean decomposition \eqref{eq:pythagorean} applies to each score function individually under its own measure.

However, the difference field $\gv_{\mathrm{KL}} = (\ss_\theta - \ss_r)$ does not automatically inherit this orthogonality structure. While each residual has mean zero under its own measure, the cross-terms between $\gv_{\mathrm{KL}}^{\mathrm{G}} = (\ss_{q_\theta} - \ss_{q_r})$ and $\gv_{\mathrm{KL}}^{\mathrm{nG}} = (\boldsymbol{\eta}_\theta - \boldsymbol{\eta}_r)$ need not vanish under either $p_\theta$ or $p_r$. When these cross-moments exist, neither vanishing cross-term is guaranteed:
\begin{equation}
  \mathbb{E}_{p_\theta}[\gv_{\mathrm{KL}}^{\mathrm{G}} \cdot \gv_{\mathrm{KL}}^{\mathrm{nG}}] \;\text{may be nonzero}, \quad
  \mathbb{E}_{p_r}[\gv_{\mathrm{KL}}^{\mathrm{G}} \cdot \gv_{\mathrm{KL}}^{\mathrm{nG}}] \;\text{may be nonzero}.
\end{equation}
The decomposition is exact algebraically but not energetically orthogonal in the combined field. The two losses therefore constrain complementary \emph{statistical structures} rather than forming an exact orthogonal decomposition of gradient variance.

\subsection{Score--Velocity Identity}
\label{app:score_velocity}

This section proves the general score--velocity identity in Eq.~\ref{eq:score_velocity} for the linear Gaussian path.
The identity is a specialization of the Gaussian-path score--velocity conversion arising in stochastic-interpolant formulations and summarized pedagogically by Holderrieth and Erives~\citep{holderrieth2025introduction} to the case $\alpha_t=1{-}t,\;\beta_t=t$.

\begin{proposition}[Score--Velocity Identity for the Linear Gaussian Path]
\label{prop:score_population_velocity}
Let $\xx \sim p$ have finite second moment, let $\mepsilon\sim p_z=\cN(\mathbf{0},\mI_d)$ be independent, and define $\xx_t = (1{-}t)\xx+t\mepsilon$ for $t\in(0,1)$.
Let $p_t$ denote the marginal density of $\xx_t$ and let $\ss_t(\xx_t) := \nabla_{\xx_t}\log p_t(\xx_t)$.
For fixed $t$, the marginal velocity field is the $L^2(p_t)$-unique minimizer, up to $p_t$-null sets, over measurable vector fields:
\begin{equation}
  \vv_t^*\in\argmin_{\vv:\R^d\to\R^d}
  \mathbb{E}\!\left[\|\vv(\xx_t)-(\mepsilon-\xx)\|_2^2\right],
  \qquad
  \vv_t^*(\xx_t)=\mathbb{E}[\mepsilon-\xx\mid\xx_t].
\end{equation}
Then, we have the following identity:
\begin{equation}
    \ss_t(\xx_t)
    =
    -\frac{\xx_t+(1{-}t)\vv_t^*(\xx_t)}{t}.
\end{equation}
\end{proposition}

\begin{proof}
\textbf{Step 1: Rewrite as an additive Gaussian observation model.}
The OT interpolant is $\xx_t=(1{-}t)\xx+t\mepsilon$. Define $\xx':=(1{-}t)\xx$. Then $\xx_t=\xx'+t\mepsilon$ and $\xx_t\mid\xx'\sim\cN(\xx',t^2\mI_d)$.

\textbf{Step 2: Apply the Gaussian posterior-mean identity.}
For a Gaussian perturbation model $\yy=\xx'+\sigma\mepsilon$, direct differentiation of the Gaussian convolution gives $\mathbb{E}[\xx'\mid\yy]=\yy+\sigma^2\nabla_\yy\log p_Y(\yy)$, where $p_Y$ is the observation density.
In our case, $\yy = \xx_t$, $\sigma = t$, so:
\begin{equation}
  \mathbb{E}[\xx' \mid \xx_t] = \xx_t + t^2 \ss_t(\xx_t)\,.
\end{equation}

\textbf{Step 3: Recover $\mathbb{E}[\xx \mid \xx_t]$.}
Since $\xx' = (1{-}t)\xx$, we have $\mathbb{E}[\xx \mid \xx_t] = \mathbb{E}[\xx' \mid \xx_t]/(1{-}t)$:
\begin{equation}
  \mathbb{E}[\xx \mid \xx_t] = \frac{\xx_t + t^2 \ss_t(\xx_t)}{1{-}t}\,.
  \label{eq:cond_x_app}
\end{equation}

\textbf{Step 4: Recover $\mathbb{E}[\mepsilon \mid \xx_t]$.}
From $\xx_t=t\mepsilon+(1{-}t)\xx$, $\mepsilon=(\xx_t-(1{-}t)\xx)/t$.
Taking conditional expectations:
\begin{align}
  \mathbb{E}[\mepsilon \mid \xx_t]
  &= \frac{1}{t}\left[\xx_t - (1{-}t)\mathbb{E}[\xx \mid \xx_t]\right] \nonumber \\
  &= \frac{1}{t}\left[\xx_t - (1{-}t) \cdot \frac{\xx_t + t^2 \ss_t(\xx_t)}{1{-}t}\right] \nonumber \\
  &= \frac{1}{t}\left[\xx_t - \xx_t - t^2 \ss_t(\xx_t)\right] \nonumber \\
  &= -t\,\ss_t(\xx_t)\,.
  \label{eq:cond_z_app}
\end{align}

\textbf{Step 5: Form the marginal velocity field.}
This field is the conditional expectation of the target $\mepsilon-\xx$:
\begin{align}
  \vv_t^*(\xx_t) &= \mathbb{E}[\mepsilon - \xx \mid \xx_t] = \mathbb{E}[\mepsilon \mid \xx_t] - \mathbb{E}[\xx \mid \xx_t] \nonumber \\
  &= -t\,\ss_t(\xx_t) - \frac{\xx_t + t^2\ss_t(\xx_t)}{1{-}t}\,.
\end{align}

\textbf{Step 6: Rearrange to obtain the identity.}
Multiplying both sides by $(1{-}t)$:
\begin{align}
  (1{-}t)\vv_t^*(\xx_t) &= -t(1{-}t)\ss_t(\xx_t) - \xx_t - t^2\ss_t(\xx_t) \nonumber \\
  &= -t\ss_t(\xx_t)\underbrace{[(1{-}t) + t]}_{=1} - \xx_t \nonumber \\
  &= -t\ss_t(\xx_t) - \xx_t\,.
\end{align}
Rearranging and dividing by $-t$ gives
\begin{equation}
  \ss_t(\xx_t) = -\frac{\xx_t + (1{-}t)\vv_t^*(\xx_t)}{t}\,.
\end{equation}
\end{proof}

\begin{remark}[Boundary $t=1$]
At $t=1$, $\xx_{1}=\mepsilon\sim\cN(\mathbf{0},\mI_d)$ irrespective of the clean distribution $p$, so $p_1=p_z$ and $\ss_1(\xx_1)=-\xx_1$. The right-hand side gives $-(\xx_1+0\cdot\vv_1^*(\xx_1))/1=-\xx_1$, confirming the identity at $t=1$.
\end{remark}

\begin{corollary}[Marginal Velocity Difference Recovers Score Difference]
\label{cor:population_vel_score}
For $i\in\{1,2\}$, let $\xx^{(i)}\sim p_i$ have finite second moment, and independent $\mepsilon^{(i)}\sim p_z$, and define $\xx_t^{(i)}=(1{-}t)\xx^{(i)}+t\mepsilon^{(i)}$ for $t\in(0,1)$.
Let $p_{i,t}$ be the marginal density of $\xx_t^{(i)}$, and $\ss_{i,t}(\xx_t) := \nabla_{\xx_t}\log p_{i,t}(\xx_t)$
be its marginal score. Define the corresponding marginal
velocity field by $\vv_{i,t}^*(\xx_t):=\mathbb{E}[\mepsilon^{(i)}-\xx^{(i)}\mid\xx_t^{(i)}=\xx_t]$.
Then, for every $\xx_t$ where both quantities are well defined,
\begin{equation}
    \vv_{1,t}^*(\xx_t)-\vv_{2,t}^*(\xx_t)
    =
    -\frac{t}{1-t}
    \left(
        \ss_{1,t}(\xx_t)-\ss_{2,t}(\xx_t)
    \right).
\end{equation}
\end{corollary}

\begin{proof}
Applying the Score--Velocity Identity to each field separately:
\begin{align}
  \ss_{1,t}(\xx_t) &= -\frac{\xx_t + (1{-}t)\vv^*_{1,t}(\xx_t)}{t}\,, \\
  \ss_{2,t}(\xx_t) &= -\frac{\xx_t + (1{-}t)\vv^*_{2,t}(\xx_t)}{t}\,.
\end{align}
Subtracting the second from the first:
\begin{align}
  \ss_{1,t}(\xx_t) - \ss_{2,t}(\xx_t)
  &= -\frac{1}{t}\left[(\xx_t + (1{-}t)\vv^*_{1,t}) - (\xx_t + (1{-}t)\vv^*_{2,t})\right] \nonumber \\
  &= -\frac{1{-}t}{t}\left[\vv^*_{1,t}(\xx_t)-\vv^*_{2,t}(\xx_t)\right]\,.
\end{align}
Rearranging for the velocity difference:
\begin{equation}
  \vv^*_{1,t}(\xx_t) - \vv^*_{2,t}(\xx_t) = -\frac{t}{1{-}t}\left(\ss_{1,t}(\xx_t) - \ss_{2,t}(\xx_t)\right)\,.
\end{equation}
\end{proof}

\subsection{DMD Loss: Score-Difference Identity and Regression}
\label{app:micro_theory}

Following the conventions above, $\gv_{\mathrm{KL},t}(\xx_t) := \ss_{\theta,t}(\xx_t) - \ss_{r,t}(\xx_t)$ denotes the reverse-KL field of the noisy marginals $p_{\theta,t}$ and $p_{r,t}$ at noise level $t$, the time-$t$ analog of the endpoint field $\gv_{\mathrm{KL}} = \ss_\theta - \ss_r$.

\begin{proposition}[Noisy-Marginal Score Difference]
\label{prop:micro_dmd}
Let $\vv^*_{\theta,t}$ and $\vv^*_{r,t}$ denote the marginal velocity fields of $p_{\theta,t}$ and $p_{r,t}$, evaluated at the same $\xx_t$. Then the exact velocity difference
\begin{equation}
  \dd_t^*(\xx_t)
  :=
  \vv^*_{\theta,t}(\xx_t)-\vv^*_{r,t}(\xx_t)
\end{equation}
is proportional to the noisy-marginal reverse-KL field:
\begin{equation}
  \dd_t^*(\xx_t)
  =
  \frac{t}{1-t}\left(\ss_{r,t}(\xx_t)-\ss_{\theta,t}(\xx_t)\right)
  =
  -\frac{t}{1-t}\,\gv_{\mathrm{KL},t}(\xx_t),
  \label{eq:micro_descent_app}
\end{equation}
where $\gv_{\mathrm{KL},t}:=\ss_{\theta,t}-\ss_{r,t}$.
Equivalently, the noisy-marginal distribution gradient can be recovered as
\begin{equation}
  \gv_{\mathrm{KL},t}(\xx_t)
  =
  -\frac{1-t}{t}\dd_t^*(\xx_t).
  \label{eq:micro_grad_app}
\end{equation}
\end{proposition}

\begin{proof}
By \corref{cor:population_vel_score} with indices $1=\theta$ and $2=r$,
\begin{align}
  \dd_t^*
  &= \vv^*_{\theta,t} - \vv^*_{r,t}
  = -\frac{t}{1-t}(\ss_{\theta,t}(\xx_t) - \ss_{r,t}(\xx_t))
  = \frac{t}{1-t}(\ss_{r,t}(\xx_t) - \ss_{\theta,t}(\xx_t))\,.
\end{align}
Since $\gv_{\mathrm{KL},t}=\ss_{\theta,t}-\ss_{r,t}$, we have
$\dd_t^* = -\frac{t}{1-t}\,\gv_{\mathrm{KL},t}(\xx_t)$.
Rearranging gives \eqref{eq:micro_grad_app}.
\end{proof}

\propref{prop:micro_dmd} is exact for marginal velocity fields and time-$t$ noisy marginals only when both fields recover their respective exact marginals at the same $(\xx_t,t)$.
Practical neural branches and mixed objectives do not guarantee recovery of these exact marginal velocities.
TwinFlow and APEX use internally separated branches to turn this conversion into practical one-step distribution estimators~\citep{cheng2025twinflow,liu2026self}; we adopt APEX-style condition separation with $\cc_{\mathrm{fake}}=-\cc$ and use the losses defined in \secref{sec:micro}.
In practice, $\widehat\dd_t(\xx_t\mid\cc):=\widehat\vv_{\theta,t}(\xx_t\mid\cc)-\widehat\vv_{r,t}(\xx_t\mid\cc)$ approximates $\dd_t^*(\xx_t)$. The correction queries it at $\widehat\xx_{0,t'}^{\sg}$ and detaches the result after weighting and normalization. With neither processing step, $\widehat\dd_{t'}^{\mathrm{DMD}}=\sg\!\left(\widehat\dd_{t'}(\widehat\xx_{0,t'}^{\sg}\mid\cc)\right)$. In Eq.~\ref{eq:dmd_targets}, subtracting $\widehat\dd_{t'}^{\mathrm{DMD}}$ from the velocity target adds $t\widehat\dd_{t'}^{\mathrm{DMD}}$ to the endpoint target.

\paragraph{Gradient of the detached DMD regression.}
For fixed $\xx_t$ and $t>0$, both targets in Eq.~\ref{eq:dmd_targets} are held constant during differentiation. Their velocity-to-endpoint conversion gives
\begin{equation*}
  \widehat\xx_0-\xx_0^{\mathrm{DMD}}
  =-t\bigl(\mmF_\theta(\xx_t,t,\cc)-\vv_t^{\mathrm{DMD}}\bigr),
  \qquad \ell_x(t)=t^2\ell_v(t).
\end{equation*}
Hence $\nabla_\theta\ell_x(t)=t^2\nabla_\theta\ell_v(t)$, and the equal-weight DMD loss in Eq.~\ref{eq:l_score} satisfies
\begin{equation*}
\begin{aligned}
  \cL_{\text{score}}
  &=\mathbb E\!\left[(1+t^2)\ell_v(t)\right],\\
  \nabla_\theta\cL_{\text{score}}
  &=2\mathbb E\!\left[(1+t^2)
    [J_\theta\mmF_\theta(\xx_t,t,\cc)]^\top
    \widehat\dd_{t'}^{\mathrm{DMD}}\right].
\end{aligned}
\end{equation*}
If the correction equals the exact, unweighted velocity difference $\dd_{t'}^*(\widehat\xx_{0,t'}^{\sg})$, then Eq.~\ref{eq:micro_descent_app} and $J_\theta\widehat\xx_0=-tJ_\theta\mmF_\theta$ yield the per-example identity
\begin{equation*}
  \nabla_\theta(\ell_v(t)+\ell_x(t))
  =\frac{2t'(1+t^2)}{t(1-t')}
  [J_\theta\widehat\xx_0]^\top
  \gv_{\mathrm{KL},t'}(\widehat\xx_{0,t'}^{\sg}).
\end{equation*}
This is a pointwise pullback of the field defined by $p_{\theta,t'}$ and $p_{r,t'}$. An on-policy KL interpretation additionally requires the queried prediction law to agree with the model law used to train the model-velocity branch. For the one-step law of \secref{sec:setup}, this agreement is immediate at $t=1$, $\xx_t=\zz$, but is not implied for intermediate regression inputs.
For the squared losses above, the two per-example gradients are collinear and their combination changes the time weighting and magnitude; at $t=1$, $\ell_x=\ell_v$. Separate residual-dependent normalizations can change the proportionality factor and must be included when applying this calculation to a normalized loss.
With approximate branches and correction processing, Eq.~\ref{eq:l_score} does not in general have the same gradient as the unweighted time-averaged noisy reverse KL.

\subsection{Internal model-velocity training}
\label{app:internal_dmd_training}
Draw $\zz,\mepsilon\sim p_z$ independently, set $\xx_\theta^{\sg}:=\sg(G_\theta(\zz,\cc))$ and $\xx_{\theta,t}^{\sg}:=(1{-}t)\xx_\theta^{\sg}+t\mepsilon$, and minimize
\begin{equation}
  \cL_{\text{fake}}(\theta) = \EEb{\zz,\mepsilon,t,\cc}{\|\widehat\vv_{\theta,t}(\xx_{\theta,t}^{\sg}\mid\cc) - (\mepsilon-\xx_\theta^{\sg})\|_2^2}\,,
  \label{eq:l_fake}
\end{equation}
Here $\cc\sim p_c$, $\zz,\mepsilon\sim p_z$ are independent, and $t$ is independently sampled from a distribution supported on $(0,1)$. This auxiliary regression uses the one-step output law $p_\theta$ defined in \secref{sec:setup}.
Gradients pass through the model-velocity branch, while generated samples and regression targets are detached. The branches share $\theta$, so this auxiliary regression also updates shared parameters.

\subsection{Gradient of the Feature-Moment Loss}
\label{app:macro_theory}

\paragraph{Fixed-feature updates.}
Within an update, the EMA feature-map parameters and queued features are held fixed, while derivatives through the current model samples are retained.
The feature-moment loss therefore cannot be reduced by directly changing the feature-map parameters within that update.
This removes a direct feature-collapse shortcut of one update but does not rule out slow degeneration as the EMA map evolves over training.

\begin{proposition}[Gradient of the Feature-Moment Loss]
\label{prop:mt_grad}
Let $\hh_n$, $n=1,\ldots,N_\phi$, denote model features: current entries $\hh_n=\phi(\xx_n,\cc_n)$ retain gradients, while stored queue entries are detached. Define the empirical means and biased empirical variances
\begin{equation*}
  \widehat\mu_\theta^\phi := N_\phi^{-1}\sum_n\hh_n\,,\qquad
  \widehat\nu_{\theta,j}^\phi := N_\phi^{-1}\sum_n(h_{n,j}-\widehat\mu_{\theta,j}^\phi)^2\,,
\end{equation*}
and for the queued reference features, let $\{\hh_k^r\}_{k=1}^K$ denote the $K$ pre-computed feature vectors stored in $\mathcal{Q}_r$ (computed as $\phi(\xx_k,\cc_k)$ at enqueue time and detached), and define
\begin{equation*}
  \widehat\mu_r^\phi := K^{-1}\!\sum_{k=1}^{K}\!\hh_k^r\,,\qquad
  \widehat\nu_{r,j}^\phi := K^{-1}\!\sum_{k=1}^{K}\!(h^r_{k,j}-\widehat\mu_{r,j}^\phi)^2\,;
\end{equation*}
assume $\widehat\nu_{\theta,j}^\phi>0$. These are estimators of the population quantities $\mu_i^\phi,\nu_{i,j}^\phi$ defined in \secref{sec:moment_general}. Write $\widehat{\cL}_{\text{pop}}$ for Eq.~\ref{eq:moment_loss} evaluated at these empirical moments. Its gradient with respect to a current feature $\hh_k$ is
\begin{equation}
  \frac{\partial \widehat{\cL}_{\text{pop}}}{\partial h_{k,j}}
  = \frac{2}{N_\phi}\!\left[(\widehat\mu_{\theta,j}^\phi - \widehat\mu_{r,j}^\phi) + \frac{h_{k,j} - \widehat\mu_{\theta,j}^\phi}{\sqrt{\widehat\nu_{\theta,j}^\phi}}\left(\sqrt{\widehat\nu_{\theta,j}^\phi} - \sqrt{\widehat\nu_{r,j}^\phi}\right)\right]\!.
\end{equation}
This is the exact gradient of the empirical objective at fixed queues; it is not, in general, an unbiased estimate of the gradient of the population loss $\cL_{\text{pop}}$.
\end{proposition}

\begin{proof}
The empirical loss decomposes as $\widehat{\cL}_{\text{pop}} = \cL_\mu + \cL_\sigma$, where
\begin{align}
  \cL_\mu &= \sum_j (\widehat\mu_{\theta,j}^\phi - \widehat\mu_{r,j}^\phi)^2\,, \\
  \cL_\sigma &= \sum_j (\sqrt{\widehat\nu_{\theta,j}^\phi} - \sqrt{\widehat\nu_{r,j}^\phi})^2\,.
\end{align}

\textbf{(a) Mean term.}
Since $\partial\widehat\mu_{\theta,j}^\phi/\partial h_{k,j}=1/N_\phi$,
\begin{equation}
  \frac{\partial \cL_\mu}{\partial h_{k,j}}
  = \frac{2}{N_\phi}(\widehat\mu_{\theta,j}^\phi - \widehat\mu_{r,j}^\phi)\,.
\end{equation}

\textbf{(b) Variance term.}
Since $\widehat\nu_{\theta,j}^\phi=N_\phi^{-1}\sum_n(h_{n,j}-\widehat\mu_{\theta,j}^\phi)^2$,
\begin{equation}
  \frac{\partial(h_{n,j}-\widehat\mu_{\theta,j}^\phi)}{\partial h_{k,j}}
  = \delta_{nk} - \frac{1}{N_\phi}\,.
\end{equation}
By the chain rule and the centered-sum identity, summing over $n$ gives
\begin{align}
  \frac{\partial \widehat\nu_{\theta,j}^\phi}{\partial h_{k,j}}
  = \frac{2}{N_\phi}(h_{k,j}-\widehat\mu_{\theta,j}^\phi)\,.
\end{align}
Applying the chain rule to the standard deviation gives
\begin{align}
  \frac{\partial \cL_\sigma}{\partial h_{k,j}}
  = \frac{2}{N_\phi}\frac{h_{k,j}-\widehat\mu_{\theta,j}^\phi}{\sqrt{\widehat\nu_{\theta,j}^\phi}}
  (\sqrt{\widehat\nu_{\theta,j}^\phi}-\sqrt{\widehat\nu_{r,j}^\phi})\,.
\end{align}

\textbf{Combining.}
Summing (a) and (b):
\begin{equation}
  \frac{\partial \widehat{\cL}_{\text{pop}}}{\partial h_{k,j}}
  = \frac{2}{N_\phi}\left[(\widehat\mu_{\theta,j}^\phi - \widehat\mu_{r,j}^\phi) + \frac{h_{k,j}-\widehat\mu_{\theta,j}^\phi}{\sqrt{\widehat\nu_{\theta,j}^\phi}}\left(\sqrt{\widehat\nu_{\theta,j}^\phi} - \sqrt{\widehat\nu_{r,j}^\phi}\right)\right]\!.
\end{equation}
For a stabilized version with $\sqrt{\widehat\nu_{i,j}^\phi+\delta}$, where $\delta>0$ is a scalar stabilizer, the same derivation replaces every standard deviation in the gradient by its stabilized value, including the denominator. The displayed gradient is for the unstabilized objective under the stated positive-variance assumption.
\end{proof}

\pagebreak[4]
\subsection{Covariance Contraction Under Regression}
\label{app:var_collapse}

This result shows a limitation of interpreting squared-error regression as distribution matching: a conditional-mean predictor need not preserve the target covariance. It is a separate observation about regression, not a consequence of the reverse-KL field decomposition or a claim that DMD lacks variance information.

\begin{proposition}[Covariance Contraction Under Regression]
\label{prop:var_collapse}
Let $(\mX,\mY)$ be any square-integrable jointly distributed random vectors and let $\mm(\mY):=\mathbb{E}[\mX\mid\mY]$ be the optimal squared-error predictor. Then
\begin{equation}
  \operatorname{Cov}(\mm(\mY))\preceq\operatorname{Cov}(\mX).
\end{equation}
Moreover, $\operatorname{Tr}\operatorname{Cov}(\mm(\mY))<\operatorname{Tr}\operatorname{Cov}(\mX)$ whenever $\mathbb{E}[\operatorname{Tr}\operatorname{Cov}(\mX\mid\mY)]>0$.
\end{proposition}

\begin{proof}
By the law of total covariance,
\begin{equation}
  \operatorname{Cov}(\mX)
  =\mathbb{E}[\operatorname{Cov}(\mX\mid\mY)]
  +\operatorname{Cov}(\mathbb{E}[\mX\mid\mY])
  =\mathbb{E}[\operatorname{Cov}(\mX\mid\mY)]+\operatorname{Cov}(\mm(\mY)).
\end{equation}
The conditional-covariance term is positive semidefinite, proving the matrix inequality. Taking traces gives the strict statement.
\end{proof}

This proposition applies only when the trained predictor realizes the conditional mean under the chosen squared-error coupling.
Strict covariance contraction occurs only when unresolved conditional variance remains.
It therefore motivates possible variance suppression in ambiguous endpoint regression, but it does not imply that deterministic or one-step generators generally underestimate variance.
Note that exact score differences \emph{do} contain variance information: for $p_\theta=\cN(0,\sigma_\theta^2)$ and $p_r=\cN(0,\sigma_r^2)$,
\begin{equation}
  \ss_\theta(x)-\ss_r(x)
  =
  -\frac{x}{\sigma_\theta^2}+\frac{x}{\sigma_r^2}\,,
\end{equation}
so the descent direction $\ss_r(x)-\ss_\theta(x)=x(1/\sigma_\theta^2-1/\sigma_r^2)$ pushes samples outward when $\sigma_\theta<\sigma_r$.
The learned pointwise estimator can nevertheless be noisy because it depends on the accuracy of the learned velocity branches and random time/noise samples.
The pooled feature-moment loss provides an explicit population-statistic correction.
Under an approximately independent, stationary-feature idealization with finite moments, the model- and reference-side mean-estimation variances scale as $O(1/N_\phi)$ and $O(1/K)$; this does not establish that its empirical parameter-gradient variance is lower than that of the DMD loss, and queue staleness can introduce bias.

\clearpage

%% file: resources/appendix_related_work.tex
\section{Related Work}
\label{sec:related}

\paragraph{Fast Generative Models.}
Progressive distillation, consistency models, shortcut models, and MeanFlow reduce sampling steps by learning transitions or relations along a generative trajectory~\citep{salimans2022progressive,song2023consistency,lu2024simplifying,frans2024one,geng2025mean,geng2026improved}. Duality Models jointly predict velocities and flow maps with a shared backbone~\citep{sun2026duality}. Distributional approaches use adversarial or score-based supervision to train generated outputs directly~\citep{sauer2023adversarial,sauer2024fast,yin2024one,yin2024improved}. rCM combines continuous-time consistency with score regularization and demonstrates their complementarity on large image and video models~\citep{zheng2025large}. LLaDA-Image-Turbo supports inference in two to four steps, while Vidu S2 targets real-time interactive video generation and editing~\citep{chen2026llada,zhang2026vidu}. \method studies distributional supervision from the generator itself.

\paragraph{Distribution Gradients and Internal Score Estimation.}
DMD estimates the score difference between noised reference and generated distributions using a reference network and an online fake-score model~\citep{yin2024one}; DMD2 adds adversarial supervision and techniques for stable few-step distillation~\citep{yin2024improved}. Related work connects adversarial and score-based updates through sample-space fields~\citep{franceschi2023unifying,weber2023score} and extends distribution matching to broader divergence families~\citep{xu2025one,wang2026uni}. FlowCPO develops an offline forward-KL formulation for preference alignment in flow models~\citep{han2026flowcpo}. TwinFlow and APEX estimate score differences through internally separated velocity branches~\citep{cheng2025twinflow,liu2026self}. \method adds a feature-moment loss to internal DMD.

\paragraph{Feature-Space Distribution Matching.}
GMMN and MMD-GAN use moment or kernel discrepancies for generative training~\citep{li2015generative,li2017mmd}. FD-Loss turns a full-covariance Gaussian feature discrepancy into a training objective and uses queue or EMA estimators to connect population statistics with batch gradients~\citep{heusel2017gans,yang2026fdloss}. Three-Body Scattering derives one-step regression targets from energy-distance interactions in frozen features~\citep{sun2026three}. Drifting and W-Flow construct distributional updates through kernel interactions and batchwise optimal transport, respectively~\citep{deng2026generative,han2026one}, while Teacher-Feature Drifting uses representations from a pretrained teacher~\citep{zhang2026teacher}. \method instead matches feature means and per-coordinate standard deviations in the generator's EMA feature space.

\paragraph{Gaussian Score Structure and Intrinsic Representations.}
The Gaussian-plus-residual score representation of \citet{wang2023hidden,wang2024unreasonable} underlies our reverse-KL field decomposition and motivates an additional feature-moment constraint. REPA shows that diffusion-transformer hidden states contain useful semantic information and can be strengthened through external representation alignment~\citep{yu2024representation}. Embedded Representation Warmup adds a dedicated representation-alignment stage before generative training~\citep{liu2025efficient}. Our EMA feature map reuses such internal representations and evolves with the generator; its utility is evaluated against external, teacher, and live self-feature alternatives in \secref{sec:exp_ablation}.

\paragraph{Text Rendering in Image Generation.}
Recent text-to-image models, including Qwen-Image and X-Omni, place increasing emphasis on accurate text rendering~\citep{wu2025qwen,geng2025x}. LongText-Bench evaluates extended Chinese and English text across multiple scenarios~\citep{geng2025x}, complementing the broader dimensions assessed by OneIG~\citep{chang2026oneig}. Z-Image-Turbo++ reports LongText as a sensitive measure of text quality under a limited sampling budget~\citep{liu2026high}. We focus on LongText performance at 2~NFE.

%% file: resources/reference.bib
@STRING{NeurIPS = "Advances in Neural Information Processing Systems"}

@STRING{ICML = "International Conference on Machine Learning"}

@STRING{ICLR = "International Conference on Learning Representations"}

@STRING{CVPR = "IEEE Conference on Computer Vision and Pattern Recognition"}

@inproceedings{esser2024scaling,
  author    = {Esser, Patrick and Kulal, Sumith and Blattmann, Andreas and Entezari, Rahim and M{\"u}ller, Jonas and Saini, Harry and Levi, Yam and Lorenz, Dominik and Sauer, Axel and Boesel, Frederic and others},
  booktitle = ICML,
  title     = {Scaling Rectified Flow Transformers for High-Resolution Image Synthesis},
  year      = {2024}
}

@article{podell2023sdxl,
  author  = {Podell, Dustin and English, Zion and Lacey, Kyle and Blattmann, Andreas and Dockhorn, Tim and M{\"u}ller, Jonas and Penna, Joe and Rombach, Robin},
  journal = {arXiv preprint arXiv:2307.01952},
  title   = {{SDXL}: Improving Latent Diffusion models for High-Resolution Image Synthesis},
  year    = {2023}
}

@inproceedings{rombach2022high,
  author    = {Rombach, Robin and Blattmann, Andreas and Lorenz, Dominik and Esser, Patrick and Ommer, Bj{\"o}rn},
  booktitle = CVPR,
  title     = {High-Resolution Image Synthesis with Latent Diffusion Models},
  year      = {2022}
}

@inproceedings{goodfellow2014generative,
  author    = {Goodfellow, Ian and Pouget-Abadie, Jean and Mirza, Mehdi and Xu, Bing and Warde-Farley, David and Ozair, Sherjil and Courville, Aaron and Bengio, Yoshua},
  booktitle = NeurIPS,
  title     = {Generative Adversarial Nets},
  year      = {2014}
}

@inproceedings{heusel2017gans,
  author    = {Heusel, Martin and Ramsauer, Hubert and Unterthiner, Thomas and Nessler, Bernhard and Hochreiter, Sepp},
  booktitle = NeurIPS,
  title     = {{GANs} Trained by a Two Time-Scale Update Rule Converge to a Local Nash Equilibrium},
  year      = {2017}
}

@article{sauer2024fast,
  author  = {Sauer, Axel and Boesel, Frederic and Dockhorn, Tim and Blattmann, Andreas and Esser, Patrick and Rombach, Robin},
  journal = {arXiv preprint arXiv:2403.12015},
  title   = {Fast High-Resolution Image Synthesis with Latent Adversarial Diffusion Distillation},
  year    = {2024}
}

@article{sauer2023adversarial,
  author  = {Sauer, Axel and Lorenz, Dominik and Blattmann, Andreas and Rombach, Robin},
  journal = {arXiv preprint arXiv:2311.17042},
  title   = {Adversarial Diffusion Distillation},
  year    = {2023}
}

@article{chen2024pixartsigma,
  author  = {Junsong Chen and Chongjian Ge and Enze Xie and Yue Wu and Lewei Yao and Xiaozhe Ren and Zhongdao Wang and Ping Luo and Huchuan Lu and Zhenguo Li},
  journal = {arXiv preprint arXiv:2403.04692},
  title   = {{PixArt-{$\Sigma$}}: Weak-to-Strong Training of Diffusion Transformer for 4K Text-to-Image Generation},
  year    = {2024}
}

@inproceedings{liu2022flow,
  author    = {Liu, Xingchao and Gong, Chengyue and Liu, Qiang},
  booktitle = ICLR,
  title     = {Flow Straight and Fast: Learning to Generate and Transfer Data with Rectified Flow},
  year      = {2023}
}

@article{lipman2022flow,
  author  = {Lipman, Yaron and Chen, Ricky TQ and Ben-Hamu, Heli and Nickel, Maximilian and Le, Matt},
  journal = {arXiv preprint arXiv:2210.02747},
  title   = {Flow matching for generative modeling},
  year    = {2022}
}

@article{yu2024representation,
  author  = {Yu, Sihyun and Kwak, Sangkyung and Jang, Huiwon and Jeong, Jongheon and Huang, Jonathan and Shin, Jinwoo and Xie, Saining},
  journal = {arXiv preprint arXiv:2410.06940},
  title   = {Representation alignment for generation: Training diffusion transformers is easier than you think},
  year    = {2024}
}

@article{ho2020denoising,
  author  = {Ho, Jonathan and Jain, Ajay and Abbeel, Pieter},
  journal = {Advances in neural information processing systems},
  pages   = {6840--6851},
  title   = {Denoising diffusion probabilistic models},
  volume  = {33},
  year    = {2020}
}

@article{song2020score,
  author  = {Song, Yang and Sohl-Dickstein, Jascha and Kingma, Diederik P and Kumar, Abhishek and Ermon, Stefano and Poole, Ben},
  journal = {arXiv preprint arXiv:2011.13456},
  title   = {Score-based generative modeling through stochastic differential equations},
  year    = {2020}
}

@article{yin2024improved,
  author  = {Yin, Tianwei and Gharbi, Micha{\"e}l and Park, Taesung and Zhang, Richard and Shechtman, Eli and Durand, Fredo and Freeman, William T},
  journal = {arXiv:2405.14867},
  title   = {Improved Distribution Matching Distillation for Fast Image Synthesis},
  year    = {2024}
}

@article{chen2025sana,
  author  = {Chen, Junsong and Xue, Shuchen and Zhao, Yuyang and Yu, Jincheng and Paul, Sayak and Chen, Junyu and Cai, Han and Xie, Enze and Han, Song},
  journal = {arXiv preprint arXiv:2503.09641},
  title   = {SANA-Sprint: One-Step Diffusion with Continuous-Time Consistency Distillation},
  year    = {2025}
}

@inproceedings{yin2024one,
  title     = {One-step diffusion with distribution matching distillation},
  author    = {Yin, Tianwei and Gharbi, Micha{\"e}l and Zhang, Richard and Shechtman, Eli and Durand, Fredo and Freeman, William T and Park, Taesung},
  booktitle = {Proceedings of the IEEE/CVF conference on computer vision and pattern recognition},
  pages     = {6613--6623},
  year      = {2024}
}

@misc{Dalle-3,
  author = "OpenAI",
  title = "Dalle-3",
  year = "2023",
  url = "https://openai.com/dall-e-3",
}

@misc{FLUX,
  author = "Black Forest Labs",
  title = "FLUX",
  year = "2024",
  url = "https://blackforestlabs.ai/",
}

@article{chenpixartsigma,
  title={Pixart-$\sigma$: Weak-to-strong training of diffusion transformer for 4k text-to-image generation},
  author={Chen, Junsong and Ge, Chongjian and Xie, Enze and Wu, Yue and Yao, Lewei and Ren, Xiaozhe and Wang, Zhongdao and Luo, Ping and Lu, Huchuan and Li, Zhenguo},
  journal={arXiv preprint arXiv:2403.04692},
  year={2024}
}

@article{wang2024phased,
  title={Phased Consistency Model},
  author={Wang, Fu-Yun and Huang, Zhaoyang and Bergman, Alexander William and Shen, Dazhong and Gao, Peng and Lingelbach, Michael and Sun, Keqiang and Bian, Weikang and Song, Guanglu and Liu, Yu and others},
  journal={arXiv preprint arXiv:2405.18407},
  year={2024}
}

@article{luo2023latent,
  title={Latent consistency models: Synthesizing high-resolution images with few-step inference},
  author={Luo, Simian and Tan, Yiqin and Huang, Longbo and Li, Jian and Zhao, Hang},
  journal={arXiv preprint arXiv:2310.04378},
  year={2023}
}

@article{chen2024pixartdelta,
  title={Pixart-$\{$$\backslash$delta$\}$: Fast and controllable image generation with latent consistency models},
  author={Chen, Junsong and Wu, Yue and Luo, Simian and Xie, Enze and Paul, Sayak and Luo, Ping and Zhao, Hang and Li, Zhenguo},
  journal={arXiv preprint arXiv:2401.05252},
  year={2024}
}

@article{xie2024show,
  title={Show-o: One single transformer to unify multimodal understanding and generation},
  author={Xie, Jinheng and Mao, Weijia and Bai, Zechen and Zhang, David Junhao and Wang, Weihao and Lin, Kevin Qinghong and Gu, Yuchao and Chen, Zhijie and Yang, Zhenheng and Shou, Mike Zheng},
  journal={arXiv preprint arXiv:2408.12528},
  year={2024}
}

@article{wang2024emu3,
  title={Emu3: Next-token prediction is all you need},
  author={Wang, Xinlong and Zhang, Xiaosong and Luo, Zhengxiong and Sun, Quan and Cui, Yufeng and Wang, Jinsheng and Zhang, Fan and Wang, Yueze and Li, Zhen and Yu, Qiying and others},
  journal={arXiv preprint arXiv:2409.18869},
  year={2024}
}

@inproceedings{chen2024pixartalpha,
  title={PixArt-$\alpha$: Fast Training of Diffusion Transformer for Photorealistic Text-to-Image Synthesis},
  author={Chen, Junsong and Yu, Jincheng and Ge, Chongjian and Yao, Lewei and Xie, Enze and Wang, Zhongdao and Kwok, James T and Luo, Ping and Lu, Huchuan and Li, Zhenguo},
  booktitle={ICLR},
  year={2024}
}

@misc{flux2024,
    author={BlackForest},
    title={FLUX},
    year={2024},
    howpublished={\url{https://github.com/black-forest-labs/flux}},
}

@inproceedings{ma2025janusflow,
  title={Janusflow: Harmonizing autoregression and rectified flow for unified multimodal understanding and generation},
  author={Ma, Yiyang and Liu, Xingchao and Chen, Xiaokang and Liu, Wen and Wu, Chengyue and Wu, Zhiyu and Pan, Zizheng and Xie, Zhenda and Zhang, Haowei and Yu, Xingkai and others},
  booktitle={Proceedings of the Computer Vision and Pattern Recognition Conference},
  pages={7739--7751},
  year={2025}
}

@article{qin2025lumina,
  title={Lumina-image 2.0: A unified and efficient image generative framework},
  author={Qin, Qi and Zhuo, Le and Xin, Yi and Du, Ruoyi and Li, Zhen and Fu, Bin and Lu, Yiting and Yuan, Jiakang and Li, Xinyue and Liu, Dongyang and others},
  journal={arXiv preprint arXiv:2503.21758},
  year={2025}
}

@article{chen2025janus,
  title={Janus-pro: Unified multimodal understanding and generation with data and model scaling},
  author={Chen, Xiaokang and Wu, Zhiyu and Liu, Xingchao and Pan, Zizheng and Liu, Wen and Xie, Zhenda and Yu, Xingkai and Ruan, Chong},
  journal={arXiv preprint arXiv:2501.17811},
  year={2025}
}

@article{cai2025hidream,
  title={HiDream-I1: A High-Efficient Image Generative Foundation Model with Sparse Diffusion Transformer},
  author={Cai, Qi and Chen, Jingwen and Chen, Yang and Li, Yehao and Long, Fuchen and Pan, Yingwei and Qiu, Zhaofan and Zhang, Yiheng and Gao, Fengbin and Xu, Peihan and others},
  journal={arXiv preprint arXiv:2505.22705},
  year={2025}
}

@misc{gptimage,
  title={GPT-Image-1},
  author={OpenAI},
  year={2025},
  url={https://openai.com/index/introducing-4o-image-generation/},
}

@article{gao2025seedream,
  title={Seedream 3.0 technical report},
  author={Gao, Yu and Gong, Lixue and Guo, Qiushan and Hou, Xiaoxia and Lai, Zhichao and Li, Fanshi and Li, Liang and Lian, Xiaochen and Liao, Chao and Liu, Liyang and others},
  journal={arXiv preprint arXiv:2504.11346},
  year={2025}
}

@article{wu2025qwen,
  title={Qwen-image technical report},
  author={Wu, Chenfei and Li, Jiahao and Zhou, Jingren and Lin, Junyang and Gao, Kaiyuan and Yan, Kun and Yin, Sheng-ming and Bai, Shuai and Xu, Xiao and Chen, Yilei and others},
  journal={arXiv preprint arXiv:2508.02324},
  year={2025}
}

@article{li2024hunyuan-dit,
  title={Hunyuan-dit: A powerful multi-resolution diffusion transformer with fine-grained chinese understanding},
  author={Li, Zhimin and Zhang, Jianwei and Lin, Qin and Xiong, Jiangfeng and Long, Yanxin and Deng, Xinchi and Zhang, Yingfang and Liu, Xingchao and Huang, Minbin and Xiao, Zedong and others},
  journal={arXiv preprint arXiv:2405.08748},
  year={2024}
}

@article{li2024playgroundv2.5,
  title={Playground v2. 5: Three insights towards enhancing aesthetic quality in text-to-image generation},
  author={Li, Daiqing and Kamko, Aleks and Akhgari, Ehsan and Sabet, Ali and Xu, Linmiao and Doshi, Suhail},
  journal={arXiv preprint arXiv:2402.17245},
  year={2024}
}

@article{ghosh2023geneval,
  title={Geneval: An object-focused framework for evaluating text-to-image alignment},
  author={Ghosh, Dhruba and Hajishirzi, Hannaneh and Schmidt, Ludwig},
  journal={Advances in Neural Information Processing Systems},
  volume={36},
  pages={52132--52152},
  year={2023}
}

@article{hu2024ella,
  title={Ella: Equip diffusion models with llm for enhanced semantic alignment},
  author={Hu, Xiwei and Wang, Rui and Fang, Yixiao and Fu, Bin and Cheng, Pei and Yu, Gang},
  journal={arXiv preprint arXiv:2403.05135},
  year={2024}
}

@misc{Qwen-Image-Lightning,
  author = {ModelTC},
  title = {Qwen-Image-Lightning},
  year = {2025},
  publisher = {GitHub},
  journal = {GitHub repository},
  howpublished = {\url{GitHub - ModelTC/Qwen-Image-Lightning: Qwen-Image-Lightning: Speed up Qwen-Image model with distilla}},
}

@inproceedings{wu2025janus,
  title={Janus: Decoupling visual encoding for unified multimodal understanding and generation},
  author={Wu, Chengyue and Chen, Xiaokang and Wu, Zhiyu and Ma, Yiyang and Liu, Xingchao and Pan, Zizheng and Liu, Wen and Xie, Zhenda and Yu, Xingkai and Ruan, Chong and others},
  booktitle={Proceedings of the Computer Vision and Pattern Recognition Conference},
  pages={12966--12977},
  year={2025}
}

@article{chen2025blip3,
  title={Blip3-o: A family of fully open unified multimodal models-architecture, training and dataset},
  author={Chen, Jiuhai and Xu, Zhiyang and Pan, Xichen and Hu, Yushi and Qin, Can and Goldstein, Tom and Huang, Lifu and Zhou, Tianyi and Xie, Saining and Savarese, Silvio and others},
  journal={arXiv preprint arXiv:2505.09568},
  year={2025}
}

@article{chen2025sharegpt,
  title={ShareGPT-4o-Image: Aligning Multimodal Models with GPT-4o-Level Image Generation},
  author={Chen, Junying and Cai, Zhenyang and Chen, Pengcheng and Chen, Shunian and Ji, Ke and Wang, Xidong and Yang, Yunjin and Wang, Benyou},
  journal={arXiv preprint arXiv:2506.18095},
  year={2025}
}

@article{zhuo2024lumina,
  title={Lumina-next: Making lumina-t2x stronger and faster with next-dit},
  author={Zhuo, Le and Du, Ruoyi and Xiao, Han and Li, Yangguang and Liu, Dongyang and Huang, Rongjie and Liu, Wenze and Zhu, Xiangyang and Wang, Fu-Yun and Ma, Zhanyu and others},
  journal={Advances in Neural Information Processing Systems},
  volume={37},
  pages={131278--131315},
  year={2024}
}

@misc{lu2025hyperbagel,
      title={Hyper-Bagel: A Unified Acceleration Framework for Multimodal Understanding and Generation}, 
      author={Yanzuo Lu and Xin Xia and Manlin Zhang and Huafeng Kuang and Jianbin Zheng and Yuxi Ren and Xuefeng Xiao},
      year={2025},
      eprint={2509.18824},
      archivePrefix={arXiv},
      primaryClass={cs.CV}
}

@article{deng2025bagel,
  title   = {Emerging Properties in Unified Multimodal Pretraining},
  author  = {Deng, Chaorui and Zhu, Deyao and Li, Kunchang and Gou, Chenhui and Li, Feng and Wang, Zeyu and Zhong, Shu and Yu, Weihao and Nie, Xiaonan and Song, Ziang and Shi, Guang and Fan, Haoqi},
  journal = {arXiv preprint arXiv:2505.14683},
  year    = {2025}
}

@article{niu2025wise,
  title={Wise: A world knowledge-informed semantic evaluation for text-to-image generation},
  author={Niu, Yuwei and Ning, Munan and Zheng, Mengren and Jin, Weiyang and Lin, Bin and Jin, Peng and Liao, Jiaqi and Feng, Chaoran and Ning, Kunpeng and Zhu, Bin and others},
  journal={arXiv preprint arXiv:2503.07265},
  year={2025}
}

@article{pan2025transfer,
  title={Transfer between modalities with metaqueries},
  author={Pan, Xichen and Shukla, Satya Narayan and Singh, Aashu and Zhao, Zhuokai and Mishra, Shlok Kumar and Wang, Jialiang and Xu, Zhiyang and Chen, Jiuhai and Li, Kunpeng and Juefei-Xu, Felix and others},
  journal={arXiv preprint arXiv:2504.06256},
  year={2025}
}

@article{wu2024vila,
  title={Vila-u: a unified foundation model integrating visual understanding and generation},
  author={Wu, Yecheng and Zhang, Zhuoyang and Chen, Junyu and Tang, Haotian and Li, Dacheng and Fang, Yunhao and Zhu, Ligeng and Xie, Enze and Yin, Hongxu and Yi, Li and others},
  journal={arXiv preprint arXiv:2409.04429},
  year={2024}
}

@article{fang2025flux,
  title={Flux-reason-6m \& prism-bench: A million-scale text-to-image reasoning dataset and comprehensive benchmark},
  author={Fang, Rongyao and Yu, Aldrich and Duan, Chengqi and Huang, Linjiang and Bai, Shuai and Cai, Yuxuan and Wang, Kun and Liu, Si and Liu, Xihui and Li, Hongsheng},
  journal={arXiv preprint arXiv:2509.09680},
  year={2025}
}

@article{team2025zimage,
  title={Z-Image: An Efficient Image Generation Foundation Model with Single-Stream Diffusion Transformer},
  author={Z-Image Team},
  journal={arXiv preprint arXiv:2511.22699},
  year={2025}
}

@article{liu2026high,
  title={High-Fidelity Two-Step Image Generation via Teacher-Aligned End-to-End Distillation},
  author={Liu, Dongyang and Du, Ruoyi and Liu, David and Jiang, Dengyang and Li, Liangchen and Wu, Qilong and Li, Zhen and Hoi, Steven CH and Li, Hongsheng and Gao, Peng},
  journal={arXiv preprint arXiv:2606.12575},
  year={2026}
}

@article{frans2024one,
  title={One step diffusion via shortcut models},
  author={Frans, Kevin and Hafner, Danijar and Levine, Sergey and Abbeel, Pieter},
  journal={arXiv preprint arXiv:2410.12557},
  year={2024}
}

@inproceedings{song2023consistency,
  title={Consistency models},
  author={Song, Yang and Dhariwal, Prafulla and Chen, Mark and Sutskever, Ilya},
  booktitle={Proceedings of the 40th International Conference on Machine Learning},
  pages={32211--32252},
  year={2023}
}

@article{lu2024simplifying,
  title={Simplifying, stabilizing and scaling continuous-time consistency models},
  author={Lu, Cheng and Song, Yang},
  journal={arXiv preprint arXiv:2410.11081},
  year={2024}
}

@article{salimans2022progressive,
  title={Progressive distillation for fast sampling of diffusion models},
  author={Salimans, Tim and Ho, Jonathan},
  journal={arXiv preprint arXiv:2202.00512},
  year={2022}
}

@article{geng2025mean,
  title={Mean flows for one-step generative modeling},
  author={Geng, Zhengyang and Deng, Mingyang and Bai, Xingjian and Kolter, J Zico and He, Kaiming},
  journal={arXiv preprint arXiv:2505.13447},
  year={2025}
}

@article{cheng2025twinflow,
  title={TwinFlow: Realizing One-step Generation on Large Models with Self-adversarial Flows},
  author={Cheng, Zhenglin and Sun, Peng and Li, Jianguo and Lin, Tao},
  journal={arXiv preprint arXiv:2512.05150},
  year={2025}
}

@inproceedings{arjovsky2017wasserstein,
  title={Wasserstein generative adversarial networks},
  author={Arjovsky, Martin and Chintala, Soumith and Bottou, L{\'e}on},
  booktitle={International conference on machine learning},
  pages={214--223},
  year={2017},
  organization={PMLR}
}

@article{nowozin2016f,
  title={f-gan: Training generative neural samplers using variational divergence minimization},
  author={Nowozin, Sebastian and Cseke, Botond and Tomioka, Ryota},
  journal={Advances in neural information processing systems},
  volume={29},
  year={2016}
}

@misc{sun2025anystep,
  author = {Sun, Peng and Lin, Tao},
  note   = {GitHub repository},
  title  = {Any-step Generation via N-th Order Recursive Consistent Velocity Field Estimation},
  url    = {https://github.com/LINs-lab/RCGM},
  year   = {2025}
}

@article{liu2025efficient,
  title={Efficient generative model training via embedded representation warmup},
  author={Liu, Deyuan and Sun, Peng and Li, Xufeng and Lin, Tao},
  journal={arXiv preprint arXiv:2504.10188},
  year={2025}
}

@article{liu2026self,
  title={Self-Adversarial One Step Generation via Condition Shifting},
  author={Liu, Deyuan and Sun, Peng and Han, Yansen and Cheng, Zhenglin and Chen, Chuyan and Lin, Tao},
  journal={arXiv preprint arXiv:2604.12322},
  year={2026}
}

@inproceedings{geng2026improved,
  title={Improved mean flows: On the challenges of fastforward generative models},
  author={Geng, Zhengyang and Lu, Yiyang and Wu, Zongze and Shechtman, Eli and Kolter, J Zico and He, Kaiming},
  booktitle={Proceedings of the IEEE/CVF Conference on Computer Vision and Pattern Recognition},
  pages={30467--30476},
  year={2026}
}

@article{han2026one,
  title={One-Step Generative Modeling via Wasserstein Gradient Flows},
  author={Han, Jiaqi and Li, Puheng and Guo, Qiushan and Xu, Renyuan and Ermon, Stefano and Cand{\`e}s, Emmanuel J},
  journal={arXiv preprint arXiv:2605.11755},
  year={2026}
}

@article{deng2026generative,
  title={Generative Modeling via Drifting},
  author={Deng, Mingyang and Li, He and Li, Tianhong and Du, Yilun and He, Kaiming},
  journal={arXiv preprint arXiv:2602.04770},
  year={2026}
}

@article{yang2026fdloss,
  title={Representation Fréchet Loss for Visual Generation},
  author={Yang, Jiawei and Geng, Zhengyang and Ju, Xuan and Tian, Yonglong and Wang, Yue},
  journal={arXiv:2604.28190},
  url={https://arxiv.org/abs/2604.28190},
  year={2026}
}

@article{zheng2025large,
  title={Large scale diffusion distillation via score-regularized continuous-time consistency},
  author={Zheng, Kaiwen and Wang, Yuji and Ma, Qianli and Chen, Huayu and Zhang, Jintao and Balaji, Yogesh and Chen, Jianfei and Liu, Ming-Yu and Zhu, Jun and Zhang, Qinsheng},
  journal={arXiv preprint arXiv:2510.08431},
  year={2025}
}

@article{zhang2026teacher,
  title={Teacher-Feature Drifting: One-Step Diffusion Distillation with Pretrained Diffusion Representations},
  author={Zhang, Yuan and Li, Chenyi and Ma, Guoqing and Zha, Jiajun and Yang, Yuanming and Wang, Bo and Tang, Wei and Li, Wenbo and Huang, Haoyang and Duan, Nan},
  journal={arXiv preprint arXiv:2605.07327},
  year={2026}
}

@article{xu2025one,
  title={One-step Diffusion Models with $ f $-Divergence Distribution Matching},
  author={Xu, Yilun and Nie, Weili and Vahdat, Arash},
  journal={arXiv preprint arXiv:2502.15681},
  year={2025}
}

@article{wang2026uni,
  title={Uni-instruct: One-step diffusion model through unified diffusion divergence instruction},
  author={Wang, Yifei and Bai, Weimin and Zhang, Debing and Luo, Weijian and Sun, He and others},
  journal={Advances in Neural Information Processing Systems},
  volume={38},
  pages={57097--57132},
  year={2026}
}

@inproceedings{liu2016kernelized,
  title={A kernelized Stein discrepancy for goodness-of-fit tests},
  author={Liu, Qiang and Lee, Jason and Jordan, Michael},
  booktitle={International conference on machine learning},
  pages={276--284},
  year={2016},
  organization={PMLR}
}

@article{holderrieth2025introduction,
  title={An introduction to flow matching and diffusion models},
  author={Holderrieth, Peter and Erives, Ezra},
  journal={arXiv preprint arXiv:2506.02070},
  year={2025}
}

@inproceedings{li2015generative,
  title={Generative Moment Matching Networks},
  author={Li, Yujia and Swersky, Kevin and Zemel, Rich},
  booktitle={International Conference on Machine Learning},
  pages={1718--1727},
  year={2015}
}

@article{li2017mmd,
  title={{MMD} {GAN}: Towards Deeper Understanding of Moment Matching Network},
  author={Li, Chun-Liang and Chang, Wei-Cheng and Cheng, Yu and Yang, Yiming and P{\'o}czos, Barnab{\'a}s},
  journal={Advances in Neural Information Processing Systems},
  volume={30},
  year={2017}
}

@article{wang2023prolificdreamer,
  title={ProlificDreamer: High-Fidelity and Diverse Text-to-3D Generation with Variational Score Distillation},
  author={Wang, Zhengyi and Lu, Cheng and Wang, Yikai and Bao, Fan and Li, Chongxuan and Su, Hang and Zhu, Jun},
  journal={Advances in Neural Information Processing Systems},
  volume={36},
  pages={8406--8441},
  year={2023}
}

@inproceedings{yi2023monoflow,
  title={Monoflow: Rethinking divergence gans via the perspective of wasserstein gradient flows},
  author={Yi, Mingxuan and Zhu, Zhanxing and Liu, Song},
  booktitle={International Conference on Machine Learning},
  pages={39984--40000},
  year={2023},
  organization={PMLR}
}

@article{franceschi2023unifying,
  title={Unifying gans and score-based diffusion as generative particle models},
  author={Franceschi, Jean-Yves and Gartrell, Mike and Dos Santos, Ludovic and Issenhuth, Thibaut and de B{\'e}zenac, Emmanuel and Chen, Micka{\"e}l and Rakotomamonjy, Alain},
  journal={Advances in Neural Information Processing Systems},
  volume={36},
  pages={59729--59760},
  year={2023}
}

@article{weber2023score,
  title={The score-difference flow for implicit generative modeling},
  author={Weber, Romann M},
  journal={arXiv preprint arXiv:2304.12906},
  year={2023}
}

@incollection{ambrosio2007gradient,
  title={Gradient flows of probability measures},
  author={Ambrosio, Luigi and Savar{\'e}, Giuseppe},
  booktitle={Handbook of differential equations: evolutionary equations},
  volume={3},
  pages={1--136},
  year={2007},
  publisher={Elsevier}
}

@article{wang2023hidden,
  title={The hidden linear structure in score-based models and its application},
  author={Wang, Binxu and Vastola, John J},
  journal={arXiv preprint arXiv:2311.10892},
  year={2023}
}

@article{wang2024unreasonable,
  title={The unreasonable effectiveness of gaussian score approximation for diffusion models and its applications},
  author={Wang, Binxu and Vastola, John J},
  journal={arXiv preprint arXiv:2412.09726},
  year={2024}
}

@article{sun2026three,
  title={Three-Body Scattering for Generative Modeling},
  author={Sun, Peng and Cheng, Zhenglin and Liu, Deyuan and Xie, Jun and Shang, Xinyi and Lin, Tao},
  journal={arXiv preprint arXiv:2607.18198},
  year={2026}
}

@article{chang2026oneig,
  title={Oneig-bench: Omni-dimensional nuanced evaluation for image generation},
  author={Chang, Jingjing and Fang, Yixiao and Xing, Peng and Wu, Shuhan and Cheng, Wei and Wang, Rui and Zeng, Xianfang and Yu, Gang and Chen, Hai-Bao},
  journal={Advances in Neural Information Processing Systems},
  volume={38},
  year={2026}
}

@article{geng2025x,
  title={X-omni: Reinforcement learning makes discrete autoregressive image generative models great again},
  author={Geng, Zigang and Wang, Yibing and Ma, Yeyao and Li, Chen and Rao, Yongming and Gu, Shuyang and Zhong, Zhao and Lu, Qinglin and Hu, Han and Zhang, Xiaosong and others},
  journal={arXiv preprint arXiv:2507.22058},
  year={2025}
}

@article{sun2026duality,
  title={Duality Models: An Embarrassingly Simple One-step Generation Paradigm},
  author={Sun, Peng and Shang, Xinyi and Lin, Tao and Shen, Zhiqiang},
  journal={arXiv preprint arXiv:2602.17682},
  year={2026}
}

@article{zhang2026vidu,
  title={Vidu S2: Real-Time Interactive, Editable, and Spatial Video Generation},
  author={Zhang, Jintao and Jiang, Kai and Chen, Jintao and Wang, Xu and Liu, Deyuan and Li, Jungang and Chen, Dechuang and Lin, Ming and Zhou, Jingjiang and Jin, Haopeng and others},
  journal={arXiv preprint arXiv:2609.11638},
  year={2026}
}

@article{han2026flowcpo,
  title={FlowCPO: A Unified Divergence View of Preference Alignment for Flow Models},
  author={Han, Yansen and Liao, Shengyi and Sun, Peng and Liu, Deyuan and Zhang, Yuanxing and Wan, Pengfei and Lin, Tao},
  journal={arXiv preprint arXiv:2609.09905},
  year={2026}
}

@article{chen2026llada,
  title={LLaDA-Image: Building Strong Image Generators with Fully Open Training Recipes},
  author={Chen, Chuyan and Chen, Haoxing and Chen, Kun and Cheng, Zhenglin and Cui, Long and Fang, Ruishan and Gu, Zhangxuan and Huang, Zhicheng and Lan, Zhenzhong and Lei, Yuanting and others},
  journal={arXiv preprint arXiv:2609.03796},
  year={2026}
}
